%% file: arXiv.tex
\documentclass[dvipsnames]{article}

\input{math_commands.tex}

    \usepackage[final]{neurips_2024}

\usepackage[utf8]{inputenc} 
\usepackage[T1]{fontenc}    
\usepackage{titletoc}
\usepackage[pagebackref]{hyperref}   
\usepackage{url}            
\usepackage{booktabs}       
\usepackage{amsfonts}       
\usepackage{nicefrac}       
\usepackage{microtype}      
\usepackage{xcolor}         

\usepackage[ruled,vlined]{algorithm2e}

\hypersetup{colorlinks=true,linkcolor=red!70!black,linktocpage=false,citebordercolor=blue!70!black,citecolor=blue!70!black,anchorcolor=blue!70!black}

\usepackage{graphicx,subfig}
\usepackage{etoc}

\usepackage{amsmath}
\usepackage{amssymb}
\usepackage{mathtools}
\usepackage{amsthm}
\usepackage{multirow}
\theoremstyle{plain}
\newtheorem{theorem}{Theorem}[section]

\newtheorem{lemma}[theorem]{Lemma}
\newtheorem{example}[theorem]{Example}

\theoremstyle{definition}
\newtheorem{definition}[theorem]{Definition}
\newtheorem{assumption}[theorem]{Assumption}
\newtheorem{setting}[theorem]{Setting}
\newtheorem{remark}[theorem]{Remark}
\newtheorem*{main result}{Main Theorem}
\allowdisplaybreaks[4]

\usepackage{enumitem}
\definecolor{thistle}{rgb}{0.85, 0.75, 0.85}
\renewcommand{\geq}{\geqslant}
\renewcommand{\leq}{\leqslant}

\usepackage{wrapfig}
\usepackage{subfig}

\title{Improving Generalization and Convergence by Enhancing Implicit Regularization}

\author{%
  Mingze Wang$^{1,3,\dag}$
  \quad
  Jinbo Wang$^{1,3}$
  \quad
  Haotian He$^{1,3}$
  \quad
  Zilin Wang$^{1}$
  \quad
  Guanhua Huang$^{5,6}$
  \\
  {\bf Feiyu Xiong}$^{3}$
  \quad
  {\bf Zhiyu Li}$^{3}$
  \quad
  {\bf Weinan E}$^{1,2,3,4}$
  \quad
  {\bf Lei Wu}$^{1,2,\dag}$
  \vspace{.3cm}
  \\
  $^1$School of Mathematical Sciences, Peking University
  \\
  $^2$Center for Machine Learning Research, Peking University
  \\
  $^3$Institute for Advanced Algorithms Research (Shanghai)
  \quad
  $^4$AI for Science Institute
  \\
  $^5$School of Data Science, University of Science and Technology of China \quad
  $^6$ByteDance Research
  \vspace{.3cm}
  \\
  {\small\texttt{$\{$mingzewang, wangjinbo, haotianhe, wangzilin$\}$@stu.pku.edu.cn}}
  \\{\small\texttt{guanhuahuang@mail.ustc.edu.cn}}\quad\ 
  {\small\texttt{$\{$xiongfy, lizy$\}$@iaar.ac.cn}}
  \\
  {\small\texttt{$\{$weinan, leiwu$\}$@math.pku.edu.cn}}
}

\begin{document}

\maketitle

\begin{abstract}
    In this work, we propose an Implicit Regularization Enhancement (IRE) framework to accelerate the discovery of flat solutions in deep learning, thereby improving generalization and convergence. 
    Specifically, IRE decouples the dynamics of flat and sharp directions, which boosts the sharpness reduction along flat directions while maintaining the training stability in sharp directions. We show that IRE can be practically incorporated with {\em generic base optimizers} without introducing significant computational overload. Experiments show that IRE consistently improves the generalization performance for image classification tasks across a variety of benchmark datasets (CIFAR-10/100, ImageNet) and models (ResNets and ViTs). 
    Surprisingly, IRE also  achieves a $2\times$ {\em speed-up} compared to AdamW in the pre-training of Llama models (of sizes ranging from 60M to 229M) on datasets including Wikitext-103, Minipile, and Openwebtext. Moreover, we provide theoretical guarantees, showing that IRE can substantially accelerate the convergence towards flat minima in sharpness-aware minimization (SAM).
\end{abstract}

\vspace{-2.em}~\footnote{\dag\ Correspondence to: Mingze Wang and Lei Wu.}

\section{Introduction}
Deep learning has achieved remarkable success across a variety of fields, including computer vision, scientific computing, and artificial intelligence. The core challenge in deep learning lies in how to train deep neural networks (DNNs) efficiently to achieve superior  performance. Understanding and improving the generalization and convergence of commonly-used optimizers, such stochastic gradient descent (SGD) \citep{robbins1951stochastic,rumelhart1986learning}, in deep learning is crucial for both theoretical research and practical applications. 

Notably, optimizers often exhibit a preference for certain solutions in training DNNs. For instance, SGD and its variants consistently converge to solutions that generalize well, even when DNNs are highly over-parameterized and there are many solutions that generalize poorly. This phenomenon is referred to as \textit{implicit regularization} in the literature~\citep{neyshabur2014search, zhang2017understanding}.



The most popular explanation for implicit regularization is that SGD and its variants tend to converge to flat minima~\citep{keskar2016large, wu2017towards}, and flat minima generalize better~\citep{hochreiter1997flat, jiang2019fantastic}.
However, the process of this \textit{implicit sharpness regularization} occurs at a very slow pace, as demonstrated in works such as \cite{blanc2020implicit}, \cite{li2021happens}, and \cite{ma2022beyond}. 
Consequently, practitioners often use a large learning rate (LR) and extend the training time even when the loss no longer decreases, ensuring the convergence to flatter minima~\citep{he2016deep,goyal2017accurate,hoffer2017train}. 
Nevertheless, the largest allowable LR  is constrained by the need to maintain training stability. In addition, \cite{foret2020sharpness} proposed SAM, which aims to explicitly regularize sharpness during training and has achieved superior performance across a variety of tasks.

\textbf{Our contributions} can be summarized as follows:
\begin{itemize}[leftmargin=2em]
    \item We propose an Implicit Regularization Enhancement (IRE) framework to speed up the convergence towards flatter minima. As suggested by works like \cite{blanc2020implicit}, \cite{li2021happens} and \cite{ma2022beyond}, the implicit sharpness reduction often occurs at a very slow pace, along flat directions. Inspired by this picture, 
    IRE particularly accelerates the dynamics along flat directions, while keeping  sharp directions' dynamics unchanged. 
    As such, IRE can boost the implicit sharpness reduction substantially without hurting  training stability.  For a detailed illustration of this mechanism, we refer to Section \ref{section: example}.

    \item We then provide a practical IRE framework, which can be efficiently incorporated with generic base optimizers. We evaluate the  performance of this practical IRE in both vision and language tasks. For vision tasks, IRE consistently improves the generalization performance of popular optimizers like SGD, Adam, and SAM in classifying the CIFAR-10/100 and ImageNet datasets with ResNets~\citep{he2016deep} and vision transformers (ViTs)~\citep{dosovitskiy2020image}. For language modelling, we consider the  pre-training of Llama models~\citep{touvron2023llama} of various sizes, finding that IRE surprisingly  can accelerate the pre-training convergence. Specifically, we observe a remarkable $2.0\times$ speedup compared to AdamW in the scenarios we examined, despite IRE being primarily motivated to speed up the convergence to flat solutions.

    \item Lastly, we provide theoretical guarantees showing that   IRE can achieves a $\Theta(1/\rho)$-time acceleration over the base SAM algorithm in minimizing the trace of Hessian, where $\rho\in (0,1)$ is a small hyperparameter in SAM. 

\end{itemize}

\subsection{Related works}
\label{section: related works}

\textbf{Implicit sharpness regularization.} There have been  extensive attempts to explain the mystery of implicit regularization in deep learning (see the survey by \cite{vardi2023implicit} and references therein). Here, we focus on  works related to implicit sharpness regularization. 
\cite{wu2018sgd,wu2022does} and \cite{ma2021linear} provided an explanation of implicit sharpness regularization from a dynamical stability perspective. Moreover, in-depth analysis of SGD dynamics near global minima shows that the SGD noise \citep{blanc2020implicit,li2021happens,ma2022beyond,damian2021label} and the edge of stability (EoS)-driven~\citep{wu2018sgd,cohen2021gradient} oscillations~\citep{even2024s} can drive SGD/GD towards flatter minima. Additional studies explored how  training components, including learning rate and batch size~\citep{jastrzkebski2017three}, 
normalization~\citep{lyu2022understanding}, cyclic LR~\citep{wang2023noise},  influence this sharpness  regularization.
Furthermore, some works have provided theoretical evidence explaining the superior generalization of flat minima for neural networks~\citep{ma2021linear,mulayoff2021implicit,wu2023implicitstability,gatmiry2023inductive,wen2023sharpness}. Our work is inspired by this line of research, aiming to boost implicit sharpness regularization by decoupling the dynamics along flat and sharp directions.

\textbf{Sharpness-aware minimization.}  IRE shares the same motivation as SAM in enhancing sharpness regularization, although their specific approaches differ significantly. It is worth noting that the per-step computational cost of SAM is twice that of base optimizers. Consequently, there have been numerous attempts to reduce the computational cost of SAM~\citep{kwon2021asam, liu2022towards, du2021efficient, mi2022make, mueller2024normalization}. In contrast, the per-step computational cost of IRE is only approximately 1.1 times that of base optimizers (see Table \ref{table: running time}). Moreover, we provide both theoretical and experimental evidence demonstrating that the mechanism of IRE in boosting sharpness regularization is nearly orthogonal to that of SAM.

{\bf Optimizers for large language model (LLM) pre-training.} 
(Momentum) SGD~\citep{sutskever2013importance,nesterov1983method} and its adaptive variants like Adagrad \citep{duchi2011adaptive}, RMSProp~\citep{tieleman2012lecture}, and Adam~\citep{kingma2014adam} have been widely used in DNN training. Despite the efforts in designing better adaptive gradient methods~\citep{liu2019variance,luo2019adaptive,heo2020adamp,zhuang2020adabelief,xie2022adaptive,xie2022adan}, AdamW(Adam+decoupled weight decay)~\citep{loshchilov2017decoupled} has become the default optimizer in LLM pre-training. Recently,~\cite{chen2024symbolic} discovered Lion by searching the space of adaptive first-order
optimizers;~\cite{liu2023sophia} introduced Sophia, a scalable second-order optimizer. 
In this paper, we instead empirically demonstrate that IRE can accelerate the  convergence of AdamW in the pre-training of Llama models. 
\subsection{Notations} 
Throughout this paper, let $\cL:\RR^p\mapsto\RR_{\geqslant 0}$ be the function of total loss, where $p$ denotes the number of model parameters.
For a $\cC^2$-submanifold $\cM$ in $\bbR^p$, we denote the tangent space of $\cM$ at $\btheta\in\cM$ as $\cT_{\btheta}\cM$, which is a linear subspace in $\bbR^p$. For $f\in \cC^1(\cM)$ and $\btheta\in\cM$, let  $\nabla_{\cM}f(\btheta) := \mathbb{\cQ}_{\cT_{\btheta}\cM}\nabla f(\btheta)$ denote the Riemannian gradient, where $\cQ_{\cT_{\btheta}\cM}:\RR^p\mapsto\RR^p$ denotes the orthogonal projection to $\cT_{\btheta}\cM$.
For a symmetric matrix $A\in\bbR^{p\times p}$, its eigen pairs are denoted as $\{(\lambda_i,\bu_i)\}_{i\in[p]}$ with the order $\lambda_1\geq\cdots\geq\lambda_p$. We use $P_{i:j}(A)=\sum_{k=i}^j \bu_k \bu_k^\top$  to denote the projection operator onto  $\mathrm{span}\{\bu_i,\dots,\bu_j\}$. 
Denote $\cN(\bmu,\Sigma)$ as the Gaussian distribution with mean $\bmu$ and covariance matrix $\Sigma$, and $\bbU(\cS)$ as the uniform distribution over a set $\cS$. Given a vector $\bh=(h_1,\dots,h_p)$, let $|\bh|=(|h_1|,\dots,|h_p|)$. We denote by $\bm{1}$ the all-ones vector.
We will use standard big-O notations like $\cO(\cdot)$, $\Omega(\cdot)$, and $\Theta(\cdot)$ to hide constants.

\section{An Illustrative Example Motivating IRE}\label{section: example}
In this section, we provide an illustration of how the dynamics along flat directions can {\em reduce the sharpness} (curvatures along sharp directions) and how IRE can accelerate this sharpness reduction.
To this end, we consider the following phenomenological problem:
\begin{equation}\label{equ: toy}
\min_{\btheta\in\bbR^p}\cL(\btheta):=\bv^\top H(\bu)\bv/2,
\end{equation}
where $\bv\in\bbR^{m}$, $\bu\in\bbR^{p-m}$, and $\btheta=\vec(\bu,\bv)\in\bbR^p$. We assume $H(\cdot)\in\cC^2(\bbR^{p-m})$ and $\inf_{\bu}\lambda_{\min}(H(\bu))>0$. Then, the minimizers of $\cL(\cdot)$ form a $m$-dim manifold $\cM=\{(\bu,\bv):\bv=\bzero\}$ and the Hessian at  any $\btheta\in\cM$ is given by $\nabla^2\cL(\btheta)=\begin{pmatrix}
    \bzero & \bzero \\ \bzero & H(\bu)
\end{pmatrix}$. For clarity, we shall call $\bu$ and $\bv$ the flat and sharp directions, respectively.

\begin{example}
The loss landscape of fitting zero labels with two-layer neural networks (2LNNs) exhibits exactly the form \eqref{equ: toy}. Let $f(\bx;\btheta)=\ba^\top\phi(\bx;W)$ be a 2LNN with $\btheta=\vec(W,\ba)$. Then  $\cL(\btheta)=\bbE_{(\bx,y)}[(f(\bx;\btheta)-y)^2]/2=\ba^\top\bbE_{\bx}[\phi(\bx;W)\phi(\bx;W)^\top]\ba/2=:\ba^\top H(W)\ba/2$.
\end{example}

For breviety,
we further assume $H(\bu)=\diag(\blambda(\bu))$ with $\blambda(\bu)=(\lambda_1(\bu),\cdots,\lambda_m(\bu))$. In this case, the GD dynamics can be naturally decomposed into the flat and sharp directions as follows
\begin{equation}\label{equ: toy gd dynamics}
\begin{aligned}
    \bu_{t+1}&=\bu_t-\frac{\eta}{2}\sum_{i=1}^m v_{t,i}^2\nabla\lambda_i(\bu_t),
    \\
    \bv_{t+1}&=(\bone-\eta\blambda(\bu_t))\odot \bv_{t},
\end{aligned}
\end{equation}
where $\odot$ denotes the element-wise multiplication of two vectors. 

\textbf{The implicit sharpness regularization.}
From Eq.~\eqref{equ: toy gd dynamics}, we can see that 1) the flat direction $\bu_t$'s dynamics  monotonically reduces the sharpness $\blambda(\bu)$ as long as $\bv_t$ is nonzero; 2) the sharp direction $\bv_t$'s dynamics determines the speed of sharpness reduction. The larger $|\bv_t|$ is, the faster the curvature $\blambda(\bu)$ decreases. Particularly, when near convergence, we have $|\bv_t|=o(1)$ and thus the implicit sharpness reduction is {\em very slow} during the late phase of GD. Figure \ref{fig: toy, gd converge} provides a visualization of this slow implicit sharpness reduction.

\begin{figure}[!ht]
    \hspace{-.4cm}
    \centering
    \subfloat[\small GD with $\eta=1$.]{\label{fig: toy, gd converge}
    \includegraphics[width=3.8cm]{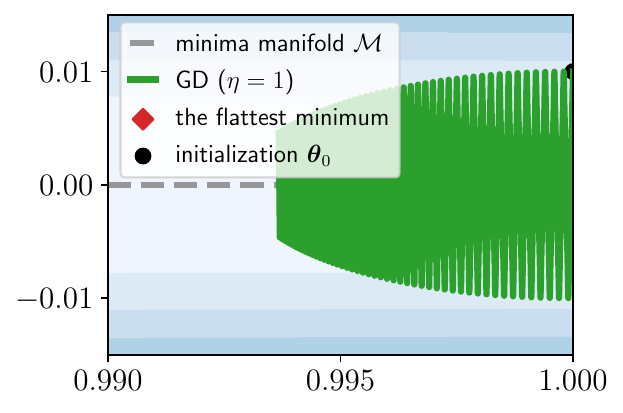}}
    \hspace{-.3cm}
    \subfloat[\small GD with $\eta=2$.]{\label{fig: toy, gd diverge}
    \includegraphics[width=3.7cm]{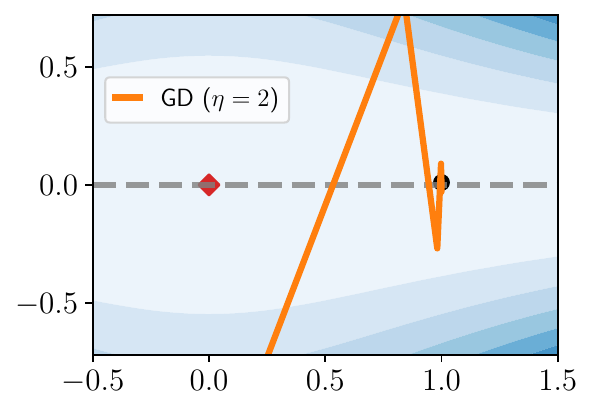}}
    \hspace{-.3cm}
    \subfloat[\small IRE with various $\kappa$'s v.s. GD with $\eta=1$.]{\label{fig: toy, ire}
    \includegraphics[width=6.6cm]{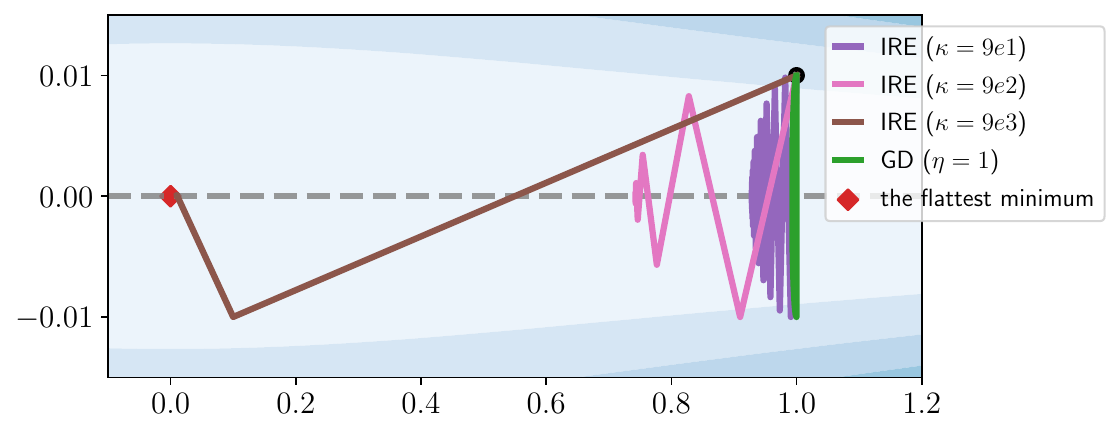}}
    \caption{\small A $2$-d example of~\eqref{equ: toy}: $\cL(u,v)=(1+u^2)v^2/2$. The {\color{gray}{gray arrows}} denote to the minima manifold $\cM=\{(u,v): v=0\}$, where 
    the smaller the $|u|$, the flatter the minimizer. 
    The {\color{red!80!black}{red marker}} highlights the flattest minimizer $(0,0)$.
    (a) The dynamics of {\color{green!50!black}{GD ($\eta=1$)}}, which moves {\em slowly} towards flatter minima as it converges.
    (b)  The dynamics of {\color{orange}{GD ($\eta=2$)}}, which diverges due to the excessively large $\eta$.
    (c) The behavior of our IRE approach with varying $\kappa$'s v.s. {\color{green!50!black}{GD ($\eta=1$)}}. Is is shown that IRE can significantly accelerate the $u_t$'s dynamics, almost reaching the flattest minimum $(0,0)$ when taking a very large $\kappa$.}
    \label{fig: toy}
    \vspace*{-1.em}
\end{figure}



\paragraph*{Accelerating the sharpness reduction.} Inspired by the above analysis, we can accelerate the sharpness reduction by speeding up the flat directions' dynamics. 
To this end, there are two approaches:
\begin{itemize}[leftmargin=2em]
\item {\bf Naively increasing the global learning rate $\eta$ {\color{orange}{(fail)}}.} 
Increasing $\eta$ accelerates the dynamics of $\bu_t$, but the largest allowed $\eta$ is constrained by curvatures of sharpest directions. In GD \eqref{equ: toy gd dynamics}, to maintain training stability, $\eta$ must be smaller than $2/\max_{i}\lambda_i(\bu_t)$. Otherwise, $\bv_t$'s dynamics will blow up. 
As illustrated in Figure~\ref{fig: toy, gd diverge},  setting $\eta=2$ leads to divergence, whereas $\eta=1$ ensures convergence.

\item {\bf Increasing only the flat directions' learning rate {\color{violet}{(our approach, IRE)}}}. Specifically, for GD \eqref{equ: toy gd dynamics}, the GD-IRE dynamics is given by 
\vspace{-.2cm}
\begin{equation}\label{equ: toy ire dynamics}
\begin{aligned}
    \bu_{t+1}&=\bu_t-{\color{violet}{(1+\kappa)}}\frac{\eta}{2}\sum_{i=1}^m v_{t,i}^2\nabla\lambda_i(\bu_t),
    \\
    \bv_{t+1}&=(\bone-\eta\blambda(\bu_t))\odot \bv_{t},
\end{aligned}
\end{equation}
where $\kappa>0$ controls the enhancement strength.  In GD-IRE~\eqref{equ: toy ire dynamics},  $\bu_t$'s dynamics is {\bf\color{violet} $(1+\kappa)$ faster} than that of GD~\eqref{equ: toy gd dynamics}. Notably, the sharp directions' dynamics ($\bv_t$) are unchanged. The choice of $\kappa$ only needs to maintain the stability of flat directions' dynamics, for which, we can always take a significantly large $\kappa$ to enhance the sharpness regularization.
As demonstrated in Figure~\ref{fig: toy, ire}, IRE with larger $\kappa$ always find  flatter minima.
\end{itemize}

\begin{remark}[The generality]
It is worth noting that similar implicit sharpness regularization also holds for SGD~\citep{ma2022beyond,li2021happens} and SAM~\citep{wen2023how}. 
In this section, we focus on the above toy model and GD mainly for illustration. In Appendix \ref{appendix: example}, we provide an analogous illustrative analysis of how IRE accelerates the sharpness reduction of SGD. In Section~\ref{section: theory}, we further provide  theoretical evidence to show that IRE can boost the implicit sharpness regularization of SAM.
\end{remark}

\section{A Practical Framework of Implementing IRE}
\label{section: algorithm}



Although the preceding illustration of IRE is for GD, in practice, we can incorporate IRE with {\em any base optimizers}. Specifically, 
for a generic update: $\btheta_{t+1}=\btheta_t - \eta \bg_t$, the corresponding IRE modification is given by
\begin{equation}\label{eqn: ire-genetric}
    \btheta_{t+1} = \btheta_t - \eta (\bg_t+\kappa \cP_{t}\bg_t),
\end{equation}

where $\kappa$ denotes the enhancement strength and $\cP_t:\bbR^p\to\bbR^p$ projects $\bg_t$ into the {\em flat directions} of the landscape. 
The flat directions and corresponding projection operator $\cP_t$ can be estimated using the Hessian information.

However, estimating the full Hessian matrix $\nabla^2 \cL(\btheta_t)\in\RR^{p\times p}$ is computationally infeasible. Consequently, we propose to use only the {\em diagonal Hessian} $\diag(\nabla^2 \cL(\btheta_t))\in\bbR^{p}$ to estimate $\cP_t$. 
Let $\bh_t\in\bbR^p$ be an estimate of the diagonal Hessian. Then, we perform the projection  as follows
\begin{align}\label{eqn: proj}
\cP_t \bg_t &= \bn_t \odot \bg_t, \text{ with } 
    (\bn_t)_i = \begin{cases}
    1 & \text{ if } (|\bh_t|)_i\leqslant \texttt{Top}_{\rm small}(|\bh_t|,\gamma)\\ 
    0 & \text{ otherwise }
    \end{cases},
\end{align}
where $\gamma\in (0,1)$ and  $\texttt{Top}_{\rm small}(|\bh_t|,\gamma)$  returns the $\lfloor p\cdot \gamma\rfloor$-th smallest value in $|\bh_t|$. Note that  $\bn_t\in\RR^p$ denotes a mask vector and the above approximate projection essentially  masks the top-$(1-\gamma)$ sharp coordinates out.  As such, the projection  \eqref{eqn: proj} will retain the {\em top-$\gamma$ flat coordinates}. Noticing that in DNNs, there are much more flat directions  than sharp directions ~\citep{yao2020pyhessian}, we thus often use $\gamma>0.5$ in practice.

\begin{algorithm}[!ht]
	\caption{{\bf Practical IRE} (A practical framework of implementing IRE)}
	\label{alg: ire}
	\KwIn{\text{$\btheta_0$, $T$, $K$, learning rate $\{\eta_t\}_{t}$, warm-up time $T_{\rm w}$, \colorbox{thistle}{\!IRE hyperparams: $\kappa\geq 0$ and $\gamma\in (0.5,1)$\!};}}  
    \For{$t=T_{\rm w},\cdots,T-1$}{
        Compute the original update direction $\bg_t$ according to the base optimizer;
        \\
        \If {$(t-T_{\rm w})\!\!\mod K=0$}
            {
            Estimate the diagonal Hessian $\bh_t\in\RR^p$ using Eq.~\eqref{equ: fisher estimate};\\
            Update the mask $\bn_t\in\RR^p$ using Eq.~\eqref{eqn: proj};
            } 
        \Else{$\bn_{t}=\bn_{t-1}$}
        $\btheta_{t+1}=\btheta_t-\eta_t\big(\bg_t+$\colorbox{thistle}{$\kappa \bn_t\odot \bg_t$}$\big)$;
        }
    \KwOut{$\btheta_T$.}
\end{algorithm}
\vspace{-.2cm}

{\bf A light-weight estimator of the diagonal Hessian.}
Let $\ell(\cdot,\cdot)$ be the  cross-entropy loss. 
Given an input data $\bx\in\RR^{d_x}$ and label $\by\in\RR^{d_y}$, let the model's prediction be $f(\bx;\btheta)\in\RR^{d_y}$.
The Fisher (Gauss-Newton) matrix $F(\btheta)$ is widely acknowledged to be a good approximation of the Hessian, particularly near  minima. Thus, we can estimate the diagonal Hessian by $\bh_t={\rm diag}(F(\btheta_t))$, which has been widely used in deep learning optimization~\citep{martens2015optimizing,grosse2016kronecker,george2018fast,mi2022make,liu2023sophia}.
Given an input batch $\{(\bx_b,\by_b)\}_{b=1}^B$, the empirical diagonal Fisher is given by 
$
{\rm diag}(\hat{F}(\btheta))=\frac{1}{B}\sum_{b=1}^B\nabla\ell(f(\bx_b;\btheta);\hat{\by}_b)\odot\nabla\ell(f(\bx_b;\btheta);\hat{\by}_b), \text{ where }\hat{\by}_b\sim{\rm softmax}(f(\btheta;\bx_b)).
$ 
However, as noted by~\cite{liu2023sophia},
implementing this estimator is computationally expensive due to the need to calculate $B$ single-batch gradients.
\cite{liu2023sophia} proposed a more convenient estimator ${\rm diag}(\hat{F}_{\rm eff}(\btheta))$,  only requires computing the mini-batch gradient $\nabla\hat{\cL}_B(\btheta)=\frac{1}{B}\sum_{b=1}^B\nabla\ell(f(\bx_b;\btheta);\hat{\by}_b) \text{ with }\hat{\by}_b\sim{\rm softmax}(f(\bx_b;\btheta))$:
\begin{equation}\label{equ: fisher estimate}
\begin{aligned}
   \bh_t = {\rm diag}(\hat{F}_{\rm eff}(\btheta))=B\cdot\nabla\hat{\cL}_B(\btheta)\odot\nabla\hat{\cL}_B(\btheta).
\end{aligned}
\end{equation}
According to~\cite{liu2023sophia}, this estimator is an unbiased estimate of the empirical diagonal Fisher, i.e., $\bbE_{\hat{\by}}[{\rm diag}(\hat{F}_{\rm eff}(\btheta))]=\bbE_{\hat{\by}}[{\rm diag}(\hat{F}(\btheta))]$. For more discussions on the efficiency of this estimator, please refer to \cite[Section 2]{liu2023sophia}.
Additionally, for squared loss, one can simply use Fisher as the estimator~\citep{liu2023sophia}.

{\bf The practical IRE and computational efficiency.} The practical IRE is summarized in Algorithm~\ref{alg: ire}, which is notably lightweight. The estimation of $\bh_t$ using \eqref{equ: fisher estimate} requires computational resources roughly equivalent to one back-propagation.
Consequently, by setting $K=10$ in Algorithm~\ref{alg: ire} (estimating the projection every 10 steps), the average per-step computational load of IRE is {\em only $1.1$ times} that of the base optimizer. 
This claim can be empirically validated as shown in Table \ref{table: running time}.


\section{Experiments}
\label{section: experiments}

In this section, we evaluate how IRE performs when incorporating with various base optimizers. 
Specifically, we examine the incorporation of IRE with SGD (SGD-IRE),  SAM (SAM-IRE), and  AdamW (AdmIRE) across vision and language tasks. 

\subsection{Image classification}
\label{section: exp: CV}

\subsubsection{Validating our motivation} 
\label{section: exp: CV: validate theory}

To show that IRE can accelerate the sharpness reduction, 
we train WideResNet-16-8~\citep{zagoruyko2016wide} on CIFAR-10 dataset~\citep{krizhevsky2009learning} by SAM-IRE (with $K=10$, varying $\kappa$ and $\gamma$). 
Here, we incoporate IRE into SAM starting from the $30$-th epochs.
We vary $\gamma\in \{0.8,0.9,0.95\}$ and $\kappa\in \{0, 2,5,10\}$.
Regarding the learning rate (LR),  both constant LR  and decayed LR are considered.
The sharpness is measured by  $\Tr(\nabla^2\cL(\btheta))$.
Further experimental details can be found in Appendix~\ref{appendix: exp}.

As depicted in Fig.~\ref{fig: validation, cifar}(a), SAM-IRE (with constant LR) consistently finds flatter solutions compared to SAM and higher  $\kappa$ always leads to flatter minima. Additionally, SAM-IRE also shows robustness to variations of $\gamma$.
For SAM-IRE with decayed LR (Fig.~\ref{fig: validation, cifar}(b)), SAM-IRE still consistently finds flatter solutions than SAM.
Notably, flatter solutions correlate positively with lower training loss and higher test accuracy (Fig.~\ref{fig: validation, cifar}(c,d)).

\vspace{-.4cm}
\begin{figure}[!ht]
    \centering
    \subfloat[\small sharpness, constant LR.]{
    \includegraphics[width=3.4cm]{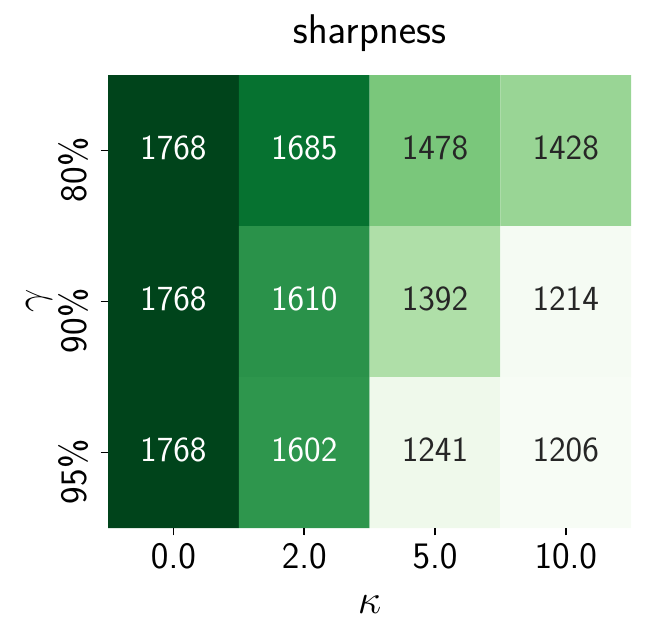}}
    \subfloat[\small sharpness, decayed LR.]{
    \includegraphics[width=3.4cm]{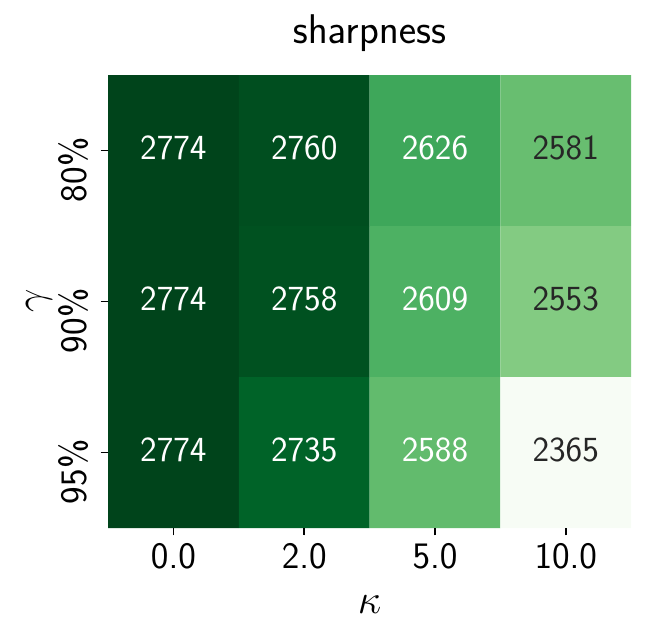}}
    \subfloat[\small train loss, decayed LR.]{
    \includegraphics[width=3.4cm]{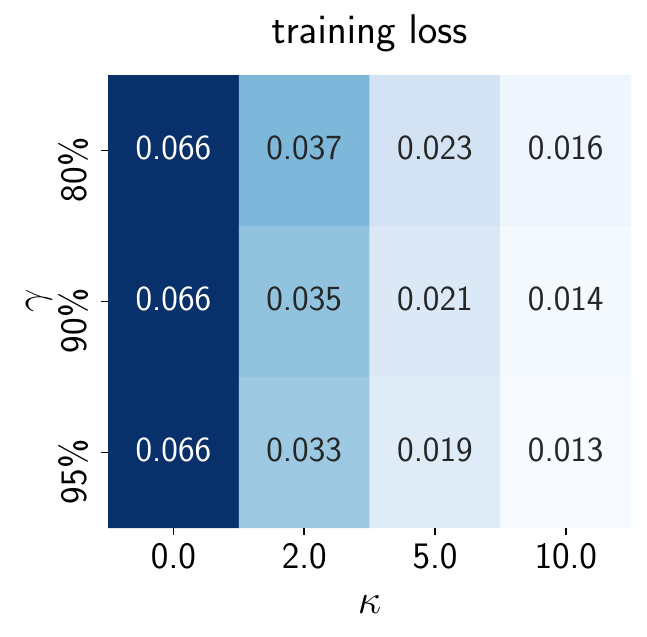}}
    \subfloat[\small test acc, decayed LR.]{
    \includegraphics[width=3.4cm]{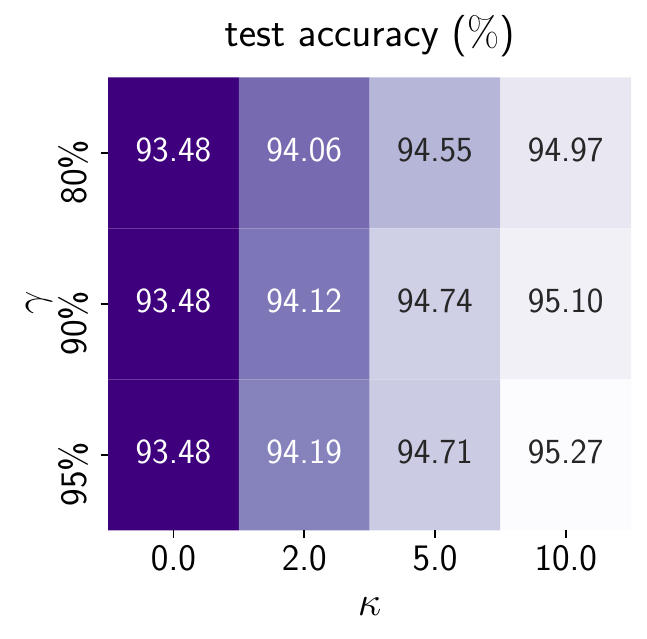}}
    \vspace{-.1cm}
    \caption{\small Training WRN-16-8 on CIFAR-10 by SAM-IRE with varying $\gamma,\kappa$. Particularly, the case of $\kappa=0$ correspond to the standard SAM.
    }
    \label{fig: validation, cifar}
\end{figure}

\subsubsection{IRE can consistently improve generalization}
\label{section: exp: CV: default results}

{\bf Convolutional Neural Networks (CNNs).} In this experiment, we first consider the classification of   CIFAR-\{10,100\} with WideResNet-28-10~\citep{zagoruyko2016wide} and ResNet-56~\citep{he2016deep}. Both SGD and SAM optimizers are adopted. All the experiments use base data augmentation and label smoothing.
For SGD-IRE/SAM-IRE, we fix $K=10$, and tune hyperparameters $\gamma$ and $\kappa$ via a grid search over
$\gamma\in\{0.99, 0.9, 0.8\}$ and $\kappa\in\{1,2\}$. 
The total epochs are set to 100 for CIFAR-10 and 200 for CIFAR-100, and we switch from SGD/SAM to SGD-IRE/SAM-IRE when the training loss approaches 0.1.
The other experimental details are deferred to Appendix~\ref{appendix: exp} and the results are shown in Table~\ref{table: resnet cifar}.

Secondly, we evaluate IRE for training ResNet-50 on ImageNet~\citep{deng2009imagenet}. The experimental details are deferred to Appendix~\ref{appendix: exp} and the results are shown in Table~\ref{table: resnet imagenet}.

{\bf Vision Transformers (ViTs).} We also examine the impact of IRE on generalization of ViT-T and ViT-S~\citep{dosovitskiy2020image} on CIFAR100. 
The default optimizers used are AdamW and SAM~\citep{mueller2024normalization}.
Strong data augmentations (basic + AutoAugment) are utilized.
The total epochs are set to $200$ and we switch from AdamW/SAM to AdmIRE/SAM-IRE when the training loss approaches $0.5$. 
For AdmIRE/SAM-IRE, we fix $K=10$, and tune hyperparameters $\gamma$ and $\kappa$ via a grid search over $\gamma\in\{0.99, 0.9, 0.8\}$ and $\kappa\in\{20,50\}$. 
Other experimental details are deferred to Appendix~\ref{appendix: exp}. 
The results are shown in Table~\ref{table: vit cifar}.

Additionally, we evaluate IRE for training ViT-S on ImageNet. The experimental details are deferred to Appendix~\ref{appendix: exp} and the results are shown in Table~\ref{table: vit imagenet}.

As demonstrated in Table~\ref{table: resnet cifar},~\ref{table: resnet imagenet},~\ref{table: vit cifar}, and~\ref{table: vit imagenet}, IRE {\em consistently improves} generalization of  SGD, AdamW and  SAM across all settings examined.

\vspace{-.1cm}
\begin{table}[!ht]
    \centering
    \hspace{-1.2cm}
    \begin{minipage}[t]{0.7\textwidth}
    \small
    \caption{\small WRN-28-10/ResNet-56 on CIFAR-10/100.}
    \begin{tabular}{c||c|c||c|c}
    \hline\hline
    \multirow{2}*{~} & \multicolumn{2}{c||}{ WRN-28-10} & \multicolumn{2}{c}{ ResNet-56} \\ \cline{2-5}
     &   CIFAR-10  &   CIFAR-100  &   CIFAR-10  &   CIFAR-100\\\hline
      SGD &   95.84 &    80.81 &   93.49 &   72.81
    \\ 
     SGD-IRE &  {\bf 96.24} ({\bf +0.40}) &  {\bf 81.49} ({\bf +0.68}) &  {\bf 93.78} ({\bf +0.29}) &  {\bf 73.78} ({\bf +0.97})
    \\
    \hline
     SAM &  96.58 &  83.05 &  94.05 &  75.54
    \\ 
     SAM-IRE &  {\bf 96.70} ({\bf +0.12}) &  {\bf 83.50} ({\bf +0.45}) &  {\bf 94.46} ({\bf +0.41}) &  {\bf 75.86} ({\bf +0.32})
    \\ \hline\hline
    \end{tabular}
    \label{table: resnet cifar}
    \end{minipage}
\end{table}

\vspace{-.2cm}
\begin{table}[!ht]
    \centering
    \begin{minipage}[t]{0.5\textwidth}
    \small
    \caption{\small ResNet-50 on ImageNet.}
    \begin{tabular}{c||c|c}
    \hline\hline
     \multirow{2}*{~}
     &  Top-1  &  Top-5\\\hline
     SGD &   76.81 &  93.31 
    \\ 
     SGD-IRE &  {\bf 77.04} ({\bf +0.23}) &  {\bf 93.58} ({\bf +0.27}) 
    \\
    \hline
     SAM &  77.47  &   93.90
    \\ 
     SAM-IRE &  {\bf 77.92} ({\bf +0.45}) &  {\bf 94.12} ({\bf +0.22})
    \\ \hline\hline
    \end{tabular}
    \label{table: resnet imagenet}
    \end{minipage}
    \hspace{-.5cm}
    \begin{minipage}[t]{0.5\textwidth}
    \small
    \caption{\small ViT-T/S on CIFAR-100.}
    \begin{tabular}{c||c|c}
    \hline\hline
     \multirow{2}*{~}
     &  ViT-T  &  ViT-S\\\hline
     AdamW &   63.90 &  65.43
    \\ 
     AdmIRE &  {\bf 67.05} ({\bf +3.15}) &  {\bf 68.39} ({\bf +2.96}) 
    \\
    \hline
     SAM &  64.25  &   66.93
    \\ 
     SAM-IRE &  {\bf 67.33} ({\bf +3.08}) &  {\bf 70.47} ({\bf +3.54})
    \\ \hline\hline
    \end{tabular}
    \label{table: vit cifar}
    \end{minipage}
\end{table}

\newpage

\begin{table}[!ht] 
    \small
    \caption{\small ViT-S on ImageNet.}
    \centering
    \vspace{.1cm}
    \begin{tabular}{c||c|c}
        \hline\hline
          & Top-1 & Top-5 \\ \hline
          AdamW &  78.7 &  94.0  \\ 
          AdmIRE ($\kappa=2,\gamma=0.6$) & {\bf 79.0} ({\bf +0.3})  & {\bf 94.3} ({\bf +0.3})   \\ 
          AdmIRE ($\kappa=2,\gamma=0.8$) & {\bf 79.1} ({\bf +0.4}) &  {\bf 94.2} ({\bf +0.2})    \\\hline\hline
    \end{tabular}
    \label{table: vit imagenet}
\end{table}

\subsection{Large language model pre-training}
\label{section: exp: NLP}

We now evaluate IRE in the pre-training of decoder-only large language models (LLMs).
Following the training protocol of Llama models, we employ the AdamW optimizer with hyperparameters $\beta_1=0.9,\beta_2=0.95$ and weight decay $\lambda=0.1$~\citep{touvron2023llama}. The learning rate strategy includes a warm-up phase followed by a cosine decay scheduler, capped at \texttt{lr\_max}.
In each experiment, we tune \texttt{lr\_max} only for AdamW and use  it also for AdmIRE, for which the IRE is activated at the end of warm-up phase.



\vspace{-.2cm}
\subsubsection{Computational efficiency and hyperparameter robustness}
\label{section: exp: NLP: robustness}
\vspace{-.1cm}

The first experiment is conducted to verify both the computational efficiency and the robustness of hyperparameters ($\gamma,\kappa$) in IRE for pre-training tasks.
Specifically, we train a 2-layer decoder-only Transformer (8M) on the Wikitext-2 dataset (4.3M)~\citep{merity2016pointer} by AdamW and AdmIRE (with $K=10$ and varying $\gamma,\kappa$).
The total training duration is 100k steps, including a 3k-step warm-up phase.

\begin{wraptable}{r}{0.35\linewidth}
    \centering
    \caption{\small Wall-clock time on 1 A800.}
    \vspace{-.1cm}
    \begin{tabular}{c|c}
    \hline\hline
    \small
    Algorithm &  \small time (/step) \\\hline
    \small AdamW &  \small 0.165s \\\hline
    \small AdmIRE &  \small 0.185s \\ \hline\hline
    \end{tabular}
    \label{table: running time}
\end{wraptable}
\vspace{.1cm}
First, we tune \texttt{lr\_max} in AdamW, identifying the optimal \texttt{lr\_max}=\texttt{6e-4}. 
Subsequently, we train both AdamW and AdmIRE using this  \texttt{lr\_max}.

\vspace{-.25cm}
\paragraph{Computational efficiency.} As shown in Table~\ref{table: running time}, AdmIRE with $K=10$ (estimating the projection mask every 10 steps) is computationally efficient: the average time per step of AdmIRE is only $1.12$ times that of AdamW, corresponding to the theoretical estimation ($1.1$ times). 

\begin{wrapfigure}{r}{0.35\linewidth}
    \vspace{-1.cm}
    \includegraphics[width=5cm]{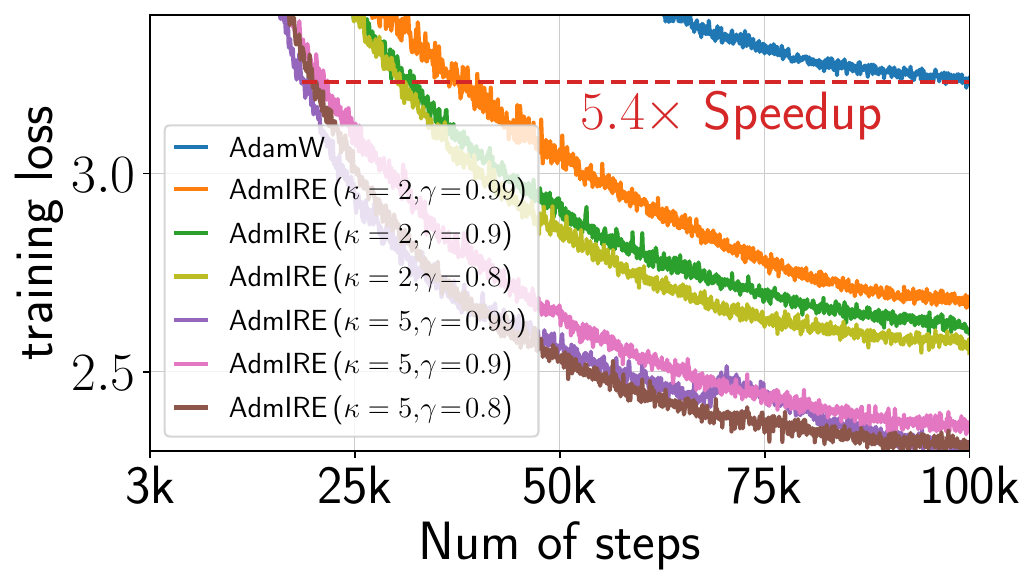}
    \vspace{-.55cm}
    \caption{\small Transformer on wikitext-2.}
    \label{fig: wiki2}
    \vspace{-.6cm}
\end{wrapfigure}
\vspace{-.25cm}
\paragraph{Robustness to hyperparameters.} Figure~\ref{fig: wiki2} shows that AdmIRE, with varying $\gamma$ and $\kappa$, consistently speeds up the pre-training. 
Remarkably, with the best configuration, AdmIRE can achieves {\color{violet}{\bf 5.4$\times$ speedup}} than well-tuned AdamW.

More experimental details and results are deferred to Appendix~\ref{appendix: exp}.




\subsubsection{Experiments on Llama models}
\label{section: exp: NLP: Llama}
\vspace{-.1cm}

Llama~\citep{touvron2023llama}, a popular open LLM, exhibits remarkable capabilities across general domains. In this section, we examine the performance of AdmIRE in training Llama models of various sizes across various  datasets:

\vspace{-.1cm}
\begin{itemize}[leftmargin=2em]
\item {\bf Llama (60M) on wikitext-103 (0.5G).}
{Wikitext-103}~\citep{merity2016pointer} serves as a standard language modeling benchmark for pre-training, which contains 103M training tokens from 28K articles, with an average length of 3.6K tokens per article. 

\vspace{-.05cm}
\item {\bf Llama (119M) on minipile (6G).}
{Minipile}~\citep{kaddour2023minipile}, a 6GB subset of the deduplicated Pile (825GB)~\citep{gao2020pile} presents a highly diverse text corpus. 
Given its diversity, training on minipile poses more challenges and potential instabilities for optimizers compared to Wikitext-103.

\vspace{-.05cm}
\item  {\bf Llama (229M) on openwebtext (38G).}
{Openwebtext}~\citep{Gokaslan2019OpenWeb}, an open-source recreation of the WebText corpus, is extensively utilized for LLM  pre-training  such as RoBERTa~\citep{liu2019roberta} and GPT-2~\citep{radford2019language}. 

\end{itemize}

\vspace{-.1cm}
Additionally, gradient clipping is adopted  to maintain the training stability~\citep{pascanu2012understanding}. 
First, we tune \texttt{lr\_max} in AdamW for each of the three experiments, separately.
The optimal \texttt{lr\_max} identified for these three experiments is all \texttt{6e-4}.
Then, both AdamW and AdmIRE are trained using this optimal \texttt{lr\_max}. For more details, please refer to Appendix~\ref{appendix: exp}.

{\bf AdmIRE is {\color{violet}{$2\times$ faster}} than AdamW}. The results are reported in Figure~\ref{fig: llama}. We can see that AdmIRE consistently achieves a $2.1\times$ speedup compared with well-tuned AdamW for all three cases.

\begin{figure}[!ht]
\hspace*{-.5em}
\subfloat{
    \includegraphics[width=4.8cm]{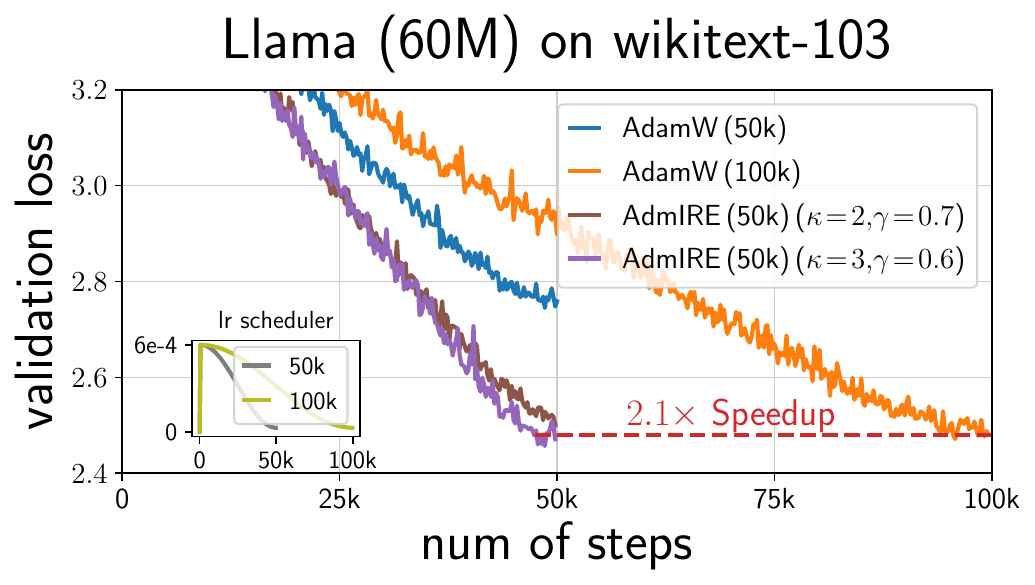}
}
\hspace*{-.5em}
\subfloat{
    \includegraphics[width=4.8cm]{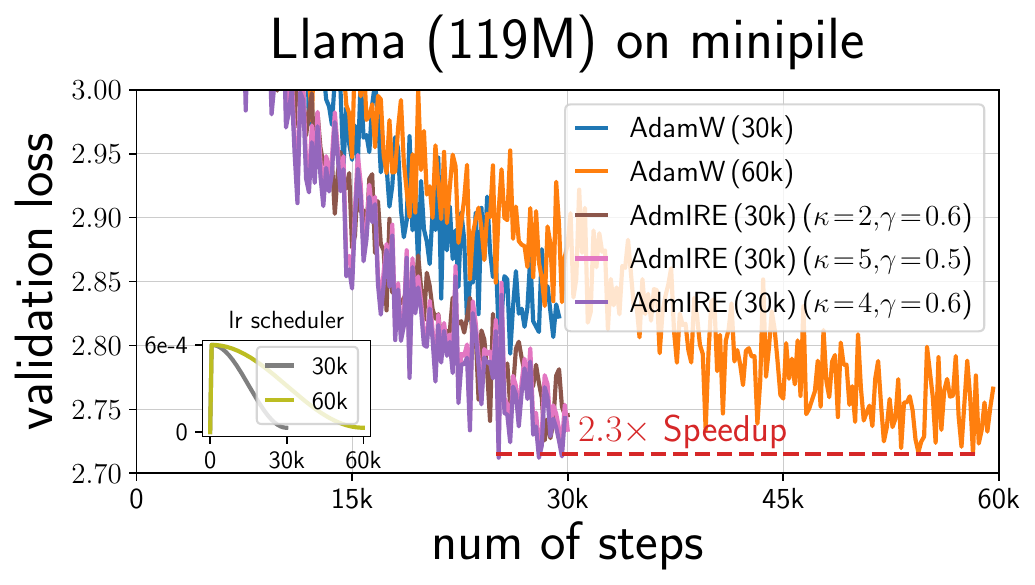}
}
\hspace*{-.5em}
\subfloat{
    \includegraphics[width=4.8cm]{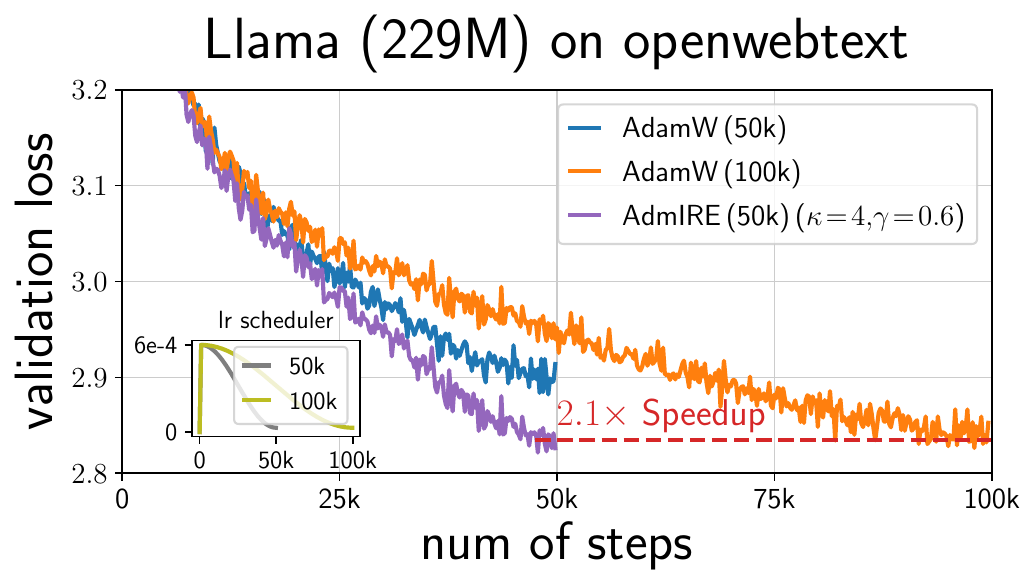}
}
\caption{AdmIRE outperforms AdamW in the pre-training of Llama models.}
\label{fig: llama}
\end{figure}

Notice that the primary motivation behind IRE is to speed up the sharpness reduction, which only requires to increase learning rate along completely flat (zero-curvature) directions. 
However, practical implementation may also increase the learning rate along directions with small but non-zero curvatures, which can further speed up loss convergence. 
A thorough explanation for the significant acceleration provided by this approach is left for future research.


\begin{wraptable}{r}{0.45\linewidth}
    \centering
    \caption{\small Comparison of the sharpness of the solutions found by AdamW/AdmIRE.}
    \begin{tabular}{c|c|c}
        \hline\hline
          & \small AdamW & \small AdmIRE \\ \hline
          \small training steps & \small 100k  &  \small 50k \\ \hline
        \small final $\cL(\btheta)$  & \small 2.47 & \small 2.47 \\ \hline
        \small final $\Tr(\nabla^2\cL(\btheta))$ & \small 120.41 & \small 88.86
            \\\hline\hline
    \end{tabular}
    \label{table: flatness NLP}
\end{wraptable}
We further assess the sharpness reduction capability of IRE for LLM pre-training.
Specifically, we compare the sharpness of solutions, $\Tr(\nabla^2\cL(\btheta))$, found by AdamW/AdmIRE during pre-training of Llama (60M) on wiki-103 dataset (corresponding to Figure~\ref{fig: llama} (left)).
The results shown in Table~\ref{table: flatness NLP} demonstrate that AdamIRE not only achieves the same loss in {\em only half the iterations} required by AdamW, but also the solutions found by AdmIRE are {\em significantly flatter} than that found by AdamW.

Recently,~\cite{liu2023same} revealed a strong correlation between the sharpness and downstream task performance, suggesting that for models with the same pre-training loss, flatter solutions yield better performance on downstream tasks. 
Based on this, we hypothesize that the solutions found by IRE may also have better performance in downstream tasks, which we leave to future work.

\section{Theoretical Guarantees for IRE on  SAMs}\label{section: theory}

\vspace{-.2cm}
Both empirical~\citep{foret2020sharpness} and theoretical~\citep{wen2023how} studies have validated that SAM algorithms exhibit superior sharpness regularization compared to (S)GD.
In this section, we provide a theoretical analysis demonstrating that IRE can further enhance the sharpness regularization of SAM algorithms substantially.


\subsection{Theoretical setups}

Recall that  $\cL(\btheta):=\frac{1}{n}\sum_{i=1}^n\cL_i(\btheta)$ denote the total loss, where $\cL_i(\btheta)$ is the loss on the $i$-th data. Without loss of generality, we assume $\min_{\btheta}\cL(\btheta)=0$. We further make the  following assumption:

\begin{assumption}[Manifold of minimizers]\label{ass: minima manifold}
    Assume that $\cL\in \cC^4(\RR^p)$, $\cM:=\argmin_{\btheta}\cL(\btheta)$ is a ($p-m$)-dim $\cC^2$-submanifold in $\bbR^p$ for some  $m\in[p]$, and   $\rank(\nabla^2\cL(\btheta))=m$ for any $\btheta\in\cM$.
\end{assumption}

\vspace{-.1cm}
The above connectivity assumption on the manifold of minimizers $\cM$ has been empirically verified in works such as \cite{draxler2018essentially} and \cite{garipov2018loss}, and theoretically supported in~\cite{cooper2018loss}.
This assumption is also widely used in the theoretical analysis of implicit regularization~\citep{fehrman2020convergence,li2021happens,arora2022understanding,wen2023how}.

Besides, we introduce the following definitions to characterize the dynamics of gradient flow (GF) near the minima manifold $\cM$, which is also used in the related works above.

\begin{definition}[Limiting map of GF]\label{def: limit gradient flow}
     Consider the GF: $\frac{\rd\btheta(t)}{\rd t}=-\nabla\cL(\btheta(t))$ starting from $\btheta(0)=\btheta$. Denote by  $\Phi(\btheta):=\lim_{t\to\infty}\btheta(t)$ the limiting map of this GF.
\end{definition}

\begin{definition}[Attraction set of $\cM$]\label{def: attraction set}
    Let $U$ be the attraction set of $\cM$ under GF, i.e., GF starting in $U$ converges to some point in $\cM$. Formally, $U:=\{\btheta\in\bbR^p:\Phi(\btheta)\in\cM\}$.    
\end{definition}

\vspace{-.1cm}
As proven in \cite[Lemma B.15]{arora2022understanding}, Assumption~\ref{ass: minima manifold} ensures that $U$ (in Definition~\ref{def: attraction set}) is open and $\Phi(\cdot)$ (in Definition~\ref{def: limit gradient flow}) is $\cC^2$ on $U$.

\subsection{Theoretical results}


The  stochastic SAM~\citep{foret2020sharpness} is given by 
\begin{equation}\label{equ: SAM, worst}
    \text{standard SAM:}\quad \btheta_{t+1}=\btheta_t-\eta\nabla\cL_{i_t}\left(\btheta_t+\rho\frac{\nabla\cL_{i_t}(\btheta_t)}{\norm{\nabla\cL_{i_t}(\btheta_t)}}\right),\text{ where }i_t\sim\bbU([n]).
\end{equation}
The generalization capability of standard SAM can be bounded by the average sharpness, $\cL^{\rm avg}(\btheta):=\bbE_{\bxi\sim\cN(\bzero,I)}\cL\left(\btheta+\rho{\bxi}/{\norm{\bxi}}\right)$~\citep{foret2020sharpness}. 
This leads researchers to also explore average SAM~\citep{wen2023how,zhu2023decentralized,ujvary2022rethinking}, which minimizes $\cL^{\rm avg}$:
\begin{equation}\label{equ: SAM, average}
    \text{average SAM:}\quad\btheta_{t+1}=\btheta_t-\eta\nabla\cL\left(\btheta_t+\rho{\bxi_t}/{\norm{\bxi_t}}\right),\text{ where }\bxi_t\sim\cN(\bzero,I).
\end{equation}

\textbf{Two-phase algorithms.} Our theoretical focus is on how IRE accelerates the sharpness reduction of SAM~\eqref{equ: SAM, worst} and \eqref{equ: SAM, average} {\em near the minima manifold} $\cM$. Thus, we analyze the two-phase algorithms.
Specifically, let the initialization $\btheta_0\in U$. 
In {\bf Phase I} ($t\leq T_{\rm I}$), we employ GF $\frac{\rd \btheta_t}{\rd t}=-\nabla\cL(\btheta_t)$ to ensure that the loss decreases sufficiently;
then in {\bf Phase II} ($T_{\rm I}<t\leq T_{\rm I}+T_{\rm II}:=T$), we incorporate IRE into the standard / average SAM.

\begin{wrapfigure}{r}{0.3\linewidth}
    \centering
    \vspace*{-.cm}
    \includegraphics[width=4.5cm]{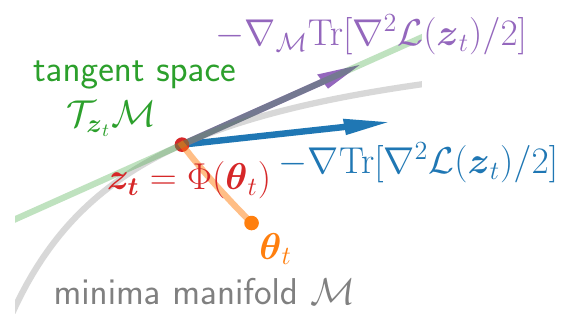}
    \vspace*{-2em}
\end{wrapfigure}
\textbf{Effective dynamics: sharpness regularization.} The implicit regularization of SAMs can be modeled  using effective dynamics. In Phase II, $\btheta_t$ are close the manifold of minimizers $\cM$ and let $\bz_t:=\Phi(\btheta_t)\in\cM$.
Then, the effective dynamics is given by $\{\bz_{t}\}_{t=T_{\rm I}+1}^T$, revealing how SAMs explore the manifold of minimizers $\cM$. Particularly, \cite{wen2023how} showed that the effective dynamics of standard/average SAM are both  
\begin{equation}\label{eqn: eff-dy}
\bbE\left[\bz_{t+1}\right]=\bz_t -\eta_{\rm eff} \nabla_{\cM} \mathrm{Tr}\left[\nabla^2 \cL(\bz_t)/2\right]+o(\eta_{\rm eff}), 
\end{equation}
which minimizes the trace of Hessian on $\cM$. 
The difference between the standard SAM \eqref{equ: SAM, worst} and average SAM~\eqref{equ: SAM, average} lies in the effective learning rate (LR) $\eta_{\rm eff}$'s. A visual illustration of some quantities in~\eqref{eqn: eff-dy} is provided in the figure above.


\textbf{Summary of our theoretical results.} 
In this section, we show that incorporating IRE into SAMs can significantly increase the effective LR $\eta_{\rm eff}$ in \eqref{eqn: eff-dy} while maintaining the same training stability as SAMs.
In Table \ref{tab: theoretical results}, we present the effective LR for SAMs and the SAM-IREs. 
We see clearly that IRE can accelerate the sharpness reduction by a non-trivial factor for both standard and average SAM.

\begin{table}[!ht]
\centering
\vspace{-0.1cm}
    \caption{Comparison of the implicit regularization strength of SAMs w/o IRE.}
    \begin{tabular}{c|c}
    \hline\hline
    Algorithm & Effective LR: $\eta_{\rm eff}$ \\\hline
    average SAM~\eqref{equ: SAM, average} & $\eta^2/p$ (Thm~\ref{thm: Phase II, ave SAM-IRE}) \\\hline
    IRE + average SAM~\eqref{equ: SAM, average} & $\eta^{1.5}/p$ (Thm~\ref{thm: Phase II, ave SAM-IRE})\\\hline
    standard SAM~\eqref{equ: SAM, worst} &$\eta\rho^2$ (Thm~\ref{thm: Phase II, std SAM-IRE}; \cite{wen2023how}) \\\hline
    IRE + standard SAM~\eqref{equ: SAM, worst} &$\eta\rho$ (Thm~\ref{thm: Phase II, std SAM-IRE}) \\ \hline\hline
    \end{tabular}
    \label{tab: theoretical results}
\end{table}

\begin{remark}[The mechanism of IRE's success]
The success of SAM-IRE follows  the same mechanism  illustrated in Section \ref{section: example}. 
The key fact that IRE only  increases the LR along flat directions has two implications: 
1) It does not change the trend of implicit regularization in Eq.~\eqref{eqn: eff-dy} but accelerates SAMs' effective dynamics by a factor of $(1+\kappa)$; 
2) Since the LR is only increased along flat directions, $\kappa$ can be set substantially large without hurting the training stability, because the dynamics in sharp directions remain unchanged.  Specifically, we theoretically justify in SAM-IRE, $\kappa$ can be selected as large as $1/\rho$.
\end{remark}



\subsubsection{IRE on average SAM: An \texorpdfstring{$\Omega(1/\eta^{0.5})$}{} acceleration}
\label{subsec: theory: average SAM}

We first consider IRE on average SAM. Let  $T_{\rm I}$ be the hitting time: $T_{\rm I}:=\inf\{t\geq0:\norm{\btheta_t-\Phi(\btheta_t)}=\cO(\sqrt{\eta}\rho)\}$. 
When running GF in Phase I, Definition~\ref{def: attraction set} guarantees $T_{\rm I}<\infty$. 
Thus, at the starting of Phase II, $\norm{\btheta_t-\Phi(\btheta_t)}=\cO(\sqrt{\eta}\rho)$. 
Furthermore, the following result holds for Phase II.

\begin{theorem}[IRE on average SAM]\label{thm: Phase II, ave SAM-IRE} Suppose Assumption~\ref{ass: minima manifold} holds. If $\eta=\cO(1)$ and $\rho=\cO(\sqrt{\eta})$ in SAM~\eqref{equ: SAM, average}, \colorbox{thistle}{\!$\kappa\leq 1/\rho$\!}, and $\cP_t=P_{m+1:p}(\nabla^2\cL(\btheta_t))$ in IRE \eqref{eqn: ire-genetric}, then with high probability at least $1-T_{\rm II}^2\exp\left(-\Omega\left(1/\left(\eta+p^{-1}\right)\right)\right)$, $\norm{\btheta_t-\Phi(\btheta_t)}=\cO(\sqrt{\eta}\rho)$ holds for all $T_{\rm I}\leq t\leq T$. 
Furthermore, the effective dynamics of $\bz_t:=\Phi(\btheta_t)\in\cM$ satisfies:
\begin{align*}
    \bbE_{\bxi_t}[\bz_{t+1}]=\bz_{t}-\frac{{\color{violet}(1+\kappa)}\eta\rho^2}{p}\nabla_{\cM}\Tr\left[\nabla^2\cL(\bz_{t})/2\right]+\cO(\eta^{3/2}\rho^2).
\end{align*}
\end{theorem}

Note that $\rho=\cO(\sqrt{\eta})$ and  $\kappa$ can be as large as $1/\rho$. Consequently, the effective LR of minimizing the trace of Hessian can be selected as large as $\eta_{\rm eff}=(\kappa+1)\eta\rho^2/p=\cO(\eta^{1.5}/p)$. In contrast, that of average SAM is at most $\cO(\eta^2/p)$. The proof of Theorem \ref{thm: Phase II, ave SAM-IRE} can be found in Appendix \ref{appendix: average SAM}.



\subsubsection{IRE on standard SAM: An \texorpdfstring{$\Omega(1/\rho)$}{} acceleration}
\label{subsec: theory: standard SAM}

This subsection delves into IRE on standard SAM~\eqref{equ: SAM, worst}, which is more widely used and often yields better performance than average SAM~\eqref{equ: SAM, average}. 
However, since standard SAM~\eqref{equ: SAM, worst} requires stochastic gradients $\nabla\cL_i(\btheta)$ ($i\in[n]$), we need an additional assumption regarding the features on the manifold (see Setting~\ref{setting: erm problem}), which is commonly used in the literature~\citep{du2018gradient,du2019gradient,li2021happens,arora2022understanding,wen2023how}. 
We defer it to Appendix~\ref{appendix: standard SAM} due to space constraints.
Under this Setting, Assumption~\ref{ass: minima manifold} holds naturally with $m=n$.

During Phase I of GF, Definition~\ref{def: attraction set} ensures that there exists $t<\infty$ such that $\norm{\btheta_t-\Phi(\btheta_t)}=\cO(\eta^{1-\alpha}\rho)$ for any $\alpha\in[0,1)$.
We define $T_{\rm I}$ as the hitting time: $T_{\rm I}:=\inf\{t\geq0:\norm{\btheta_t-\Phi(\btheta_t)}=\cO(\eta^{1-\alpha}\rho)\}$. 
Then the following result holds for Phase II, whose proof can be founded in Appendix~\ref{appendix: standard SAM}.

\begin{theorem}[IRE on standard SAM]\label{thm: Phase II, std SAM-IRE} Under Setting~\ref{setting: erm problem}, if $\eta,\rho=\cO(1)$ in SAM~\eqref{equ: SAM, worst},  \colorbox{thistle}{\!$\kappa\leq 1/\rho$\!}, and $\cP_t=P_{n+1:p}(\nabla^2\cL(\btheta_t))$ in IRE \eqref{eqn: ire-genetric}, then with high probability at least $1-T_{\rm II}^2\exp\left(-\Omega(1/\eta^\alpha)\right)$, $\norm{\btheta_t-\Phi(\btheta_t)}=\cO(\eta^{1-\alpha}\rho)$ holds for all $T_{\rm I}\leq t\leq T$. Moreover, the effective dynamics $\bz_t=\Phi(\btheta_t)\in\cM$ satisfies:
\begin{align*}
    \bbE_{i_t}[\bz_{t+1}]=\bz_t-{\color{violet}(1+\kappa)}\eta\rho^2\nabla_{\cM}\Tr\left[\nabla^2\cL(\bz_t)/2\right]+\cO\left((\kappa+1)\eta\rho^2(\rho+\eta^{1-\alpha})\right).
\end{align*}
\end{theorem}
Taking $\kappa=0$ and $\alpha=0$ recovers the result established in \cite{wen2023how}. However, $\kappa$ can be as large as $1/\rho$, where IRE provides a $\Theta(1/\rho)$-time acceleration over the standard SAM.

\section{Conclusion}
\label{section: discussion}

In this work, we propose a novel IRE framework to enhance the implicit sharpness regularization of base optimizers. Experiments demonstrate that IRE not only consistently improves generalization but also accelerates loss convergence in the pre-training of Llama models of various sizes. 
The code is available at \texttt{\href{https://github.com/wmz9/IRE-algorithm-framework}{https://github.com/wmz9/IRE-algorithm-framework}}.

For future work, there are two urgent directions: 1) understanding why IRE can accelerate convergence, which may require studying the interplay between IRE and the Edge of Stability (EoS)~\citep{wu2018sgd, jastrzkebski2017three, cohen2021gradient}; and 2) conducting a larger-scale investigation into the acceleration of AdmIRE compared to AdamW in LLM pre-training, as well as the downstream performance of the LLMs pre-trained by AdmIRE.

\section*{Acknowledgments}
Lei Wu is supported by the National Key R\&D Program of China (No.~2022YFA1008200) and National Natural Science Foundation of China (No.~12288101).
Mingze Wang is supported in part by the National Key Basic Research Program of China (No.~2015CB856000).  
We thank Dr. Hongkang Yang, Liu Ziyin, Liming Liu, Zehao Lin, Hao Wu, and Kai Chen for helpful discussions and anonymous reviewers for their valuable suggestions.

\newpage

\appendix

\begin{center}
    \noindent\rule{\textwidth}{4pt} \vspace{-0.2cm}
    \LARGE \textbf{Appendix} 
    \noindent\rule{\textwidth}{1.2pt}
\end{center}

\startcontents[sections]
\printcontents[sections]{l}{1}{\setcounter{tocdepth}{2}}

\input{appendix/main}

\newpage
\section*{NeurIPS Paper Checklist}

\begin{enumerate}

\item {\bf Claims}
    \item[] Question: Do the main claims made in the abstract and introduction accurately reflect the paper's contributions and scope?
    \item[] Answer: \answerYes{} 
    \item[] Justification: We believe that the abstract and introduction reflect the contributions and scope of the paper.

\item {\bf Limitations}
    \item[] Question: Does the paper discuss the limitations of the work performed by the authors?
    \item[] Answer: \answerYes{} 
    \item[] Justification: In Section~\ref{section: discussion}.

\item {\bf Theory Assumptions and Proofs}
    \item[] Question: For each theoretical result, does the paper provide the full set of assumptions and a complete (and correct) proof?
    \item[] Answer: \answerYes{} 
    \item[] Justification: In Section~\ref{section: example} and~\ref{section: theory}; Appendix~\ref{appendix: example},~\ref{appendix: average SAM},~\ref{appendix: standard SAM}, and~\ref{appendix: lemmas}.

    \item {\bf Experimental Result Reproducibility}
    \item[] Question: Does the paper fully disclose all the information needed to reproduce the main experimental results of the paper to the extent that it affects the main claims and/or conclusions of the paper (regardless of whether the code and data are provided or not)?
    \item[] Answer: \answerYes{} 
    \item[] Justification: We believe that all of the experimental results are reproducible in our work.

\item {\bf Open access to data and code}
    \item[] Question: Does the paper provide open access to the data and code, with sufficient instructions to faithfully reproduce the main experimental results, as described in supplemental material?
    \item[] Answer: \answerYes{} 
    \item[] Justification: In \texttt{\href{https://github.com/wmz9/IRE-algorithm-framework}{https://github.com/wmz9/IRE-algorithm-framework}}.

\item {\bf Experimental Setting/Details}
    \item[] Question: Does the paper specify all the training and test details (e.g., data splits, hyperparameters, how they were chosen, type of optimizer, etc.) necessary to understand the results?
    \item[] Answer: \answerYes{} 
    \item[] Justification: In Section~\ref{section: experiments} and Appendix~\ref{appendix: exp}.

\item {\bf Experiment Statistical Significance}
    \item[] Question: Does the paper report error bars suitably and correctly defined or other appropriate information about the statistical significance of the experiments?
    \item[] Answer: \answerYes{} 
    \item[] Justification: In Appendix~\ref{appendix: exp}.

\item {\bf Experiments Compute Resources}
    \item[] Question: For each experiment, does the paper provide sufficient information on the computer resources (type of compute workers, memory, time of execution) needed to reproduce the experiments?
    \item[] Answer: \answerYes{} 
    \item[] Justification: In Appendix~\ref{appendix: exp}.
    
\item {\bf Code Of Ethics}
    \item[] Question: Does the research conducted in the paper conform, in every respect, with the NeurIPS Code of Ethics \url{https://neurips.cc/public/EthicsGuidelines}?
    \item[] Answer: \answerYes{} 
    \item[] Justification: We have confirmed that the research is conducted with the NeurIPS Code of Ethics.

\item {\bf Broader Impacts}
    \item[] Question: Does the paper discuss both potential positive societal impacts and negative societal impacts of the work performed?
    \item[] Answer: \answerNA{} 
    \item[] Justification: \answerNA{}
    
\item {\bf Safeguards}
    \item[] Question: Does the paper describe safeguards that have been put in place for responsible release of data or models that have a high risk for misuse (e.g., pretrained language models, image generators, or scraped datasets)?
    \item[] Answer: \answerNA{} 
    \item[] Justification: \answerNA{}

\item {\bf Licenses for existing assets}
    \item[] Question: Are the creators or original owners of assets (e.g., code, data, models), used in the paper, properly credited and are the license and terms of use explicitly mentioned and properly respected?
    \item[] Answer: \answerNA{} 
    \item[] Justification: \answerNA{}

\item {\bf New Assets}
    \item[] Question: Are new assets introduced in the paper well documented and is the documentation provided alongside the assets?
    \item[] Answer: \answerNA{} 
    \item[] Justification: \answerNA{}

\item {\bf Crowdsourcing and Research with Human Subjects}
    \item[] Question: For crowdsourcing experiments and research with human subjects, does the paper include the full text of instructions given to participants and screenshots, if applicable, as well as details about compensation (if any)? 
    \item[] Answer: \answerNA{} 
    \item[] Justification: \answerNA{}

\item {\bf Institutional Review Board (IRB) Approvals or Equivalent for Research with Human Subjects}
    \item[] Question: Does the paper describe potential risks incurred by study participants, whether such risks were disclosed to the subjects, and whether Institutional Review Board (IRB) approvals (or an equivalent approval/review based on the requirements of your country or institution) were obtained?
    \item[] Answer: \answerNA{} 
    \item[] Justification: \answerNA{}

\end{enumerate}

\end{document}

%% file: math_commands.tex
\usepackage{amsmath,amsfonts,bm}

\def\1{\bm{1}}

\def\rd{{\textnormal{d}}}

\DeclareMathAlphabet{\mathsfit}{\encodingdefault}{\sfdefault}{m}{sl}
\SetMathAlphabet{\mathsfit}{bold}{\encodingdefault}{\sfdefault}{bx}{n}

\newcommand{\rank}{\mathrm{rank}}

\DeclareMathOperator*{\argmin}{arg\,min}

\DeclareMathOperator{\Tr}{Tr}
\DeclareMathOperator{\diag}{diag}

\newcommand{\trinorm}[1]{{\left\vert\kern-0.25ex\left\vert\kern-0.25ex\left\vert #1 
   \right\vert\kern-0.25ex\right\vert\kern-0.25ex\right\vert}}

\newcommand{\fb}{\bm}
\newcommand{\ba}{\fb{a}}

\newcommand{\be}{\fb{e}}

\newcommand{\bg}{\fb{g}}
\newcommand{\bh}{\fb{h}}

\newcommand{\bn}{\fb{n}}

\newcommand{\bu}{\fb{u}}
\newcommand{\bv}{\fb{v}}

\newcommand{\bx}{\fb{x}}
\newcommand{\by}{\fb{y}}
\newcommand{\bz}{\fb{z}}
\newcommand{\bA}{\fb{A}}
\newcommand{\bB}{\fb{B}}

\newcommand{\bD}{\fb{D}}
\newcommand{\bE}{\fb{E}}
\newcommand{\bF}{\fb{F}}

\newcommand{\bI}{\fb{I}}

\newcommand{\bU}{\fb{U}}
\newcommand{\bV}{\fb{V}}

\newcommand{\bmu}{\bm{\mu}}

\newcommand{\btheta}{\fb{\theta}}
\newcommand{\blambda}{\bm{\lambda}}

\newcommand{\bxi}{\bm{\xi}}

\newcommand{\bLambda}{\bm{\Lambda}}

\newcommand{\bGamma}{\fb{\Gamma}}

\newcommand{\cC}{\mathcal{C}}

\newcommand{\cL}{\mathcal{L}}
\newcommand{\cM}{\mathcal{M}}
\newcommand{\cN}{\mathcal{N}}
\newcommand{\cO}{\mathcal{O}}
\newcommand{\cP}{\mathcal{P}}
\newcommand{\cQ}{\mathcal{Q}}

\newcommand{\cS}{\mathcal{S}}
\newcommand{\cT}{\mathcal{T}}

\newcommand{\bbB}{\mathbb{B}}

\newcommand{\bbE}{\mathbb{E}}

\newcommand{\bbN}{\mathbb{N}}

\newcommand{\bbP}{\mathbb{P}}

\newcommand{\bbR}{\mathbb{R}}
\newcommand{\RR}{\mathbb{R}}
\newcommand{\bbS}{\mathbb{S}}

\newcommand{\bbU}{\mathbb{U}}

\newcommand{\bzero}{\mathbf{0}}
\newcommand{\bone}{\mathbf{1}}
\newcommand{\pll}{\kern 0.56em/\kern -0.8em /\kern 0.56em}

\newcommand{\norm}[1]{\ensuremath{\left\| #1 \right\|}}
\newcommand{\bracket}[1]{\ensuremath{\left( #1 \right)}}

\renewcommand{\vec}{\textup{\mbox{vec}}}
\newcommand{\<}{\left\langle}
\renewcommand{\>}{\right\rangle}


%% file: appendix/main.tex
\input{appendix/related_work}

\input{appendix/proof_SGD}

\input{appendix/experiment}

\input{appendix/proof_SAM}

\input{appendix/proof_basic}

%% file: appendix/related_work.tex
\vspace{1.cm}

\section{Other Related Works}
\label{appendix: works}

{\bf Other Implicit Biases.} 
Beyond the implicit sharpness regularization, numerous other attempts have explored implicit biases in deep learning algorithms~\cite{vardi2023implicit}.
Among these, a popular research line is the {\em max-margin bias}, which suggests that optimizers favor the solutions with large margin, which generalizes well.
\cite{soudry2018implicit} showed that GD converges to max-margin solutions under exponentially-tailed loss on linearly separable data, albeit with an extremely slow rate $\cO(1/\log t)$. 
Furthermore,~\citet{nacson2019stochastic} studied this bias for SGD, \citet{ji2018risk} investigated linearly non-separable data, and~\citet{ji2020gradient} analyzed the effects of the tail behavior of loss functions.

To achieve faster margin maximization,
\citet{nacson2019convergence,ji2021characterizing} demonstrated that GD with aggressively loss-scaled step sizes can achieve a faster polynomial rate of $\cO(1/t)$.
Building on this,~\citet{ji2021fast, wang2022accelerated} proposed momentum-based gradient methods which achieve a rate of $\cO(1/t^2)$.
Recently,~\cite{wang2023achieving} established that the polynomial rates for most previous algorithms are tight, and proposed a progressive rescaling gradient descent method that achieves margin maximization exponentially fast $\cO(e^{-\Omega(t)})$.

The margin-maximization bias has also been studied for nonlinear models.
~\citet{ji2018gradient, gunasekar2018implicit} examined deep linear networks, while~\citet{chizat2020implicit} focused on wide two-layer ReLU networks. 
Notably,~\citet{nacson2019lexicographic,lyu2019gradient,ji2020directional} demonstrated that for general homogeneous networks, Gradient Flow (GF) and GD converge to solutions corresponding the KKT point of the max-margin problem. Recently,~\citet{kunin2022asymmetric} extended this analysis to quasi-homogeneous networks. 
For two-layer (leaky-)ReLU neural networks, \citet{lyu2021gradient,vardi2022margin,wang2023understanding} examined whether the convergent KKT point of GF correspond to global optima of the max-margin problem.

Future work could investigate whether the IRE framework can also enhance the margin maximization bias, although it is primarily designed for enhancing the implicit sharpness regularization.

Additionally,~\cite{woodworth2020kernel,pesme2021implicit,nacson2022implicit,pesme2023saddle,even2023s} conducted fine-grained analyses of training dynamics, examining how initialization and step size impact (S)GD's minima selection in linear diagonal networks.

%% file: appendix/proof_SGD.tex
\vspace{1.cm}

\section{Proofs for SGD in Section~\ref{section: example}}
\label{appendix: example}

For the example in Section~\ref{section: example}, we have studied the implicit sharpness regularization of GD dynamics and how IRE enhances the implicit regularization of GD.
In this Section, we illustrate that, for this example, similar results hold for SGD dynamics.


{\bf SDE Modelling of SGD.} We consider SGD approximated by SDE~\citep{li2017stochastic,li2019stochastic,li2021happens} with noise covariance $\Sigma$: $\rd \btheta_t=-\nabla\cL(\btheta_t)\rd t+\sqrt{\eta\Sigma(\btheta_t)}\rd W_t$. 
We consider that the noise covariance aligns with the Hessian near minima, i.e.,  $\Sigma(\btheta)=\begin{pmatrix}
    \bzero_d & 0 \\ 0 & \sigma^2 h(\bu)
\end{pmatrix}$ (where $\sigma>0$ is a scalar), such as the label noise~\citep{damian2021label,li2021happens}.  Then, the SDE above can
be rewritten as
\begin{equation}\label{equ: sgd SDE}
\begin{aligned}
    \rd\begin{pmatrix}
    \bu_t \\ v_t
\end{pmatrix} = - \begin{pmatrix}
    v_t^2\nabla h(\bu_t)/2 \\ v_t h(\bu_t)
\end{pmatrix} \rd t +  \begin{pmatrix}
    0 \\ \sqrt{\eta\sigma^2 h(\bu_t)} \rd W_t
\end{pmatrix}.
\end{aligned}
\end{equation}

{\bf Implicit Sharpness Regularization.}
Intuitively,
when $v_t$ is close to $0$, the speed of $\bu_t$ is much slower than $v_t$ due to $v_t^2\ll v_t$. 
Following~\cite{ma2022beyond}, when this speed separation is large, $v_t$ is always at equilibrium given $\bu_t$. 
Solving the Ornstein–Uhlenbeck process about $v_t$, we know the equilibrium distribution of $v_t$ is $v_{\infty}\sim\cN(0,\eta\sigma^2)$,
and hence the dynamics $\bu_t$ (along the manifold) is
\begin{equation}\label{equ: sgd effective}
\begin{aligned}
  \rd\bu_t/\rd t=-\bbE_{v_\infty}\left[{v_{\infty}^2\nabla h(\bu_t)/2}\right]=-\nabla h(\bu_t)/2.
\end{aligned}
\end{equation}
This derivation clearly shows the slow ``effective dynamics'' $\bu_t$ along the manifold is a {\em gradient flow minimizing
the sharpness $h(\cdot)$}.
When SGD minimizes the loss, it also minimzes the sharpness implicitly, that is to say,
SGD has the following {\em implicit sharpness regularization}: $\min_{\btheta}\Tr(\nabla^2\cL(\btheta))\ {\rm s.t. \cL(\btheta)\approx0}$.

{\bf Generalization and Optimization benefits} of the sharpness regularization.
In terms of generalization, as discussed in related works, a common view is that {\em flat minima generalize well}, which has been proved in a large number of previous works. 
In addition, in terms of optimization, after SGD reaches the equilibrium $v_{\infty}\sim\cN(0,\eta\sigma^2)$, the loss near the flat minimum is smaller because $\bbE_{v_\infty}[\cL(\btheta)]=\bbE_{v_\infty}[h(\bu)v_{\infty}^2/2]=\eta\sigma^2 h(\bu)/2\propto\Tr(\nabla^2\cL(\bu))$.

\begin{center}
    {\em {\bf Q.} How can we enhance the implicit sharpness regularization of SGD?
    \\ {\bf A.} Accelerating SGD's slow ``effective dynamics'' $\bu_t$ along the minima manifold.}
\end{center}

{\bf Implicit Regularization Enhancement (IRE)} by accelerating the effective dynamics along minima manifold. 
First, it is worth noting that naively increasing the learning rate $\eta$ cannot achieve our aim, because increasing $\eta$ will influence the dynamic stability of $v_t$ and the equilibrium $v_{\infty}$. Therefore, we need to accelerate the effective dynamics $\bu_t$ without affecting the dynamics of $v_t$. 
Another main point is that SGD's effective dynamics $\bu_t$ can naturally minimize the sharpness implicitly, so we only need to enhance this property. To achieve this, we only need to correct the update direction in~\eqref{equ: sgd SDE} from $-\nabla\cL(\btheta_t)$ to $-(\nabla\cL(\btheta_t)+{\color{violet}{\kappa P_{\cM}\nabla\cL(\btheta_t))}}$, where $P_\cM$ is the projection matrix to the manifold $\cM$ and $\kappa$ is a scalar. Using this new algorithm, the SDE dynamics corresponds to
\begin{equation}\label{equ: sgd ire SDE}
\begin{aligned}
    \rd\begin{pmatrix}
    \bu_t \\ v_t
\end{pmatrix} = - \begin{pmatrix}
    {\color{violet}{(1+\kappa)}}v_t^2\nabla h(\bu_t)/2 \\ v_t h(\bu_t)
\end{pmatrix} \rd t +  \begin{pmatrix}
    0 \\ \sqrt{\eta\sigma^2 h(\bu_t)} \rd W_t
\end{pmatrix}.
\end{aligned}
\end{equation}
Comparing~\eqref{equ: sgd SDE} and~\eqref{equ: sgd ire SDE}, the dynamics of $v_t$ are the same, so they attain the same equilibrium distribution $v_{\infty}\sim\cN(0,\eta\sigma^2)$. As for the effective dynamics along manifold,~\eqref{equ: sgd SDE} corresponds to the form:
\begin{equation}\label{equ: sgd ire effective}
\begin{aligned}
  \rd\bu_t/\rd t=-\bbE_{v_\infty}\left[(1+\kappa){v_{\infty}^2\nabla h(\bu_t)/2}\right]=-(1+\kappa)\nabla h(\bu_t)/2.
\end{aligned}
\end{equation}

{\bf $(1+\kappa)$ times Enhancement.} Comparing~\eqref{equ: sgd ire effective} and~\eqref{equ: sgd effective}, it is clear that our new  algorithm can enhance implicit sharpness regularization $(1+\kappa)$ times faster than the original SGD.

%% file: appendix/experiment.tex
\vspace{1.cm}
\section{Experimental Details}
\label{appendix: exp}

This section describes the experimental details in Section~\ref{section: experiments}.

\subsection{Experimental details in Section~\ref{section: exp: CV}}

\subsubsection{Experimental details in Section~\ref{section: exp: CV: validate theory}}

We train WideResNet-16-8 on CIFAR-10 dataset by SAM-IRE (with $K=10$, varying $\kappa$ and $\gamma$). 
The experiments employ basic data augmentations and 0.1 label smoothing, as outlined by~\cite{foret2020sharpness}. 
The mini-batch size is set to 128, the weight decay is set to 5e-4, and the $\rho$ in SAM is to 0.05, as in~\cite{foret2020sharpness}.
To evaluate the implicit flatness regularization of SAM itself, the momentum is set to 0.0.
Regarding the learning rate (lr), we evaluate for both constant lr (within our theoretical framework) and decayed lr (common in practice though not covered by our theory).
In the experiment in Fig~\ref{fig: validation, cifar} (a), a fixed lr $0.1$ is used.
In the experiment in Fig~\ref{fig: validation, cifar} (b)(c)(d), a step-decayed lr schedule is employed, starting at $0.1$ and reducing lr by a factor of $5$ at epoch 20, 50, 80.
We transit from SAM to SAM-IRE at the 30th epoch out of 100 total epochs.
We test $\gamma$ in $0.8,0.9,0.95$, and $\kappa$ in $0$ (original SAM), $2,5,10$.

The flatness measure, $\Tr(\nabla^2\cL(\btheta))$, is approximated by the trace of Fisher $\Tr(\bF(\btheta))$.
Specifically, we use ${\rm diag}(\hat{F}_{\rm eff}(\btheta))$ in~\eqref{equ: fisher estimate} for the estimate 
because $\bbE_{\hat{\by}}[{\rm diag}(\hat{F}_{\rm eff}(\btheta))]=\bbE_{\hat{\by}}[{\rm diag}(\hat{F}(\btheta))]$ and thus,
\begin{align*}
\Tr(\nabla^2\cL(\btheta))\approx\bbE_{\hat{\by}}[\Tr(\hat{F}(\btheta))]=\bbE_{\hat{\by}}[\Tr(\diag(\hat{F}(\btheta)))]=\bbE_{\hat{\by}}[{\rm diag}(\hat{F}_{\rm eff}(\btheta))].
\end{align*}
Moreover, the first ``$\approx$'' above takes ``$=$'' when $\cL(\btheta)=0$.

In this section, all experiments were conducted using a single A800 GPU.

\subsubsection{Experimental details in Section~\ref{section: exp: CV: default results}}

{\bf Experiments for CNNs on CIFAR-10/CIFAR-100.} 
We first evaluate the impact of IRE on generalization of baseline models (WideResNet-28-10 and ResNet-56) and default optimizers (SGD and SAM) on CIFAR-\{10,100\}.
For the base optimizers, SGD and SAM, cosine learning rate decay is adopted with an initial lr $0.1$. 
For other training components, we follow the settings in~\cite{foret2020sharpness}: basic data augmentations and $0.1$ label smoothing;
for both SGD and SAM, the momentum is set to $0.9$, the batch size is set to $128$, and the weight decay is set to 5e-4;
for SAM, $\rho$ is set to $0.05$ for CIFAR-10 and $0.1$ for CIFAR-100.
The total epochs is set to 100 for CIFAR-10 and 200 for CIFAR-100, and we switch from SGD/SAM to SGD-IRE/SAM-IRE when the training loss approaches 0.1.
For SGD-IRE/SAM-IRE, we fix $K=10$, and tune hyperparameters $\gamma$ and $\kappa$ via a grid search over
$\gamma\in\{0.99, 0.9, 0.8\}$ and $\kappa\in\{1,2\}$. 
The results are reported in Table~\ref{table: resnet cifar}.

{\bf Experiments without finely tuned hyperparameters.} 
A high sensitivity to the choice of hyperparameters would make a method less practical. 
To demonstrate that our IRE performs {\em even when $\kappa,\gamma$ are not finely tuned}, we conduct experiments using fixed $\gamma=0.99,\kappa=1$, under the same settings as described above..
The results (averaged over 3 random seeds) are reported in Table~\ref{table: resnet cifar, fix}. 

\begin{table}[!ht]
    \caption{\small Results for SGD-IRE and SAM-IRE on \{WideResNet-28-10, ResNet-56\} on CIFAR-\{10, 100\}, using fixed $\gamma=0.99,\kappa=1$ in IRE.}
    \centering
    \small
    \begin{tabular}{c||c|c||c|c}
    \hline\hline
     \multirow{2}*{~} & \multicolumn{2}{c||}{WideResNet-28-10} & \multicolumn{2}{c}{ResNet-56} \\ \cline{2-5}
     & CIFAR-10  &  CIFAR-100  &  CIFAR-10  & CIFAR-100\\\hline
    SGD & 95.93 &  80.77 &  93.80 & 72.72
    \\ 
    SGD-IRE & {\bf 96.13} ({\bf +0.20}) &  {\bf 81.12} ({\bf +0.35}) &  {\bf 93.94}({\bf +0.14}) & {\bf 72.93}({\bf +0.21})
    \\
    \hline
    SAM & 96.73 & 83.22 &  94.58 & 75.25
    \\ 
    SAM-IRE &  {\bf 96.75} ({\bf +0.02}) & {\bf 83.40} ({\bf +0.19})  &{\bf 94.65} ({\bf +0.07})   &{\bf 75.49} ({\bf +0.24})
    \\ \hline\hline
    \end{tabular}
    \label{table: resnet cifar, fix}
\end{table}

{\bf Experiments for CNNs on ImageNet.} 
We also examine the impact of IRE on generalization of ResNet-50 and default optimizers (SGD and SAM) on on ImageNet. 
Following~\cite{foret2020sharpness} and \cite{kwon2021asam}, we use basic data augmentations and $0.1$ label smoothing.
For the base optimizers, SGD and SAM, we also follow the settings in~\cite{kwon2021asam}: the momentum is set to $0.9$; cosine learning rate decay is adopted with an initial lr $0.2$; the batch size is set to 1024; the weight decay is set to 1e-4; for SAM, $\rho$ is set to $0.05$.
The total epochs is set to 200, and we switch from SGD/SAM to SGD-IRE/SAM-IRE when the training loss approaches 1.5.
For SGD-IRE/SAM-IRE, we fix $K=10$, and tune hyperparameters $\gamma$ and $\kappa$ via a grid search over
$\gamma\in\{0.8,0.6\}$ and $\kappa\in\{2,4\}$. 
The results are reported in Table~\ref{table: resnet imagenet}.

{\bf Experiments for ViTs on CIFAR-100.} 
We examine the impact of IRE on generalization of ViT-T and ViT-S on CIFAR-100. 
We follow the settings in~\cite{mueller2024normalization}: the default optimizers used are AdamW and SAM, with cosine lr decay to $0$ starting from an initial lr 1e-4; the weight decay is 5e-4; batch size is $64$; strong data augmentations (basic + AutoAugment) are utilized; $\rho=0.1$ for SAM.
The total epochs are set to $200$, and we switch from AdamW/SAM to AdmIRE/SAM-IRE when the training loss approaches $0.5$. 
For AdmIRE/SAM-IRE, we fix $K=10$, and tune hyperparameters $\gamma$ and $\kappa$ via a grid search over $\gamma\in\{0.99, 0.9, 0.8\}$ and $\kappa\in\{20,50\}$. The results are reported in Table~\ref{table: vit cifar}.

{\bf Experiments for ViTs on ImageNet.} 
We also examine the impact of IRE on generalization of ViT-S/16 on ImageNet. 
We follow the settings in~\cite{chen2024symbolic}: RandAugment and Mixup with $\alpha=0.5$ are utilized;  the default optimizer used is AdamW with hyperparameters $\beta_1=0.9,\beta_2=0.999$ and weight decay $\lambda=1.0$; the learning rate strategy integrates a warm-up phase followed by a cosine decay scheduler with \texttt{lr\_max}=3e-3; batch size is $4096$; the total training duration is 300 epochs, including 30 warm-up epochs; 
For AdmIRE, we switch from AdamW to AdmIRE at epoch 100 and examine different $\gamma,\kappa$. The results are reported in Table~\ref{table: vit imagenet}.

In this section, the experiments on CIFAR-10/CIFAR-100 were conducted using a single A800 GPU, and the experiments on ImageNet were conducted using 4 A800 GPUs.

\subsection{Experimental details in Section~\ref{section: exp: NLP}}

\subsubsection{Experimental details in Section~\ref{section: exp: NLP: robustness}}

We train a 2-layer decoder-only Transformer (8M parameters) using absolute positional encodings~\citep{vaswani2017attention}, with 8 heads in each layer and a hidden size of 128, on the \texttt{wikitext-2} dataset (4.3M) by AdamW and AdmIRE (with $K=10$ and varying $\gamma,\kappa$). 
The (max) sequence length is set to 512, and the batch size is set to 32. 
The experiments in this section are conducted on 1 A800.

For the optimizer AdamW, we use the hyperparameters $\beta_1=0.9,\beta_2=0.95$ and weight decay $\lambda=0.1$, as suggested in~\cite{touvron2023llama}.
The total training duration is 100,000 steps, including 3,000 warm-up steps followed by a cosine decay scheduler with \texttt{lr\_max} and \texttt{lr\_min}=\texttt{lr\_max}$/20$.

First, we tune \texttt{lr\_max} in AdamW from $\{$\texttt{1.5e-4}, \texttt{3e-4}, \texttt{6e-4}, \texttt{1.2e-3}, \texttt{1.8e-3}, \texttt{3e-3}$\}$.
The results, shown in Figure~\ref{fig: tune adamw} (left), identify the optimal \texttt{lr\_max}=\texttt{6e-4}.
We also use the optimal \texttt{lr\_max} of \texttt{6e-4} for AdmIRE, for which the IRE is enable at the end of warm-up phase.

The results are reported in Figure~\ref{fig: wiki2}.

\begin{figure}[!ht]
    \centering
    \begin{minipage}[t]{0.49\textwidth}
    \centering
    \includegraphics[width=5.0cm]{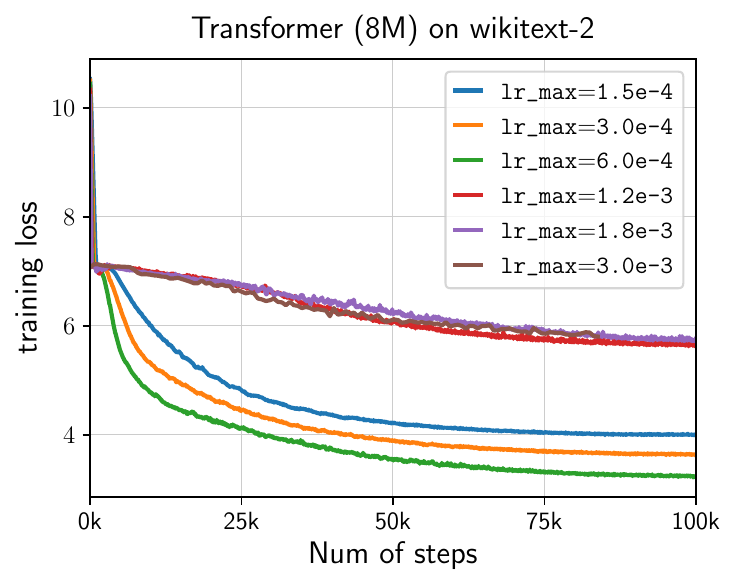}
    \end{minipage}
    \hspace{-1.cm}
    \begin{minipage}[t]{0.49\textwidth}
    \centering
    \includegraphics[width=5.0cm]{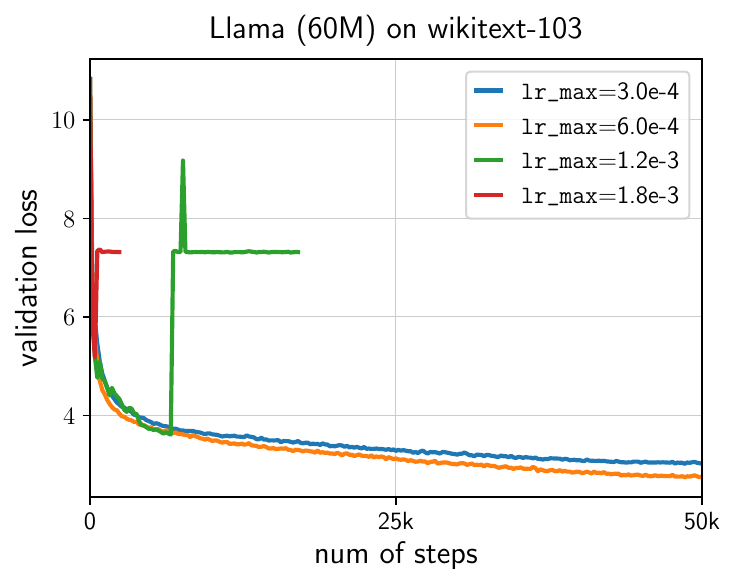}
    \end{minipage}
    \caption{ The results for tuning \texttt{lr\_max} in AdamW.}
    \label{fig: tune adamw}
\end{figure}

\subsubsection{Experimental details in Section~\ref{section: exp: NLP: Llama}}
Llama~\citep{touvron2023llama} is a decode-only Transformer architecture using Rotary Positional Encoding (RoPE)~\citep{su2024roformer}, Swish-Gated Linear Unit (SwiGLU)~\citep{shazeer2020glu}, and Root mean square layer normalization (RMSNorm)~\citep{zhang2019root}. 
The experiments in this section examine the performance of AdmIRE in training Llama models with different sizes.
For implementation, we utilize the Llama code available on~\href{https://huggingface.co/docs/transformers/model_doc/llama}{huggingface}.
The experiments are conducted on 4 H800.

For the optimizer AdamW, we use the well-tuned hyperparameters $\beta_1=0.9,\beta_2=0.95$ and weight decay $\lambda=0.1$ for LLama~\citep{touvron2023llama}. The learning rate strategy integrates a warm-up phase followed by a cosine decay scheduler with \texttt{lr\_max} and \texttt{lr\_min}=\texttt{lr\_max}$/20$.
Additionally, it is used with gradient clipping 1.0 to maintain the training stability. 

In each experiment, we tune the optimal \texttt{lr\_max} for AdamW and then use it also for AdmIRE, for which the IRE is enable at the end of warm-up phase.

{\bf Llama (60M) on wikitext-103 (0.5G).}
We train a 16-layer Llama model, with 10 heads in each layer and a hidden size of 410, on the wikitext-103 dataset.
The (max) sequence length is set to 150, and the batch size is set to 240, following~\cite{dai2019transformer}. 
The total training duration is 50,000 or 100,000 steps, including 500 warm-up steps.
First, we tune \texttt{lr\_max} in AdamW from $\{$\texttt{3e-4}, \texttt{6e-4}, \texttt{1.2e-3}, \texttt{1.8e-3}$\}$, identifying the optimal \texttt{lr\_max} \texttt{6e-4}.
The experimental results are very similar to Figure~\ref{fig: tune adamw} (right), so we will not show them again.
Then, both AdamW and AdmIRE are trained using the optimal \texttt{lr\_max}.

{\bf Llama (119M) on minipile (6G).}
We train a 6-layer Llama model, with 12 heads in each layer and a hidden size of 768, on the minipile dataset.
The (max) sequence length is set to 512, and the batch size is set to 300.
The total training duration is 30,000 or 60,000 steps, including 300 warm-up steps.
First, we tune \texttt{lr\_max} in AdamW from $\{$\texttt{3e-4}, \texttt{6e-4}, \texttt{1.2e-3}, \texttt{1.8e-3}$\}$, identifying the optimal \texttt{lr\_max} \texttt{6e-4}.
(The results are very similar to Figure~\ref{fig: tune adamw} (right), so we do not show them repeatly.)
Then, both AdamW and AdmIRE are trained using the optimal \texttt{lr\_max}.

{\bf Llama (229M) on openwebtext (38G).}
We train a 16-layer Llama model, with 12 heads in each layer and a hidden size of 768, on the openwebtext dataset.
The (max) sequence length is set to 1024, and the batch size is set to 480, following~\href{https://github.com/karpathy/nanoGPT/}{nanoGPT} and~\cite{liu2023sophia}. 
The total training duration is 50,000 or 100,000 steps, including 1,000 warm-up steps.
First, we tune \texttt{lr\_max} in AdamW from $\{$\texttt{3e-4}, \texttt{6e-4}, \texttt{1.2e-3}, \texttt{1.8e-3}$\}$, identifying the optimal \texttt{lr\_max} \texttt{6e-4}.
(The results are very similar to Figure~\ref{fig: tune adamw} (right), so we do not show them repearly.)
Then, both AdamW and AdmIRE are trained using the optimal \texttt{lr\_max}.

%% file: appendix/proof_SAM.tex

\vspace{1.cm}
\section{Proofs in Section~\ref{subsec: theory: average SAM}}
\label{appendix: average SAM}

{\bf Additional Notations.} For the proofs in Section~\ref{section: theory}, some additional notations are used.
For any set $K\subset\bbR^p$ and a constant $R>0$, we denote
$\bbB(K;R):=\{\btheta\in\bbR^p:{\rm dist}(\btheta;K)\leq R\}$.
$\<\cdot,\cdot\>$ represents the standard
Euclidean inner product between two vectors. $\norm{\cdot}$ denotes the $\ell_2$ norm of a vector or the spectral norm of a matrix, whereas $\|\cdot\|_{\rm F}$ denotes the Frobenius norm of a matrix.

\subsection{Preliminary Lemmas}

\begin{lemma}[\cite{arora2022understanding}, Lemma B.2]\label{lemma: local PL}
Under Assumption~\ref{ass: minima manifold}, for any compact set $K\subset\Gamma$, there exist absolute constants $R_1,\mu>0$ such that
\begin{itemize}[leftmargin=2em]
    \item (i) ${\bbB(K;R_1)}\subset U$;
    \item (ii) $\cL(\cdot)$ is $\mu$-PL (defined in Def~\ref{def: PL}) on ${\bbB(K;R_1)}$;
    \item (iii) $\inf\limits_{\btheta\in{\bbB(K;R_1)}}\lambda_m\left(\nabla^2\cL(\btheta)\right)\geq \mu$.
\end{itemize}
We further define the following absolute constants on $\bbB(K;R_1)$:
\begin{gather*}
    \beta_2:=\sup_{\btheta\in{\bbB(K;R_1)}}\norm{\nabla^2\cL(\btheta)};\quad
    \beta_3:=\sup_{\btheta\in{\bbB(K;R_1)}}\norm{\nabla^3\cL(\btheta)};\quad\beta_4:=\sup_{\btheta\in{\bbB(K;R_1)}}\norm{\nabla^4\cL(\btheta)};
    \\
    \nu:=\inf_{\btheta\in{\bbB(K;R_1)}}\lambda_m\left(\nabla^2\cL(\btheta)\right);\quad
    \zeta_{\Phi}:=\sup_{\btheta\in{\bbB(K;R_1)}}\norm{\nabla^2\Phi(\btheta)}
    .
\end{gather*}
\end{lemma}

\begin{lemma}[Key properties of $\Phi(\cdot)$~\citep{arora2022understanding}]\label{lemma: key property of phi}\ 
Under Assumption~\ref{ass: minima manifold}, 
\begin{itemize}[leftmargin=2em]
    \item For any $\btheta\in U$, $\partial\Phi(\btheta)\nabla\cL(\btheta)=\bzero$.
    \item For any $\btheta\in \Gamma$, $\partial\Phi(\btheta)=P_{m+1:p}(\nabla^2\cL(\btheta))$.
\end{itemize}
\end{lemma}

\begin{lemma}[Continuity of $P_{m+1:p}$]\label{lemma: projection gap estimate}
Under Assumption~\ref{ass: minima manifold}, there exists absolute constants $R_2,\zeta_{P}>0$ such that for any $\btheta\in{\bbB(K;R_2)}$,
    \begin{align*}
        \norm{P_{m+1:p}(\nabla^2\cL(\btheta))-P_{m+1:p}(\nabla^2\cL(\Phi(\btheta)))}\leq\zeta_{P}\norm{\btheta-\Phi(\btheta)}.
    \end{align*}
\end{lemma}

\begin{proof}[Proof of Lemma~\ref{lemma: projection gap estimate}]\ \\
Let the orthogonal decomposition of $\nabla^2\cL(\btheta)$ and $\nabla^2\cL(\Phi(\btheta))$ be $\nabla^2\cL(\btheta)=\sum_{k=1}^p\lambda_k\bu_k\bu_k^\top$ ($\lambda_1\geq\cdots\geq\lambda_p$) and $\nabla^2\cL(\Phi(\btheta))=\sum_{k=1}^p\mu_k\bv_k\bv_k^\top$ ($\mu_1\geq\cdots\geq\mu_p$), respectively.

By Lemma~\ref{lemma: local PL}, for any $\btheta\in{\bbB}(K;R_1)$, it holds that $\norm{\nabla^2\cL(\btheta)-\nabla^2\cL(\Phi(\btheta))}\leq\beta_3\norm{\btheta-\Phi(\btheta)}$. We choose 
$R_2:=R_1\wedge \frac{\mu}{4\beta_3}$. Then for any $\btheta\in\bbB(K;R_2)$,
\begin{align*}
     \norm{\nabla^2\cL(\btheta)-\nabla^2\cL(\Phi(\btheta))}\leq\beta_3\norm{\btheta-\Phi(\btheta)}\leq\frac{\mu}{4}.
\end{align*}

Consequently, by Lemma~\ref{lemma: Weyl thm}, we can bound the gap of eigenvalues: for any $k\in[p]$, \begin{align*}
    |\lambda_k-\mu_k|\leq\norm{\nabla^2\cL(\btheta)-\nabla^2\cL(\Phi(\btheta))}\leq\frac{\mu}{4}.
\end{align*}
Noticing $\Phi(\btheta)\in\Gamma$, by Lemma~\ref{lemma: local PL}, it holds that $\mu_1\geq\mu_m\geq\mu$ and $\mu_{m+1}=\cdots=\mu_p=0$. Thus, we can obtain the bounds of $\{\lambda_k\}_k$:
\begin{align*}
\text{for all $k\leq m$, }&\quad
    \lambda_k\geq\lambda_m\geq\mu_m-|\lambda_m-\mu_m|\geq\frac{3\mu}{4};
\\
\text{for all $k\geq m+1$, }&\quad
    \lambda_k\leq\lambda_{m+1}\leq\mu_{m+1}+|\lambda_{m+1}-\mu_{m+1}|\leq\frac{\mu}{4}.
\end{align*}

For simplicity, we denote $\bU_\text{top}:=(\bu_1,\cdots,\bu_m)$, $\bU_\text{bottom}:=(\bu_{m+1},\cdots,\bu_{p})$, $\bV_\text{top}:=(\bv_1,\cdots,\bv_m)$, $\bV_\text{bottom}:=(\bv_{m+1},\cdots,\bv_{p})$.

By Lemma~\ref{lemma: Davis-Kahan thm}, we can bound the gap between the subspaces:
\begin{align*}
    \norm{\bU_\text{bottom}^\top\bV_\text{top}}_{\rm F}\leq&\frac{\norm{\bU_\text{bottom}^\top(\nabla^2\cL(\btheta)-\nabla^2\cL(\Phi(\btheta)))\bV_\text{top}}_{\rm F}}{\frac{3\mu}{4}-\frac{\mu}{4}}
    \\\overset{\text{Lemma~\ref{lemma: F norm bounded by 2, F norm}}}{\leq}&\begin{cases}
        \frac{2}{\mu}\norm{\bU_\text{bottom}^\top}\norm{\nabla^2\cL(\btheta)-\nabla^2\cL(\Phi(\btheta))}\norm{\bV_\text{top}}_{\rm F}
        \\\frac{2}{\mu}\norm{\bU_\text{bottom}^\top}_{\rm F}\norm{\nabla^2\cL(\btheta)-\nabla^2\cL(\Phi(\btheta))}\norm{\bV_\text{top}}
    \end{cases}
    \\=&\frac{2\left(m\wedge(p-m)\right)}{\mu}\norm{\nabla^2\cL(\btheta)-\nabla^2\cL(\Phi(\btheta))}.
\end{align*}

According to the definition of $P_{m+1:p}(\cdot)$, it holds that
\begin{gather*}
    P_{m+1:p}(\nabla^2\cL(\btheta))=\sum_{k=m+1}^p\bu_k\bu_k^\top=\bU_\text{bottom}\bU_\text{bottom}^\top,
    \\
    P_{m+1:p}(\nabla^2\cL(\Phi(\btheta)))=\sum_{k=m+1}^p\bv_k\bv_k^\top=\bV_\text{bottom}\bV_\text{bottom}^\top.
\end{gather*}
Noticing the relationship 
\begin{align*}
    &\norm{P_{m+1:p}(\nabla^2\cL(\btheta))-P_{m+1:p}(\nabla^2\cL(\Phi(\btheta)))}^2\leq\norm{ P_{m+1:p}(\nabla^2\cL(\btheta))-P_{m+1:p}(\nabla^2\cL(\Phi(\btheta)))}_{\rm F}^2
    \\=&\norm{\bU_\text{bottom}\bU_\text{bottom}^\top-\bV_\text{bottom}\bV_\text{bottom}^\top}_{\rm F}^2
    =2(p-m)-2\Tr\left(\bU_\text{bottom}\bU_\text{bottom}^\top\bV_\text{bottom}\bV_\text{bottom}^\top\right)
    \\=&2\Tr\left(\bU_\text{bottom}\bU_\text{bottom}^\top\left(\bI-\bV_\text{bottom}\bV_\text{bottom}^\top\right)\right)=2\Tr\left(\bU_\text{bottom}\bU_\text{bottom}^\top\bV_\text{top}\bV_\text{top}^\top\right)
    \\=&2\norm{\bU_\text{bottom}^\top\bV_\text{top}}_{\rm F}^2,
\end{align*}
we obtain the bound:
\begin{align*}
    &\norm{P_{m+1:p}(\nabla^2\cL(\btheta))-P_{m+1:p}(\nabla^2\cL(\Phi(\btheta)))}\leq\sqrt{2}\norm{\bU_\text{bottom}^\top\bV_\text{top}}_{\rm F}
    \\\leq&\frac{2\sqrt{2}\left(m\wedge(p-m)\right)}{\mu}\norm{\nabla^2\cL(\btheta)-\nabla^2\cL(\Phi(\btheta))}\leq\frac{2\sqrt{2}\left(m\wedge(p-m)\right)\beta_3}{\mu}\norm{\btheta-\Phi(\btheta)}.
\end{align*}

To summarize, we only need to choose the constants $R_2=R_1\wedge \frac{\mu}{4\beta_3}$ and  $\zeta_P=\frac{2\sqrt{2}\left(m\wedge(p-m)\right)\beta_3}{\mu}$ to ensure this lemma holds.
\end{proof}

{\bf Proof Notations.} Now we introduce some additional useful notations in the proof in this section.

First, we choose $R:=(R_1\wedge R_2)/2$,
where $R_1$ is defined in Lemma~\ref{lemma: local PL} and $R_2$ is defined in Lemma~\ref{lemma: projection gap estimate}.
Let $\mu$ be the PL constant on $\bbB(K;R)$.
Moreover, we use the following notations:

\begin{equation}
\begin{gathered}
    \beta_2:=\sup_{\btheta\in{\bbB(K;R)}}\norm{\nabla^2\cL(\btheta)};\quad
    \beta_3:=\sup_{\btheta\in{\bbB(K;R)}}\norm{\nabla^3\cL(\btheta)};\quad\beta_4:=\sup_{\btheta\in{\bbB(K;R)}}\norm{\nabla^4\cL(\btheta)};
    \\
    \nu:=\inf_{\btheta\in{\bbB(K;R)}}\lambda_m\left(\nabla^2\cL(\btheta)\right);\quad
    \zeta_{\Phi}:=\sup_{\btheta\in{\bbB(K;R)}}\norm{\nabla^2\Phi(\btheta)};
    \\
    \zeta_P:=\sup_{\btheta\in{\bbB(K;R)}-\Gamma}\frac{\norm{P_{m+1:p}(\nabla^2\cL(\btheta))-P_{m+1:p}(\nabla^2\cL(\Phi(\btheta)))}}{\norm{\btheta-\Phi(\btheta)}}
    .
\end{gathered}
\end{equation}

Ensured by Lemma~\ref{lemma: local PL} and~\ref{lemma: projection gap estimate}, these quantities are all absolute constants in $(0,+\infty)$. Moreover, without loss of generality, we can assume that $\beta_1,\beta_2,\beta_3,\zeta_\Phi,\zeta_P>1$ and $\mu\leq\nu<1$.

\begin{lemma}[Connections between para norm, grad norm, and loss]\label{lemma: properties of PL}
For any $\btheta\in\bbB(K;R)>0$, it holds that:
\begin{itemize}[leftmargin=2em]
    \item (para norm v.s. grad norm) $\mu\norm{\nabla\cL(\btheta)}\leq\norm{\btheta-\Phi(\btheta)}\leq\beta_2\norm{\nabla\cL(\btheta)}$;
    \item (grad norm v.s. loss) $2\mu\cL(\btheta)\leq\norm{\nabla\cL(\btheta)}^2\leq\frac{2\beta_2^2}{\mu}\cL(\btheta)$;
    \item (loss v.s. para norm) $\frac{\mu}{2}\norm{\btheta-\Phi(\btheta)}^2\leq\cL(\btheta)\leq\frac{\beta_2^2}{2\mu}\norm{\btheta-\Phi(\btheta)}^2$.
\end{itemize}
    
\end{lemma}

\begin{proof}[Proof of Lemma~\ref{lemma: properties of PL}]
This lemma is a corollary of local PL and smoothness (Lemma~\ref{lemma: local PL}).
For the three lower bounds, please refer to the proof of Lemma B.6 in~\cite{arora2022understanding}. 
Then utilizing these lower bounds and the smoothness $\beta_2$, the upper bounds hold naturally.
\end{proof}

\begin{lemma}\label{lemma: near minima, projection estimate}
For all $\btheta\in\bbB(K;R)$, 
\begin{itemize}[leftmargin=2em]
    \item $\norm{P_{m+1:p}(\nabla^2\cL(\btheta))\nabla^2\cL(\btheta)}\leq\cO\left(\norm{\btheta-\Phi(\btheta)}\right)$;
    \item $\norm{P_{m+1:p}(\nabla^2\cL(\btheta))\nabla\cL(\btheta)}\leq\cO\left(\norm{\btheta-\Phi(\btheta)}^2\right)$;
    \item Let $\rho>0$ and $\bv\in\bbS^{p-1}$. If $\btheta+\rho\bv\in\bbB(K;R)$, then 
    \begin{gather*}
        \norm{\nabla\cL(\btheta+\rho\bv)}\leq\norm{\nabla\cL(\btheta)}+\rho\beta_2
        ;
        \\\norm{P_{m+1:p}(\nabla^2\cL(\btheta))\nabla\cL(\btheta+\rho\bv)}\leq\cO\left(\norm{\btheta-\Phi(\btheta)}^2\right)+\cO\left(\rho\norm{\btheta-\Phi(\btheta)}\right)+\frac{\rho^2\beta_3}{2}.
    \end{gather*}
\end{itemize}
    
\end{lemma}

\begin{proof}[Proof of Lemma~\ref{lemma: near minima, projection estimate}]
\begin{align*}
    \norm{P_{m+1:p}(\nabla^2\cL(\btheta))\nabla^2\cL(\btheta)}\leq&\norm{P_{m+1:p}(\nabla^2\cL(\Phi(\btheta)))\nabla^2\cL(\Phi(\btheta))}+\zeta_P\beta_2\norm{\btheta-\Phi(\btheta)}+\beta_3\norm{\btheta-\Phi(\btheta)}
    \\=&0+\cO\left(\norm{\btheta-\Phi(\btheta)}\right);
\end{align*}
\begin{align*}
    &\norm{P_{m+1:p}(\nabla^2\cL(\btheta))\nabla\cL(\btheta)}
    \leq\norm{P_{m+1:p}(\nabla^2\cL(\btheta))\nabla^2\cL(\Phi(\btheta))(\Phi(\btheta)-\btheta)}+\cO\bracket{\norm{\btheta-\Phi(\btheta)}^2}
    \\\leq&\norm{P_{m+1:p}(\nabla^2\cL(\Phi(\btheta)))\nabla^2\cL(\Phi(\btheta))(\Phi(\btheta)-\btheta)}+\cO\bracket{\norm{\btheta-\Phi(\btheta)}^2}=\cO\bracket{\norm{\btheta-\Phi(\btheta)}^2};
\end{align*}
\begin{align*}
     \norm{\nabla\cL(\btheta+\rho\bv)}\leq
     \norm{\nabla\cL(\btheta)}+\rho\beta_2;
\end{align*}
\begin{align*}
    &\norm{P_{m+1:p}(\nabla^2\cL(\btheta))\nabla\cL(\btheta+\rho\bv)}
    \\\leq&\norm{P_{m+1:p}(\nabla^2\cL(\btheta))\nabla\cL(\btheta)}+\rho\norm{P_{m+1:p}(\nabla^2\cL(\btheta))\nabla^2\cL(\btheta)\bv}+\frac{\rho^2\beta_3}{2}
    \\\leq&
    \cO\left(\norm{\btheta-\Phi(\btheta)}^2\right)+\cO\left(\rho\norm{\btheta-\Phi(\btheta)}\right)+\frac{\rho^2\beta_3}{2}.
\end{align*}
\end{proof}

\subsection{Proof of Theorem~\ref{thm: Phase II, ave SAM-IRE}}

\subsubsection{Proof of Keeping Moving Near Minimizers for SAM}

We first give the proof for SAM about ``keeping moving near minimizers'', which provides important insights into the proof for SAM-IRE.

Recalling~\eqref{equ: SAM, average}, the update rule of average SAM is:
\begin{align*}
    \btheta_{t+1}=\btheta_t-\eta\nabla\cL\bracket{\btheta_t+\rho\frac{\bxi_t}{\norm{\bxi_t}}},\text{ where }\bxi_t\sim\cN(\bzero,I).
\end{align*}

Let the $\rho$ in SAM satisfy:
$$\rho=\cO(\sqrt{\eta}).$$

For simplicity, we denote 
\begin{align*}
    \bv_t:=\frac{\bxi_t}{\norm{\bxi_t}},\quad
    C_1:=\frac{4\beta_2^3}{\mu},\quad
    C_2:=\beta_2^3.
\end{align*}

Notice $\cL(\btheta_{T_{\rm I}})< C_1\eta\rho^2$, we have the following upper bound for the probability
\begin{align*}
    &\bbP\left(\exists\ t\in[T_{\rm I},T_{\rm I}+T_{\rm II}],\cL(\btheta_t)\geq2C_1\eta\rho^2\right)
    \\\leq&
    \sum_{t=T_{\rm I}}^{T_{\rm I}+T_{\rm II}-1}\bbP\left(\cL(\btheta_{t+1})\geq2C_1\eta\rho^2;\ \forall\ s\in[T_{\rm I},{t}],\cL(\btheta_s)<2C_1\eta\rho^2\right).
\end{align*}

For each term $t\in[T_{\rm I},T_{\rm I}+T_{\rm II}-1]$, it can be bounded by:
\begin{align*}
    &\bbP\left(\cL(\btheta_{t+1})\geq2C_1\eta\rho^2;\ \forall\ s\in[T_{\rm I},{t}],\cL(\btheta_s)<2C_1\eta\rho^2\right)  
    \\\leq&\bbP\left(\cL(\btheta_{t+1})\geq2C_1\eta\rho^2;\cL(\btheta_{t})<C_1\eta\rho^2\right)
    \\&+\sum_{s=T_{\rm I}}^{t-1}
    \bbP\left(\cL(\btheta_{t+1})\geq2C_1\eta\rho^2;\cL(\btheta_{s})<C_1\eta\rho^2;\ \forall\  \tau\in[s+1,{t}],C_1\eta\rho^2\leq\cL(\btheta_\tau)<2C_1\eta\rho^2\right)
    \\\leq&\bbP\left(\cL(\btheta_{t+1})\geq2C_1\eta\rho^2\Big|\cL(\btheta_{t})<C_1\eta\rho^2\right)
    \\&+\sum_{s=T_{\rm I}}^{t-1}
    \bbP\left(\cL(\btheta_{t+1})\geq2C_1\eta\rho^2;\ \forall\  \tau\in[s+1,{t}],C_1\eta\rho^2\leq\cL(\btheta_\tau)<2C_1\eta\rho^2\Big|\cL(\btheta_{s})<C_1\eta\rho^2\right).
\end{align*}

For simplicity, we denote 
\begin{gather*}
    \bbP_{t+1,t}:=\bbP\left(\cL(\btheta_{t+1})\geq2C_1\eta\rho^2\Big|\cL(\btheta_{t})<C_1\eta\rho^2\right),
    \quad\\
    \bbP_{t+1,s}:=\bbP\left(\cL(\btheta_{t+1})\geq2C_1\eta\rho^2;\ \forall\  \tau\in[s+1,{t}],C_1\eta\rho^2\leq\cL(\btheta_\tau)<2C_1\eta\rho^2\Big|\cL(\btheta_{s})<C_1\eta\rho^2\right),\ s\in[T_{\rm I},t-1].
\end{gather*}

\begin{itemize}[leftmargin=2em]
    \item \underline{Step I. Bounding $\bbP_{t+1,t}$.}

    From $\cL(\btheta_{t})\leq C_1\eta\rho^2$, we have $\norm{\btheta_t-\Phi(\btheta_t)}\leq\sqrt{\frac{2}{\mu}\cL(\btheta_t)}=\cO(\frac{\sqrt{\eta}\rho}{\mu})$, thus 
\begin{align*}
    \norm{\btheta_t+\rho\bv_t-\Phi(\btheta_t)}\leq\norm{\btheta_t-\Phi(\btheta_t)}+\rho=\cO(\sqrt{\eta}\rho)+\cO(\rho)<R,
\end{align*}
which means $\btheta_t+\rho\bv_t\in\bbB(K;R)$.
Furthermore,
\begin{align*}
    &\norm{\btheta_{t+1}-\Phi(\btheta_t)}
    \leq\norm{\btheta_{t}-\Phi(\btheta_t)}+\eta\norm{\nabla\cL\left(\btheta_t+\rho\bv_t\right)}
    \\\leq&\cO(\sqrt{\eta}\rho)+\eta\norm{\nabla\cL\left(\btheta_t+\rho\bv_t\right)}
    \leq\cO(\sqrt{\eta}\rho)+\eta\norm{\nabla\cL\left(\btheta_t\right)}+\beta_2\eta\rho
    \\\leq&\cO(\sqrt{\eta}\rho)+\eta\sqrt{\frac{2\beta_2^2}{\mu}\cL(\btheta_t)}+\beta_2\eta\rho
    \leq\cO(\sqrt{\eta}\rho)+\cO(\eta^{3/2}\rho)+\cO(\eta\rho)
    \leq R,
\end{align*}
which implies $\btheta_{t+1}\in\bbB(K;R)$. 
Consequently, we have the following quadratic upper bound: 
\begin{align*}
    &\cL(\btheta_{t+1})=\cL\bracket{\btheta_t-\eta\nabla\cL\bracket{\btheta_t+\rho\bv_t}}
    \\\leq&
    \cL(\btheta_t)-\eta\<\nabla\cL(\btheta_t),\nabla\cL\bracket{\btheta_t+\rho\bv_t}\>+\frac{\eta^2\beta_2}{2}\norm{\nabla\cL\bracket{\btheta_t+\rho\bv_t}}^2
    \\\leq&
    \cL(\btheta_t)-\eta\norm{\nabla\cL(\btheta_t)}^2-\eta\rho\<\nabla\cL(\btheta_t),\nabla^2\cL(\btheta_t)\bv_t\>
    +\frac{\eta\rho^2\beta_3}{2}\norm{\nabla\cL(\btheta)}+\frac{\eta^2\beta_2}{2}\left(\norm{\nabla\cL(\btheta_t)}+\rho\beta_2\right)^2
    \\\leq&
    \cL(\btheta_t)-\eta\norm{\nabla\cL(\btheta_t)}^2-\eta\rho\<\nabla\cL(\btheta_t),\nabla^2\cL(\btheta_t)\bv_t\>
    +\frac{\eta\rho^2\beta_3}{2}\norm{\nabla\cL(\btheta)}+{\eta^2\beta_2}\left(\norm{\nabla\cL(\btheta_t)}^2+\rho^2\beta_2^2\right)
    \\\leq&
    \cL(\btheta_t)+\eta\rho\norm{\nabla^2\cL(\btheta_t)\nabla\cL(\btheta_t)}+\frac{\eta\rho^2\beta_3}{2}\norm{\nabla\cL(\btheta_t)}+\eta^2\beta_2\rho^2\beta_2^2
    \\\leq&
    \cL(\btheta_t)+\left(\eta\rho\beta_2+\frac{\eta\rho^2\beta_3}{2}\right)\norm{\nabla\cL(\btheta_t)}+\cO(\eta^2\rho^2)
    \\\leq&
    \cL(\btheta_t)+\left(\eta\rho\beta_2+\frac{\eta\rho^2\beta_3}{2}\right)\sqrt{\frac{2\beta_2^2}{\mu}}\sqrt{\cL(\btheta_t)}+\cO(\eta^2\rho^2)
    \\\leq&C_1\eta\rho^2+\cO(\eta^{3/2}\rho^2)+\cO(\eta^2\rho^2)<3C_1\eta\rho^2/2.
\end{align*}
Thus, we obtain
\begin{align*}
    \bbP_{t+1,t}=\bbP\left(\cL(\btheta_{t+1})\geq2C_1\eta\rho^2\Big|\cL(\btheta_{t})<C_1\eta\rho^2\right)=0.
\end{align*}

    \item \underline{Step II. Bounding $\bbP_{t+1,s}$ for $s\in[T_{\rm I},{t-1}]$.}

    We prove this step under the condition $\cL(\btheta_{s})<C_1\eta\rho^2$. Define a process $\{X_{\tau}\}_{\tau=s}^{t+1}$: $X_{s+1}=\cL(\btheta_{s+1})$,
    \begin{align*}
        X_{\tau+1}=\begin{cases}
            \cL(\btheta_{\tau+1}),\quad\text{ if } C_1\eta\rho^2\leq X_\tau=\cL(\btheta_\tau)\leq 2C_1\eta\rho^2
            \\
            X_{\tau}-C_2\eta^2\rho^2,\quad\text{ else}
        \end{cases}.
    \end{align*}

    It is clear that 
    \begin{align*}
        \bbP_{t+1,s}\leq\bbP\left(X_{t+1}\geq2C_1\eta\rho^2\right).
    \end{align*}

    Then our key step is to prove the following two claims about the process $\{X_{\tau}\}$.
    
    \begin{itemize}[leftmargin=2em]
        \item \underline{Claim I. $X_{\tau}-C_2\tau\eta\rho^2$ is a super-martingale.}
        From the definition of $X_{\tau}$, we only need to prove that if $C_1\eta\rho^2\leq\cL(\btheta_\tau)\leq 2C_1\eta\rho^2$, then $\bbE \left[\cL(\btheta_{\tau+1})\right]\leq \cL(\btheta_{\tau})-C_2\eta^2\rho^2$. 

        If $C_1\eta\rho^2\leq\cL(\btheta_\tau)\leq 2C_1\eta\rho^2$, similar to Step I, it is clear that $\btheta_{\tau+1}\in\bbB(K;R)$. Applying the quadratic upper bound, it holds that:
        \begin{align*}
        &\cL(\btheta_{\tau+1})=\cL\bracket{\btheta_\tau-\eta\nabla\cL\bracket{\btheta_\tau+\rho\bv_\tau}}
        \\\leq&
        \cL(\btheta_\tau)-\eta\<\nabla\cL(\btheta_\tau),\nabla\cL\bracket{\btheta_\tau+\rho\bv_\tau}\>+\frac{\eta^2\beta_2}{2}\norm{\nabla\cL\bracket{\btheta_\tau+\rho\bv_\tau}}^2
        \\\leq&
        \cL(\btheta_\tau)-\eta\norm{\nabla\cL(\btheta_\tau)}^2-\eta\rho\<\nabla\cL(\btheta_\tau),\nabla^2\cL(\btheta_\tau)\bv_t\>
        \\&\quad
        \quad\quad+\frac{\eta\rho^2\beta_3}{2}\norm{\nabla\cL(\btheta_\tau)}+\eta^2\beta_2\left(\norm{\nabla\cL(\btheta_\tau)}^2+\rho^2\beta_2^2\right).
        \end{align*}

        Taking the expectation, we have:
         \begin{align*}
            \bbE\left[\cL(\btheta_{\tau+1})\right]\leq
        &\cL(\btheta_\tau)-\eta\norm{\nabla\cL(\btheta_\tau)}^2+\frac{\eta\rho^2\beta_3}{2}\norm{\nabla\cL(\btheta_\tau)}+\eta^2\beta_2\left(\norm{\nabla\cL(\btheta_\tau)}^2+\rho^2\beta_2^2\right)\\\leq&
        \cL(\btheta_\tau)-\frac{3}{4}\eta\norm{\nabla\cL(\btheta_\tau)}^2+\frac{\eta\rho^2\beta_3}{2}\norm{\nabla\cL(\btheta_\tau)}+\eta^2\rho^2\beta_2^3
        \\\leq&
        \cL(\btheta_\tau)-\eta\norm{\nabla\cL(\btheta_\tau)}\left(\frac{3}{4}\norm{\nabla\cL(\btheta_\tau)}-\frac{\rho^2\beta_3}{2}\right)+\eta^2\rho^2\beta_2^3.
        \end{align*}

    From $\cL(\btheta_{\tau})\geq C_1\eta\rho^2$ and $\rho=\cO(\sqrt{\eta})$, it holds that 
    \begin{align*}
        \norm{\nabla\cL(\btheta_\tau)}\geq\sqrt{2\mu\cL(\btheta_t)}\geq\sqrt{2C_1\mu}\sqrt{\eta}\rho\geq4\beta_3\rho^2.
    \end{align*}

    Therefore, we have:
    \begin{align*}
        \bbE\left[\cL(\btheta_{\tau+1})\right]
        \leq&
        \cL(\btheta_\tau)-\frac{5}{8}\eta\norm{\nabla\cL(\btheta_\tau)}^2+\eta^2\rho^2\beta_2^3
        \\\leq&\cL(\btheta_{\tau})-\frac{10}{8}\eta\mu\cL(\btheta_\tau)+\eta^2\rho^2\beta_2^3
        \leq\cL(\btheta_{\tau})-\frac{10}{8}C_1\mu\eta^2\rho^2+\eta^2\rho^2\beta_2^3
        \\\leq&\cL(\btheta_{\tau})-\beta_2^3\eta^2\rho^2=\cL(\btheta_{\tau})-C_2\eta^2\rho^2.
    \end{align*}

    \item \underline{Claim II. $X_{\tau+1}-X_\tau+C_2\eta^2\rho^2$ is $\cO(\eta^2\rho^2+\eta^{3/2}\rho^2/p^{1/2})$-sub-Gaussian.} 
    From the definition of $X_{\tau}$, we only need to prove for the case $C_1\eta\rho^2\leq\cL(\btheta_\tau)\leq 2C_1\eta\rho^2$.

    If $C_1\eta\rho^2\leq\cL(\btheta_\tau)\leq 2C_1\eta\rho^2$, then $X_{\tau}=\cL(\btheta_\tau)$ and $X_{\tau+1}=\cL(\btheta_{\tau+1})$. Similar to Step I, it is clear that $\btheta_{\tau+1}\in\bbB(K;R)$. 

    Applying the smoothness, we have: 
        \begin{align*}
        &\cL(\btheta_{\tau+1})=\cL\bracket{\btheta_\tau-\eta\nabla\cL\bracket{\btheta_\tau+\rho\bv_\tau}}
        \\=&
        \cL(\btheta_\tau)-\eta\<\nabla\cL(\btheta_\tau),\nabla\cL\bracket{\btheta_\tau+\rho\bv_\tau}\>+\cO\left(\eta^2\norm{\nabla\cL\bracket{\btheta_\tau+\rho\bv_\tau}}^2\right)
        \\=&
        \cL(\btheta_\tau)-\eta\norm{\nabla\cL(\btheta_\tau)}^2-\eta\rho\<\nabla\cL(\btheta_\tau),\nabla^2\cL(\btheta_\tau)\bv_t\>
        \\&\quad
        \quad\quad+\cO\left(\eta\rho^2\norm{\nabla\cL(\btheta_\tau)}\right)+\cO\left(\eta^2\norm{\nabla\cL\bracket{\btheta_\tau}}^2+\eta^2\rho^2\right)
        \\=&
        \cL(\btheta_\tau)+\cO\left(\eta^2\rho^2\right)-\eta\rho\<\nabla\cL(\btheta_\tau),\nabla^2\cL(\btheta_\tau)\bv_t\>+\cO\left(\eta\rho^2\sqrt{\eta}\rho\right)+\cO\left(\eta^3\rho^2+\eta^2\rho^2\right)
        \\=&
        \cL(\btheta_\tau)-\eta\rho\<\nabla\cL(\btheta_\tau),\nabla^2\cL(\btheta_\tau)\bv_t\>+\cO(\eta^2\rho^2),
        \end{align*}
        which implies:
        \begin{align*}
            &\norm{X_{\tau+1}-X_{\tau}-C_2\eta^2\rho^2}_{\psi_2}\leq\norm{\cL(\btheta_{\tau+1})-\cL(\btheta_\tau)}_{\psi_2}+\norm{C_2\eta^2\rho^2}_{\psi_2}
            \\\leq&\norm{\eta\rho\<\nabla\cL(\btheta_\tau),\nabla^2\cL(\btheta_\tau)\bv_t\>}_{\psi_2}+\cO(\eta^2\rho^2)
            \\\leq&\eta\rho\norm{\nabla^2\cL(\btheta_\tau)\nabla\cL(\btheta_\tau)}\norm{\<\frac{\nabla^2\cL(\btheta_\tau)\nabla\cL(\btheta_\tau)}{\norm{\nabla^2\cL(\btheta_\tau)\nabla\cL(\btheta_\tau)}},\bv_t\>}_{\psi_2}+\cO(\eta^2\rho^2)
            \\\leq&\cO\left(\eta^{3/2}\rho^{2}\norm{\<\frac{\nabla^2\cL(\btheta_\tau)\nabla\cL(\btheta_\tau)}{\norm{\nabla^2\cL(\btheta_\tau)\nabla\cL(\btheta_\tau)}},\bv_t\>}_{\psi_2}\right)+\cO(\eta^2\rho^2)
            \\\overset{\text{Lemma~\ref{lemma: gaussian projection prob bound}}}{\leq}&\cO\left(\eta^{3/2}\rho^{2}/\sqrt{p}\right)+\cO(\eta^2\rho^2).
        \end{align*}
    
    \end{itemize}

    With the preparation of Claim I and Claim II, we can use the Azuma-Hoeffding inequality (Lemma~\ref{lemma: Azuma} (ii)): for any $Q>0$, it holds that
    \begin{align*}
        \bbP\left(X_{t+1}-X_{s+1}+(t-s)C_2\eta^2\rho^2>Q\right)\leq2\exp\bracket{-\frac{Q^2}{2(t-s)\left(\cO(\eta^{3/2}\rho^2/p^{1/2}+\eta^2\rho^2)\right)^2}}.
    \end{align*}
    As proved in Claim I, $X_{s+1}=\cL(\btheta_{s+1})\leq\frac{3}{2}C_1\eta\rho^2$ due to $\cL(\btheta_s)\leq C_1\eta\rho^2$. Therefore, by choosing $Q=(t-s)C_2\eta^2\rho^2-\frac{3}{2}C_1\eta\rho^2+2C_1\eta\rho^2=(t-s)C_2\eta^2\rho^2+\frac{1}{2}C_1\eta\rho^2$, we have
    \begin{align*}
    &\bbP_{t+1,s}\leq\bbP\left(X_{t+1}\geq2C_1\eta\rho^2\right)
    \\\leq&
    \bbP\left(X_{t+1}-X_{s+1}+(t-s)C_2\eta^2\rho^2>(t-s)C_2\eta^2\rho^2+\frac{1}{2}C_1\eta\rho^2\right)
    \\\leq&
    2\exp\bracket{-\frac{\bracket{(t-s)C_2\eta^2\rho^2+\frac{1}{2}C_1\eta\rho^2}^2}{2(t-s)\left(\cO(\eta^{3/2}\rho^2/p^{1/2}+\eta^2\rho^2)\right)^2}}
    \\\leq&
    2\exp\bracket{-\frac{4(t-s)C_2\eta^2\rho^2\cdot\frac{1}{2}C_1\eta\rho^2}{4(t-s)\left(\cO(\eta^{3}\rho^4/p+\eta^4\rho^4)\right)}}\leq2\exp\left(-\Omega\left(\frac{1}{\eta+p^{-1}}\right)\right).
    \end{align*}

\end{itemize}

Therefore, we obtain the union bound:
\begin{align*}
    &\bbP\left(\exists\ t\in[T_{\rm I},T_{\rm I}+T_{\rm II}],\cL(\btheta_t)\geq2C_1\eta\rho^2\right)
    \leq\sum_{t=T_{I}}^{T_{\rm I}+T_{\rm II}-1}\left(\bbP_{t+1,t}+\sum_{s=T_{\rm I}}^{t-1}\bbP_{t+1,s}\right)
    \\\leq&
    \sum_{t=T_{I}}^{T_{\rm I}+T_{\rm II}-1}\sum_{s=T_{\rm I}}^{t-1}\bbP_{t+1,s}\leq T_{\rm II}^2\exp\left(-\Omega\left(\frac{1}{\eta+p^{-1}}\right)\right).
\end{align*}

Hence, with probability at least $1-T_{\rm II}^2\exp\left(-\Omega\left(\frac{1}{\eta+p^{-1}}\right)\right)$, for any $t\in[T_{\rm I},T_{\rm I}+T_{\rm II}]$,
\begin{align*}
    \norm{\btheta_t-\Phi(\btheta_t)}\leq\sqrt{\frac{2}{\mu}\cL(\btheta_t)}\leq2\sqrt{\frac{C_1}{\mu}}\sqrt{\eta}\rho=\frac{4\beta_2^{3/2}}{\mu}\sqrt{\eta}\rho=\cO(\sqrt{\eta}\rho).
\end{align*}

\subsubsection{Proof of Moving Near Minimizers for SAM-IRE}

We prove ``keeping moving near minimizers'' for SAM-IRE. The proof outline for SAM-IRE is the same as SAM. However, the key non-trivial difference is that the IRE term will hardly cause loss instability since IRE only perturbs the parameters in the flat directions.

Under the conditions in Theorem~\ref{thm: Phase II, ave SAM-IRE}, the update rule of IRE on average SAM is:
\begin{align*}
    \btheta_{t+1}=\btheta_t-\eta\nabla\cL\bracket{\btheta_t+\rho\frac{\bxi_t}{\norm{\bxi_t}}}-\eta\kappa P_{m+1:p}(\nabla^2\cL(\btheta_t))\nabla\cL\bracket{\btheta_t+\rho\frac{\bxi_t}{\norm{\bxi_t}}},\text{ where }\bxi_t\sim\cN(\bzero,I).
\end{align*}

Let the $\rho$ in SAM and the $\kappa$ in IRE satisfy:
\begin{equation}
    \rho=\cO(\sqrt{\eta}),\quad\kappa\leq\frac{1}{\rho}.
\end{equation}

For simplicity, we denote 
\begin{align*}
    \bv_t:=\frac{\bxi_t}{\norm{\bxi_t}},\quad P(\btheta_t):=P_{m+1:p}(\nabla^2\cL(\btheta_t)),\quad C_1=\frac{4\beta_2^3}{\mu}\vee\frac{4\beta_2\beta_3^2}{\mu},\quad
    C_2=\beta_2^3
\end{align*}

Following the proof for SAM, we denote
\begin{gather*}
    \bbP_{t+1,t}:=\bbP\left(\cL(\btheta_{t+1})\geq2C_1\eta\rho^2\Big|\cL(\btheta_{t})<C_1\eta\rho^2\right),
\end{gather*}
\begin{align*}
    \bbP_{t+1,s}:=\bbP&\Big(\cL(\btheta_{t+1})\geq2C_1\eta\rho^2;
    \\&\ \forall\  \tau\in[s+1,{t}],C_1\eta\rho^2\leq\cL(\btheta_\tau)<2C_1\eta\rho^2\Big|\cL(\btheta_{s})<C_1\eta\rho^2\Big),\ s\in[T_{\rm I},t-1].
\end{align*}

and it holds that
\begin{align*}
    \bbP\left(\exists\ t\in[T_{\rm I},T_{\rm I}+T_{\rm II}],\cL(\btheta_t)\geq2C_1\eta\rho^2\right)\leq\sum_{t=T_{\rm I}}^{T_{\rm I}+T_{\rm II}-1}\left(\bbP_{t+1,t}+\sum_{s=T_{\rm I}}^{t-1}\bbP_{t+1,s}\right).
\end{align*}

\begin{itemize}[leftmargin=2em]
    \item \underline{Step I. Bounding $\bbP_{t+1,t}$.}

    From $\cL(\btheta_{t})\leq C_1\eta\rho^2$, we have $\norm{\btheta_t-\Phi(\btheta_t)}\leq\sqrt{\frac{2}{\mu}\cL(\btheta_t)}=\cO(\sqrt{\eta}\rho)$, thus 
\begin{align*}
    \norm{\btheta_t+\rho\bv_t-\Phi(\btheta_t)}\leq\norm{\btheta_t-\Phi(\btheta_t)}+\rho=\cO(\sqrt{\eta}\rho)+\cO(\rho)<R,
\end{align*}
which means $\btheta_t+\rho\bv_t\in\bbB(K;R)$.
Furthermore,
\begin{align*}
    &\norm{\btheta_{t+1}-\Phi(\btheta_t)}
    \\\leq&\norm{\btheta_{t}-\Phi(\btheta_t)}+\eta\norm{\nabla\cL(\btheta_t+\rho\bv_t)}+\eta\kappa\norm{P(\btheta_t)\nabla\cL(\btheta_t+\rho\bv_t)}
    \\\overset{\text{Lemma~\ref{lemma: near minima, projection estimate}}}{\leq}&\cO(\sqrt{\eta}\rho)+\eta\norm{\nabla\cL\left(\btheta_t\right)}+\beta_2\eta\rho+\cO\left(\eta\kappa\rho^2\right)
    \\\leq&
    \cO(\sqrt{\eta}\rho)+\cO(\eta^{3/2}\rho)+\cO\left(\eta\rho\right)+\cO\left(\eta\rho\right)
    \\\leq&\cO(\sqrt{\eta}\rho)+\cO(\eta^{3/2}\rho)+\cO\left(\eta\rho\right)
    \leq R,
\end{align*}
which implies $\btheta_{t+1}\in\bbB(K;R)$. 
Consequently, we have the following quadratic upper bound: 
\begin{align*}
    &\cL(\btheta_{t+1})=\cL\bracket{\btheta_t-\eta\nabla\cL\bracket{\btheta_t+\rho\bv_t}-\eta\kappa P(\btheta_t)\nabla\cL\bracket{\btheta_t+\rho\bv_t}}
    \\\leq&
    \cL(\btheta_t)-\eta\<\nabla\cL(\btheta_t),\nabla\cL\bracket{\btheta_t+\rho\bv_t}+\kappa P(\btheta_t)\nabla\cL\bracket{\btheta_t+\rho\bv_t}\>
    \\&\quad\quad+\frac{\eta^2\beta_2}{2}\norm{\nabla\cL\bracket{\btheta_t+\rho\bv_t}+\kappa P(\btheta_t)\nabla\cL\bracket{\btheta_t+\rho\bv_t}}^2
    \\\leq&
    \cL(\btheta_t)-\eta\<\nabla\cL(\btheta_t),\nabla\cL\bracket{\btheta_t+\rho\bv_t}\>-\kappa\eta\<\nabla\cL(\btheta_t), P(\btheta_t)\nabla\cL\bracket{\btheta_t+\rho\bv_t}\>
    \\&\quad\quad+\eta^2\beta_2\left(\norm{\nabla\cL\bracket{\btheta_t+\rho\bv_t}}^2+\kappa^2\norm{ P(\btheta_t)\nabla\cL\bracket{\btheta_t+\rho\bv_t}}^2\right)
    \\\leq&
    \cL(\btheta_t)-\eta\norm{\nabla\cL(\btheta_t)}^2+{\eta\rho\beta_2}\norm{\nabla\cL(\btheta_t)}-\kappa\eta\<\nabla\cL(\btheta_t), P(\btheta_t)\nabla\cL(\btheta_t)\>
    \\&\quad\quad-\kappa\eta\rho\<\nabla\cL(\btheta_t), P(\btheta_t)\nabla^2\cL(\btheta_t)\bv_t\>+\kappa\eta\frac{\beta_3\rho^2}{2}\norm{ P(\btheta_t)\nabla\cL(\btheta_t)}
    \\
    &\quad\quad+\eta^2\beta_2\left(\left(\norm{\nabla\cL(\btheta_t)}+\rho\beta_2\right)^2+\kappa^2\norm{ P(\btheta_t)\nabla\cL\bracket{\btheta_t+\rho\bv_t}}^2\right)
    \\\leq&\cL(\btheta_t)+\eta\rho\beta_2\norm{\nabla\cL(\btheta_t)}+\kappa\eta\rho\norm{ P(\btheta_t)\nabla^2\cL(\btheta_t)\nabla\cL(\btheta_t)}
    \\&\quad\quad+\frac{\kappa\eta\rho^2\beta_3}{2}\norm{ P(\btheta_t)\nabla\cL(\btheta_t)}+\eta^2\beta_2\left(\left(\norm{\nabla\cL(\btheta_t)}+\rho\beta_2\right)^2+\kappa^2\norm{ P(\btheta_t)\nabla\cL\bracket{\btheta_t+\rho\bv_t}}^2\right)
    \\\overset{\text{Lemma~\ref{lemma: near minima, projection estimate}}}{\leq}&
    \cL(\btheta_t)+\cO(\eta\rho\norm{\nabla\cL(\btheta_t)})+\cO(\kappa\eta\rho\norm{\btheta_t-\Phi(\btheta_t)}\norm{\nabla\cL(\btheta_t)})
    \\&\quad\quad+\cO\left(\kappa\eta\rho^2\norm{\btheta_t-\Phi(\btheta_t)}^2\right)+\cO\left(\eta^2\rho^2\right)+\cO\left(\eta^2\kappa^2\rho^4\right)
    \\\leq&
    C_1\eta\rho^2+o(\eta\rho^2)\leq\frac{3}{2}C_1\eta\rho^2.
\end{align*}
Thus, we obtain
\begin{align*}
    \bbP_{t+1,t}=\bbP\left(\cL(\btheta_{t+1})\geq2C_1\eta\rho^2\Big|\cL(\btheta_{t})<C_1\eta\rho^2\right)=0.
\end{align*}

\item \underline{Step II. Bounding $\bbP_{t+1,s}$ for $s\in[T_{\rm I},{t-1}]$.}

    We prove this step under the condition $\cL(\btheta_{s})<C_1\eta\rho^2$. Define a process $\{X_{\tau}\}_{\tau=s}^{t+1}$: $X_{s+1}=\cL(\btheta_{s+1})$,
    \begin{align*}
        X_{\tau+1}=\begin{cases}
            \cL(\btheta_{\tau+1}),\quad\text{ if } C_1\eta\rho^2\leq X_\tau=\cL(\btheta_\tau)\leq 2C_1\eta\rho^2
            \\
            X_{\tau}-C_2\eta^2\rho^2,\quad\text{ else}
        \end{cases}.
    \end{align*}

    It is clear that 
    \begin{align*}
        \bbP_{t+1,s}\leq\bbP\left(X_{t+1}\geq2C_1\eta\rho^2\right).
    \end{align*}

    Then our key step is to prove the following two claims about the process $\{X_{\tau}\}$.
    
    \begin{itemize}[leftmargin=2em]
        \item \underline{Claim I. $X_{\tau}-C_2\tau\eta\rho^2$ is a super-martingale.}
        From the definition of $X_{\tau}$, we only need to prove that if $C_1\eta\rho^2\leq\cL(\btheta_\tau)\leq 2C_1\eta\rho^2$, then $\bbE [\cL(\btheta_{\tau+1})]\leq \cL(\btheta_{\tau})-C_2\eta^2\rho^2$. 

        If $C_1\eta\rho^2\leq\cL(\btheta_\tau)\leq 2C_1\eta\rho^2$, similar to Step I, it is clear that $\btheta_{\tau+1}\in\bbB(K;R)$. Applying the quadratic upper bound, it holds that:
    \begin{align*}
    &\cL(\btheta_{\tau+1})=\cL\bracket{\btheta_\tau-\eta\nabla\cL\bracket{\btheta_\tau+\rho\bv_\tau}-\eta\kappa P(\btheta_\tau)\nabla\cL\bracket{\btheta_\tau+\rho\bv_\tau}}
    \\\leq&
    \cL(\btheta_\tau)-\eta\<\nabla\cL(\btheta_\tau),\nabla\cL\bracket{\btheta_\tau+\rho\bv_\tau}+\kappa P(\btheta_t)\nabla\cL\bracket{\btheta_\tau+\rho\bv_\tau}\>
    \\&\quad+\frac{\eta^2\beta_2}{2}\norm{\nabla\cL\bracket{\btheta_\tau+\rho\bv_\tau}+\kappa P(\btheta_\tau)\nabla\cL\bracket{\btheta_\tau+\rho\bv_\tau}}^2
    \\\leq&
    \cL(\btheta_\tau)-\eta\<\nabla\cL(\btheta_\tau),\nabla\cL\bracket{\btheta_\tau+\rho\bv_\tau}\>-\kappa\eta\<\nabla\cL(\btheta_\tau), P(\btheta_t)\nabla\cL\bracket{\btheta_\tau+\rho\bv_\tau}\>
    \\&\quad+\eta^2\beta_2\left(\norm{\nabla\cL\bracket{\btheta_\tau+\rho\bv_\tau}}^2+\kappa^2\norm{ P(\btheta_\tau)\nabla\cL\bracket{\btheta_\tau+\rho\bv_\tau}}^2\right)
    \\\leq&
    \cL(\btheta_\tau)-\eta\norm{\nabla\cL(\btheta_\tau)}^2-\eta\rho\<\nabla\cL(\btheta_\tau),\nabla^2\cL(\btheta_t)\bv_\tau\>+\frac{\eta\rho^2\beta_3}{2}\norm{\nabla\cL(\btheta_\tau)}
    \\&\quad-\kappa\eta\<\nabla\cL(\btheta_\tau), P(\btheta_\tau)\nabla\cL(\btheta_\tau)\>-\kappa\eta\rho\<\nabla\cL(\btheta_\tau), P(\btheta_\tau)\nabla^2\cL(\btheta_\tau)\bv_\tau\>+\kappa\eta\frac{\beta_3\rho^2}{2}\norm{ P(\btheta_\tau)\nabla\cL(\btheta_\tau)}
    \\
    &\quad+\eta^2\beta_2\left(\left(\norm{\nabla\cL(\btheta_\tau)}+\rho\beta_2\right)^2+\kappa^2\norm{ P(\btheta_\tau)\nabla\cL\bracket{\btheta_\tau+\rho\bv_\tau}}^2\right).
        \end{align*}

        Taking the expectation, we have:
        \begin{align*}
        &\bbE[\cL(\btheta_{\tau+1})]
        \\\leq&
        \cL(\btheta_\tau)-\eta\norm{\nabla\cL(\btheta_\tau)}^2+\frac{\eta\rho^2\beta_3}{2}\norm{\nabla\cL(\btheta_\tau)}-\kappa\eta\<\nabla\cL(\btheta_\tau), P(\btheta_\tau)\nabla\cL(\btheta_\tau)\>
        \\&+\kappa\eta\frac{\beta_3\rho^2}{2}\norm{ P(\btheta_\tau)\nabla\cL(\btheta_\tau)}+\eta^2\beta_2\left(\left(\norm{\nabla\cL(\btheta_\tau)}+\rho\beta_2\right)^2+\kappa^2\norm{ P(\btheta_\tau)\nabla\cL\bracket{\btheta_\tau+\rho\bv_\tau}}^2\right)
        \\\leq&
        \cL(\btheta_\tau)-\eta\norm{\nabla\cL(\btheta_\tau)}^2+\frac{\eta\rho^2\beta_3}{2}\norm{\nabla\cL(\btheta_\tau)}+\kappa\eta\frac{\beta_3\rho^2}{2}\norm{ P(\btheta_\tau)\nabla\cL(\btheta_\tau)}
        \\&+2\eta^2\beta_2\norm{\nabla\cL(\btheta_\tau)}^2+2\beta_2^3\eta^2\rho^2+\beta_2\eta^2\kappa^2\norm{ P(\btheta_\tau)\nabla\cL\bracket{\btheta_\tau+\rho\bv_\tau}}^2
        \\\overset{\text{Lemma~\ref{lemma: near minima, projection estimate}}}{\leq}& \cL(\btheta_\tau)-\frac{3\eta}{4}\norm{\nabla\cL(\btheta_\tau)}^2+\frac{\eta\rho^2\beta_3}{2}\norm{\nabla\cL(\btheta_\tau)}+\cO\bracket{\kappa\eta\rho^2\cdot\eta\rho^2}
        \\&+2\beta_2^3\eta^2\rho^2+\beta_2\eta^2\kappa^2
        \left(\cO(\sqrt{\eta}\rho^2)+\frac{\beta_3}{2}\rho
        ^2\right)^2
        \\\leq& \cL(\btheta_\tau)-\frac{3\eta}{4}\norm{\nabla\cL(\btheta_\tau)}^2+\frac{\eta\rho^2\beta_3}{2}\norm{\nabla\cL(\btheta_\tau)}+2\beta_2^3\eta^2\rho^2+\beta_2\beta_3^2\eta^2\kappa^2\rho^4+o(\eta^2\rho^2).
        \end{align*}

    From $C_1\eta\rho^2 \leq\cL(\btheta_{\tau})\leq 2C_1\eta\rho^2$ and $\rho=\cO(\sqrt{\eta})$, it holds that 
    \begin{gather*}
        \norm{\nabla\cL(\btheta_\tau)}\geq\sqrt{2\mu\cL(\btheta_\tau)}\geq\sqrt{2C_1\mu}\sqrt{\eta}\rho\geq4\beta_3\rho^2.
    \end{gather*}
    Moreover, recall $\kappa\leq 1/\rho$. Therefore, we have the upper bound:
    \begin{align*}
        &\bbE[\cL(\btheta_{\tau+1})]
        \\\leq&
         \cL(\btheta_\tau)-\frac{3\eta}{4}\norm{\nabla\cL(\btheta_\tau)}^2+\frac{\eta\rho^2\beta_3}{2}\norm{\nabla\cL(\btheta_\tau)}+2\beta_2^3\eta^2\rho^2+\beta_2\beta_3^2\eta^2\kappa^2\rho^4+o(\eta^2\rho^2)
        \\\leq&
        \cL(\btheta_\tau)-\frac{5\eta}{8}\norm{\nabla\cL(\btheta_\tau)}^2+2\beta_2^3\eta^2\rho^2+\beta_2\beta_3^2\eta^2\kappa^2\rho^4+o(\eta^2\rho^2)
        \\\leq&
        \cL(\btheta_\tau)-\frac{10}{8}\eta\mu\cL(\btheta_\tau)+2\beta_2^3\eta^2\rho^2+\beta_2\beta_3^2\eta^2\rho^2+o(\eta^2\rho^2).
        \\\leq&
        \cL(\btheta_\tau)-\frac{10}{2}\left(\frac{\beta_2^3}{\mu}\vee\frac{\beta_2\beta_3^2}{\mu}\right)\eta^2\rho^2+2\beta_2^3\eta^2\rho^2+\beta_2\beta_3^2\eta^2\rho^2+o(\eta^2\rho^2)
        \\\leq&
        \cL(\btheta_\tau)-\beta_2^3\eta^2\rho^2=\cL(\btheta_\tau)-C_2\eta^2\rho^2.
    \end{align*}

\item \underline{Claim II. $X_{\tau+1}-X_\tau+C_2\eta^2\rho^2$ is $\cO(\eta^2\rho^2+\eta^{3/2}\rho^2/p^{1/2})$-sub-Gaussian.} 
    From the definition of $X_{\tau}$, we only need to prove for the case $C_1\eta\rho^2\leq\cL(\btheta_\tau)\leq 2C_1\eta\rho^2$.

    If $C_1\eta\rho^2\leq\cL(\btheta_\tau)\leq 2C_1\eta\rho^2$, then $X_{\tau}=\cL(\btheta_\tau)$ and $X_{\tau+1}=\cL(\btheta_{\tau+1})$. Similar to Step I, it is clear that $\btheta_{\tau+1}\in\bbB(K;R)$. 

    Applying the smoothness, we have: 
        \begin{align*}
        &\cL(\btheta_{\tau+1})=\cL\bracket{\btheta_\tau-\eta\nabla\cL\bracket{\btheta_\tau+\rho\bv_\tau}-\eta\kappa P(\btheta_\tau)\nabla\cL\bracket{\btheta_\tau+\rho\bv_\tau}}
         \\=&
        \cL(\btheta_\tau)-\eta\<\nabla\cL(\btheta_\tau),\nabla\cL\bracket{\btheta_\tau+\rho\bv_\tau}+\kappa P(\btheta_t)\nabla\cL\bracket{\btheta_\tau+\rho\bv_\tau}\>
        \\&\quad+\left(\eta^2\norm{\nabla\cL\bracket{\btheta_\tau+\rho\bv_\tau}+\kappa P(\btheta_\tau)\nabla\cL\bracket{\btheta_\tau+\rho\bv_\tau}}^2\right)
        \\=&
        \cL(\btheta_\tau)-\eta\norm{\nabla\cL(\btheta_\tau)}^2-\eta\rho\<\nabla\cL(\btheta_\tau),\nabla^2\cL(\btheta_\tau)\bv_\tau\>+\cO\left(\eta\rho^2\norm{\nabla\cL(\btheta_{\tau})}\right)
        \\&\quad-\eta\kappa\<\nabla\cL(\btheta_\tau), P(\btheta_\tau)\nabla\cL(\btheta_\tau)\>-\eta\kappa\rho\<\nabla\cL(\btheta_\tau), P(\btheta_\tau)\nabla^2\cL(\btheta_\tau)\bv_t\>+\cO\left(\eta\kappa\rho^2\norm{ P(\btheta_\tau)\nabla\cL(\btheta_\tau)}\right)
        \\&\quad+\left(\eta^2\norm{\nabla\cL\bracket{\btheta_\tau+\rho\bv_\tau}+\kappa P(\btheta_\tau)\nabla\cL\bracket{\btheta_\tau+\rho\bv_\tau}}^2\right).
        \end{align*}
       In the same way as the proof of Claim I, it holds that
       \begin{gather*}
           \eta\norm{\nabla\cL(\btheta_\tau)}^2=\cO(\eta^2\rho^2),\quad\eta\rho^2\norm{\nabla\cL(\btheta_\tau)}=\cO(\eta^2\rho^2),\quad
           \eta\kappa\rho^2\norm{ P(\btheta_\tau)\nabla\cL(\btheta_\tau)}=\cO\left(\eta^2\rho^3\right),
           \\
           \eta^2\norm{\nabla\cL\bracket{\btheta_\tau+\rho\bv_\tau}+\kappa P(\btheta_\tau)\nabla\cL\bracket{\btheta_\tau+\rho\bv_\tau}}^2=\cO\left(\eta^2\rho^2\right)
       \end{gather*}
        Thus, 
        \begin{align*}
            &\cL(\btheta_{\tau+1})-\cL(\btheta_{\tau})
            \\=&-\eta\rho\<\nabla\cL(\btheta_\tau),\nabla^2\cL(\btheta_\tau)\bv_t\>-\eta\kappa\rho\<\nabla\cL(\btheta_\tau), P(\btheta_\tau)\nabla^2\cL(\btheta_\tau)\bv_t\>+\cO(\eta^2\rho^2),
        \end{align*}
        
        which implies:
        \begin{align*}
            &\norm{X_{\tau+1}-X_{\tau}-C_2\eta^2\rho^2}_{\psi_2}\leq\norm{\cL(\btheta_{\tau+1})-\cL(\btheta_\tau)}_{\psi_2}+\norm{C_2\eta^2\rho^2}_{\psi_2}
            \\\leq&\eta\rho\norm{\<\nabla\cL(\btheta_\tau),\nabla^2\cL(\btheta_\tau)\bv_t\>}_{\psi_2}+\eta\kappa\rho\norm{\<\nabla\cL(\btheta_\tau), P(\btheta_\tau)\nabla^2\cL(\btheta_\tau)\bv_t\>}_{\psi_2}+\cO(\eta^2\rho^2)
            \\\leq&\eta\rho\norm{\nabla^2\cL(\btheta_\tau)\nabla\cL(\btheta_\tau)}\norm{\<\frac{\nabla^2\cL(\btheta_\tau)\nabla\cL(\btheta_\tau)}{\norm{\nabla^2\cL(\btheta_\tau)\nabla\cL(\btheta_\tau)}},\bv_t\>}_{\psi_2}
            \\&\quad+\eta\kappa\rho \norm{\nabla^2\cL(\btheta_\tau) P(\btheta_\tau)\nabla\cL(\btheta_\tau)}\norm{\<\frac{\nabla^2\cL(\btheta_\tau) P(\btheta_\tau)\nabla\cL(\btheta_\tau)}{\norm{\nabla^2\cL(\btheta_\tau) P(\btheta_\tau)\nabla\cL(\btheta_\tau)}},\bv_t\>}_{\psi_2}+\cO(\eta^2\rho^2)
            \\
            \overset{\text{Lemma~\ref{lemma: near minima, projection estimate}}}{\leq}&\cO\left(\eta^{3/2}\rho^{2}\norm{\<\frac{\nabla^2\cL(\btheta_\tau)\nabla\cL(\btheta_\tau)}{\norm{\nabla^2\cL(\btheta_\tau)\nabla\cL(\btheta_\tau)}},\bv_t\>}_{\psi_2}\right)
            \\&+\cO\left(\eta^{3/2}\rho^{2}\norm{\<\frac{\nabla^2\cL(\btheta_\tau) P(\btheta_\tau)\nabla\cL(\btheta_\tau)}{\norm{\nabla^2\cL(\btheta_\tau) P(\btheta_\tau)\nabla\cL(\btheta_\tau)}},\bv_t\>}_{\psi_2}\right)+\cO(\eta^2\rho^2)
            \\\overset{\text{Lemma~\ref{lemma: gaussian projection prob bound}}}{\leq}&\cO\left(\eta^{3/2}\rho^{2}/\sqrt{p}\right)+\cO(\eta^2\rho^2).
        \end{align*}

    \end{itemize}

    With the preparation of Claim I and Claim II, we can use the Azuma-Hoeffding inequality (Lemma~\ref{lemma: Azuma} (ii)): for any $Q>0$, it holds that
    \begin{align*}
        \bbP\left(X_{t+1}-X_{s+1}+(t-s)C_2\eta^2\rho^2>Q\right)\leq2\exp\bracket{-\frac{Q^2}{2(t-s)\left(\cO(\eta^{3/2}\rho^2/p^{1/2}+\eta^2\rho^2)\right)^2}}.
    \end{align*}
    As proved in Claim I, $X_{s+1}=\cL(\btheta_{s+1})\leq\frac{3}{2}C_1\eta\rho^2$ due to $\cL(\btheta_s)\leq C_1\eta\rho^2$. Therefore, by choosing $Q=(t-s)C_2\eta^2\rho^2-\frac{3}{2}C_1\eta\rho^2+2C_1\eta\rho^2=(t-s)C_2\eta^2\rho^2+\frac{1}{2}C_1\eta\rho^2$, we have
    \begin{align*}
    &\bbP_{t+1,s}\leq\bbP\left(X_{t+1}\geq2C_1\eta\rho^2\right)
    \\\leq&
    \bbP\left(X_{t+1}-X_{s+1}+(t-s)C_2\eta^2\rho^2>(t-s)C_2\eta^2\rho^2+\frac{1}{2}C_1\eta\rho^2\right)
    \\\leq&
    2\exp\bracket{-\frac{\bracket{(t-s)C_2\eta^2\rho^2+\frac{1}{2}C_1\eta\rho^2}^2}{2(t-s)\left(\cO(\eta^{3/2}\rho^2/p^{1/2}+\eta^2\rho^2)\right)^2}}
    \\\leq&
    2\exp\bracket{-\frac{4(t-s)C_2\eta^2\rho^2\cdot\frac{1}{2}C_1\eta\rho^2}{4(t-s)\left(\cO(\eta^{3}\rho^4/p+\eta^4\rho^4)\right)}}\leq2\exp\left(-\Omega\left(\frac{1}{\eta+p^{-1}}\right)\right).
    \end{align*}

\end{itemize}

Therefore, we obtain the union bound:
\begin{align*}
    &\bbP\left(\exists\ t\in[T_{\rm I},T_{\rm I}+T_{\rm II}],\cL(\btheta_t)\geq2C_1\eta\rho^2\right)
    \leq\sum_{t=T_{I}}^{T_{\rm I}+T_{\rm II}-1}\left(\bbP_{t+1,t}+\sum_{s=T_{\rm I}}^{t-1}\bbP_{t+1,s}\right)
    \\\leq&
    \sum_{t=T_{I}}^{T_{\rm I}+T_{\rm II}-1}\sum_{s=T_{\rm I}}^{t-1}\bbP_{t+1,s}\leq T_{\rm II}^2\exp\left(-\Omega\left(\frac{1}{\eta+p^{-1}}\right)\right).
\end{align*}

Hence, with probability at least $1-T_{\rm II}^2\exp\left(-\Omega\left(\frac{1}{\eta+p^{-1}}\right)\right)$, for any $t\in[T_{\rm I},T_{\rm I}+T_{\rm II}]$,
\begin{align*}
    \norm{\btheta_t-\Phi(\btheta_t)}\leq\sqrt{\frac{2}{\mu}\cL(\btheta_t)}\leq2\sqrt{\frac{C_1}{\mu}}\sqrt{\eta}\rho=\frac{4\beta_2^{3/2}}{\mu}\sqrt{\eta}\rho=\cO(\sqrt{\eta}\rho).
\end{align*}

\subsubsection{Proof of the Effective Dynamics}

We have proved that with high probability at least $1-T_{\rm II}^2\exp\left(-\Omega\left(\frac{1}{\eta+p^{-1}}\right)\right)$, for any $t\in[T_{\rm I},T_{\rm I}+T_{\rm II}]$, $\norm{\btheta_t-\Phi(\btheta_t)}=\cO(\sqrt{\eta}\rho)$.
Then we prove this theorem when the above event occurs.

For any $t\in[T_{\rm I},T_{\rm I}+T_{\rm II}]$,
\begin{align*}
    &\norm{\btheta_{t+1}-\btheta_t}
    \\\leq&\eta\norm{\nabla\cL(\btheta_t+\rho\bv_t)}+\eta\kappa\norm{ P(\btheta_t)\nabla\cL(\btheta_t+\rho\bv_t)}
    \\\leq&
    \eta\norm{\nabla\cL(\btheta_t)}+\cO(\eta\rho)+\eta\kappa\norm{ P(\btheta_t)\nabla\cL(\btheta_t)}+\eta\kappa\rho\norm{ P(\btheta_t)\nabla^2\cL(\btheta_t)}+\cO(\eta\kappa\rho^2)
    \\\overset{\text{Lemma~\ref{lemma: near minima, projection estimate}}}{\leq}&
    \cO(\eta^{3/2}\rho)+\cO(\eta\rho)+\cO(\eta^{3/2}\rho)+\cO(\eta^{3/2}\rho^2)+\cO(\eta\rho^2)+\cO(\eta\rho)=\cO(\eta\rho).
\end{align*}

Then by Taylor's expansion, 
\begin{align*}
    &\Phi(\btheta_{t+1})-\Phi(\btheta_t)=\partial\Phi(\btheta_t)\bracket{\btheta_{t+1}-\btheta_{t}}+\cO\bracket{\norm{\btheta_{t+1}-\btheta_{t}}^2}
    \\=&
    -\eta\partial\Phi(\btheta_t)\nabla\cL(\btheta_{t}+\rho\bv_t)-\eta\kappa\partial\Phi(\btheta_t) P(\btheta_t)\nabla\cL(\btheta_{t}+\rho\bv_t)+\cO(\eta^2\rho^2).
\end{align*}

For the term $\partial\Phi(\btheta_t)\nabla\cL(\btheta_{t}+\rho\bv_t)$ and $\partial\Phi(\btheta_t) P(\btheta_t)\nabla\cL(\btheta_{t}+\rho\bv_t)$, using Taylor's expansion, we have
\begin{align*}
    &\partial\Phi(\btheta_t)\nabla\cL(\btheta_{t}+\rho\bv_t)
    \\=&\partial\Phi(\btheta_t)\nabla\cL(\btheta_{t})+\rho\partial\Phi(\btheta_t)\nabla^2\cL(\btheta_t)\bv_t+\frac{\rho^2}{2}\partial\Phi(\btheta_t)\nabla\Tr\left(\bv_t\nabla^2\cL(\btheta_t)\bv_t^\top\right)+\cO(\rho^3)
    \\=&\partial\Phi(\btheta_t)\nabla\cL(\btheta_{t})+\rho\partial\Phi(\btheta_t)\nabla^2\cL(\btheta_t)\bv_t+\frac{\rho^2}{2}\partial\Phi(\btheta_t)\nabla\left(\bv_t^\top\nabla^2\cL(\btheta_t)\bv_t\right)+\cO(\rho^3),
\end{align*}
\begin{align*}
    &\partial\Phi(\btheta_t) P(\btheta_t)\nabla\cL(\btheta_{t}+\rho\bv_t)
    \\=&\partial\Phi(\btheta_t) P(\btheta_t)\nabla\cL(\btheta_{t})+\rho\partial\Phi(\btheta_t) P(\btheta_t)\nabla^2\cL(\btheta_t)\bv_t+\frac{\rho^2}{2}\partial\Phi(\btheta_t) P(\btheta_t)\nabla\Tr\left(\bv_t\nabla^2\cL(\btheta_t)\bv_t^\top\right)+\cO(\rho^3)
    \\=&\partial\Phi(\btheta_t) P(\btheta_t)\nabla\cL(\btheta_{t})+\rho\partial\Phi(\btheta_t) P(\btheta_t)\nabla^2\cL(\btheta_t)\bv_t+\frac{\rho^2}{2}\partial\Phi(\btheta_t) P(\btheta_t)\nabla\left(\bv_t^\top\nabla^2\cL(\btheta_t)\bv_t\right)+\cO(\rho^3).
\end{align*}
Taking the expectation (about $\bv_t$), we have
\begin{gather*}
    \bbE[\bv_t]=0,\quad \bbE\left[\bv_t^\top\nabla^2\cL(\btheta_t)\bv_t\right]=\frac{\Tr\bracket{\nabla^2\cL(\btheta_t)}}{p}.
\end{gather*} 

Additionally, using Lemma~\ref{lemma: key property of phi} and Taylor's expansion, we have:
\begin{align*}
    \partial\Phi(\btheta_t)\nabla\cL(\btheta_{t})=\bzero;
\end{align*}
\begin{align*}
    \partial\Phi(\btheta_t)=\partial\Phi(\Phi(\btheta_t))+\norm{\btheta_t-\Phi(\btheta_t)}=\partial\Phi(\Phi(\btheta_t))+\cO(\sqrt{\eta}\rho);
\end{align*}
\begin{align*}
    \partial\Phi(\btheta_t) P(\btheta_t)=\partial\Phi(\Phi(\btheta_t)) P(\Phi(\btheta_t))+\norm{\btheta_t-\Phi(\btheta_t)}=\partial\Phi(\Phi(\btheta_t))+\cO(\sqrt{\eta}\rho);
\end{align*}
\begin{align*}
    \nabla\Tr\left(\nabla^2\cL(\btheta_t)\right)=\nabla\Tr\left(\nabla^2\cL(\Phi(\btheta_t))\right)+\norm{\btheta_t-\Phi(\btheta_t)}=\nabla\Tr\left(\nabla^2\cL(\Phi(\btheta_t))\right)+\cO(\sqrt{\eta}\rho);
\end{align*}
\begin{align*}
    \partial\Phi(\btheta_t) P(\btheta_t)\nabla\cL(\btheta_{t})
    =&\partial\Phi(\Phi(\btheta_t)) P(\Phi(\btheta_t))\nabla\cL(\btheta_{t})+\cO\left(\norm{\btheta_t-\Phi(\btheta_t)}\norm{\nabla\cL(\btheta_{t})}\right)
    \\=&\partial\Phi(\Phi(\btheta_t))\nabla\cL(\btheta_{t})+\cO\left(\eta\rho^2\right)=\bzero+\cO(\eta\rho^2).
\end{align*}

Combining the results above, we obtain:
\begin{align*}
    &\bbE[\Phi(\btheta_{t+1})]
    \\=&\Phi(\btheta_{t})-\eta\partial\Phi(\btheta_t)\nabla\cL(\btheta_{t})-\eta\kappa\partial\Phi(\btheta_t) P(\btheta_t)\nabla\cL(\btheta_{t})
    \\&-\frac{\eta\rho^2}{2p}\partial\Phi(\btheta_t)\nabla\Tr\bracket{\nabla^2\cL(\btheta_t)}-\frac{\kappa\eta\rho^2}{2p}\partial\Phi(\btheta_t) P(\btheta_t)\nabla\Tr\bracket{\nabla^2\cL(\btheta_t)}+\cO(\eta\rho^3)
    \\=&\Phi(\btheta_{t})+\cO\left(\eta^2\kappa\rho^2\right)-\frac{(\kappa+1)\eta\rho^2}{2p}\partial\Phi(\btheta_t)\nabla\Tr\bracket{\nabla^2\cL(\btheta_t)}+\cO(\eta^{3/2}\rho^3)+\cO(\kappa\eta^{3/2}\rho^3)+\cO(\eta\rho^3)
    \\=&\Phi(\btheta_{t})+\cO\left(\eta^2\kappa\rho^2\right)-\frac{(\kappa+1)\eta\rho^2}{2p}\partial\Phi(\btheta_t)\nabla\Tr\bracket{\nabla^2\cL(\Phi(\btheta_t))}+\cO(\eta^{3/2}\rho^3)+\cO(\kappa\eta^{3/2}\rho^3)+\cO(\eta\rho^3)
    \\=&
    \Phi(\btheta_{t})-(\kappa+1)\eta\rho^2\partial\Phi(\btheta_t)\nabla\Tr\bracket{\nabla^2\cL(\Phi(\btheta_t))/2p}+\cO(\eta^{3/2}\rho^2).
\end{align*}

\vspace{1.cm}
\section{Proofs in Section~\ref{subsec: theory: standard SAM}}
\label{appendix: standard SAM}

\begin{setting}\label{setting: erm problem}
    Consider the empirical risk minimization $\min:\cL(\btheta)=\frac{1}{n}\sum_{i=1}^n\cL_i(\btheta)$, where $\cL_i(\btheta)=\ell(f_i(\btheta),y_i)$ is the loss on the $i$-th data $(\bx_i,y_i)$. Let $f_i(\cdot)$ and $\ell(\cdot,\cdot)$ be $\cC^4$. 
    Suppose all global minimizers interpolate the training dataset, i.e., $\cL(\btheta^\star)=\min_{\btheta}\cL(\btheta)$ implies $f_i(\btheta^\star)=y_i$ for all $i\in[n]$. We denote the minima manifold by $\cM=\{\btheta:f_i(\btheta)=y_i,\forall i\in[n]\}$. Moreover, we assume that $\frac{\partial^2\ell(\hat{y},y)}{\partial\hat{y}^2}|_{\hat{y}=y}>0$ and the feature matrix $(\nabla f_i(\btheta),\cdots,\nabla f_n(\btheta))\in\bbR^{p\times n}$ is full-rank at $\btheta\in\cM$.
\end{setting}

\begin{lemma}[Theorem 5.2 in~\cite{wen2023how}]\label{lemma: erm setting hold Ass}
Under Setting~\ref{setting: erm problem}, Assumption~\ref{ass: minima manifold} holds with $m=n$.
\end{lemma}

\subsection{Preliminary Lemmas}

Similar to the proofs for Section~\ref{subsec: theory: average SAM}, we need the following similar preliminary lemmas.

\begin{lemma}\label{lemma: local PL, std SAM}
Under Setting~\ref{setting: erm problem}, for any compact set $K\in\Gamma$, there exist absolute constants $R_1,\mu>0$ such that
\begin{itemize}[leftmargin=2em]
    \item (i) $\overline{\bbB(K;R_1)}\subset U$;
    \item (ii) $\cL_i(\cdot)$ ($i\in[n]$) and  $\cL(\cdot)$ are $\mu$-PL on $\overline{\bbB(K;R_1)}$;
    \item (iii) $\inf\limits_{\btheta\in\bbB(K;R_1)}\lambda_n\left(\nabla^2\cL(\btheta)\right)\geq \mu$; $\inf\limits_{\btheta\in\bbB(K;R_1)}\lambda_1\left(\nabla^2\cL_i(\btheta)\right)\geq \mu$, $\forall i\in[n]$.
\end{itemize}
We further define the following absolute constants on $\overline{\bbB(K;R)}$:
\begin{gather*}
    \beta_2:=\bracket{\sup_{\btheta\in\bbB(K;R_1)}\norm{\nabla^2\cL(\btheta)}}\vee\bracket{\max_{i\in[n]}\sup_{\btheta\in\bbB(K;R_1)}\norm{\nabla^2\cL_i(\btheta)}};
    \\
    \beta_3:=\bracket{\sup_{\btheta\in\bbB(K;R_1)}\norm{\nabla^3\cL(\btheta)}}\vee\bracket{\max_{i\in[n]}\sup_{\btheta\in\bbB(K;R_1)}\norm{\nabla^3\cL_i(\btheta)}};
    \\
    \nu:=\bracket{\inf_{\btheta\in\bbB(K;R_1)}\lambda_m\left(\nabla^2\cL(\btheta)\right)}\wedge\bracket{\min_{i\in[n]}\inf_{\btheta\in\bbB(K;R_1)}\lambda_1\left(\nabla^2\cL_i(\btheta)\right)};
    \\ \zeta_1^{\Phi}:=\sup_{\btheta\in\bbB(K;R_1)}\norm{\nabla\Phi(\btheta)},\quad \zeta_2^{\Phi}:=\sup_{\btheta\in\bbB(K;R_1)}\norm{\nabla^2\Phi(\btheta)}
    .
\end{gather*}
\end{lemma}

\begin{lemma}[\cite{wen2023how}]\label{lemma: key property of phi, std SAM} Under Assumption~\ref{ass: minima manifold}, 
\begin{itemize}[leftmargin=2em]
    \item For any $\btheta\in U$, $\partial\Phi(\btheta)\nabla\cL(\btheta)=\bzero$.
    \item For any $\btheta\in \Gamma$, $\partial\Phi(\btheta)= P_{n+1:p}(\nabla^2\cL(\btheta))$ and $\partial\Phi(\btheta)\nabla^2\cL(\btheta)=0$.
    \item For any $\btheta\in \Gamma$, $\partial\Phi(\btheta)\nabla^2\cL_i(\btheta)=0$, $\forall i\in[n]$.
\end{itemize}
\end{lemma}

    

\begin{lemma}\label{lemma: projection gap estimate, std SAM}
Under Setting~\ref{setting: erm problem}, there exists absolute constants $R_2,\zeta_{P}>0$ such that for any $\btheta\in\bbB(K;R_2)$,
    \begin{align*}
        \norm{ P_{n+1:p}(\nabla^2\cL(\btheta))- P_{n+1:p}(\nabla^2\cL(\Phi(\btheta)))}&\leq\zeta_{P}\norm{\btheta-\Phi(\btheta)}.
    \end{align*}
\end{lemma}

{\bf Proof Notations.} Now we introduce some additional useful notations in the proof in this section.

First, we choose $R:=(R_1\wedge R_2)/2$,
where $R_1$ is defined in Lemma~\ref{lemma: local PL, std SAM} and $R_2$ is defined in Lemma~\ref{lemma: projection gap estimate, std SAM}.
Let $\mu$ be the PL constant on $\bbB(K;R)$.
Moreover, we use the following notations:

\begin{equation}
\begin{gathered}
    \beta_2:=\bracket{\sup_{\btheta\in\bbB(K;R)}\norm{\nabla^2\cL(\btheta)}}\vee\bracket{\max_{i\in[n]}\sup_{\btheta\in\bbB(K;R)}\norm{\nabla^2\cL_i(\btheta)}};
    \\
    \beta_3:=\bracket{\sup_{\btheta\in\bbB(K;R)}\norm{\nabla^3\cL(\btheta)}}\vee\bracket{\max_{i\in[n]}\sup_{\btheta\in\bbB(K;R)}\norm{\nabla^3\cL_i(\btheta)}};
    \\
    \beta_4:=\bracket{\sup_{\btheta\in\bbB(K;R)}\norm{\nabla^3\cL(\btheta)}}\vee\bracket{\max_{i\in[n]}\sup_{\btheta\in\bbB(K;R)}\norm{\nabla^4\cL_i(\btheta)}};
    \\
    \nu:=\bracket{\inf_{\btheta\in\bbB(K;R)}\lambda_m\left(\nabla^2\cL(\btheta)\right)}\wedge\bracket{\min_{i\in[n]}\inf_{\btheta\in\bbB(K;R)}\lambda_1\left(\nabla^2\cL_i(\btheta)\right)};
    \\
    \zeta_{\Phi}:=\sup_{\btheta\in{\bbB(K;R)}}\norm{\nabla^2\Phi(\btheta)};\quad
    \zeta_P:=\sup_{\btheta\in{\bbB(K;R)}-\Gamma}\frac{\norm{ P_{n+1:p}(\nabla^2\cL(\btheta))- P_{n+1:p}(\nabla^2\cL(\Phi(\btheta)))}}{\norm{\btheta-\Phi(\btheta)}}
    .
\end{gathered}
\end{equation}

Ensured by Lemma~\ref{lemma: local PL} and~\ref{lemma: projection gap estimate}, these quantities are all absolute constants in $(0,+\infty)$. Moreover, without loss of generality, we can assume that $\beta_1,\beta_2,\beta_3,\zeta_\Phi,\zeta_P>1$ and $\mu\leq\nu<1$.

Then we have the following two lemmas, similar to Lemma~\ref{lemma: properties of PL} and~\ref{lemma: near minima, projection estimate}.

\begin{lemma}\label{lemma: properties of PL, std SAM}
For any $\btheta\in\bbB(K;R)$, it holds that:
\begin{itemize}[leftmargin=2em]
    \item (para norm v.s. grad norm) $\mu\norm{\nabla\cL(\btheta)}\leq\norm{\btheta-\Phi(\btheta)}\leq\beta_2\norm{\nabla\cL(\btheta)}$; $\mu\norm{\nabla\cL_i(\btheta)}\leq\norm{\btheta-\Phi(\btheta)}\leq\beta_2\norm{\nabla\cL_i(\btheta)}$, $\forall i\in[n]$.
    \item (grad norm v.s. loss) $2\mu\cL(\btheta)\leq\norm{\nabla\cL(\btheta)}^2\leq\frac{2\beta_2^2}{\mu}\cL(\btheta)$;
    $2\mu\cL_i(\btheta)\leq\norm{\nabla\cL_i(\btheta)}^2\leq\frac{2\beta_2^2}{\mu}\cL_i(\btheta)$,
     $\forall i\in[n]$.
    \item (loss v.s. para norm) $\frac{\mu}{2}\norm{\btheta-\Phi(\btheta)}^2\leq\cL(\btheta)\leq\frac{\beta_2^2}{2\mu}\norm{\btheta-\Phi(\btheta)}^2$;
    $\frac{\mu}{2}\norm{\btheta-\Phi(\btheta)}^2\leq\cL_i(\btheta)\leq\frac{\beta_2^2}{2\mu}\norm{\btheta-\Phi(\btheta)}^2$,
     $\forall i\in[n]$.
\end{itemize}
    
\end{lemma}

\begin{lemma}\label{lemma: near minima, projection estimate, std SAM}
For all $\btheta\in\bbB(K;R)$, $\forall i\in[n]$, 
\begin{itemize}[leftmargin=2em]
    \item 
    $$\norm{ P_{n+1:p}(\nabla^2\cL(\btheta))\nabla^2\cL(\btheta)},\norm{ P_{n+1:p}(\nabla^2\cL(\btheta))\nabla^2\cL_i(\btheta)},\norm{\partial\Phi(\btheta)\nabla^2\cL_i(\btheta)}\leq\cO\left(\norm{\btheta-\Phi(\btheta)}\right);$$
    \item 
    $$\norm{ P_{n+1:p}(\nabla^2\cL(\btheta))\nabla\cL(\btheta)},\norm{ P_{n+1:p}(\nabla^2\cL(\btheta))\nabla\cL_i(\btheta)},\norm{\partial\Phi(\btheta)\nabla\cL_i(\btheta)}\leq\cO\bracket{\norm{\btheta-\Phi(\btheta)}^2};$$
    \item Let $\rho>0$ and $\bv\in\bbS^{p-1}$. If $\btheta+\rho\bv\in\bbB(K;R)$, then 
    \begin{gather*}
        \norm{\nabla\cL_i(\btheta+\rho\bv)}\leq\norm{\nabla\cL_i(\btheta)}+\rho\beta_2, \forall i\in[n];
        ;
        \\\norm{ P_{n+1:p}(\nabla^2\cL(\btheta))\nabla\cL_i(\btheta+\rho\bv)}\leq\cO\bracket{\norm{\btheta-\Phi(\btheta)}^2}+\cO\bracket{\rho\norm{\btheta-\Phi(\btheta)}}+\frac{\rho^2\beta_3}{2}, \forall i\in[n];
    \end{gather*}
\end{itemize}
    
\end{lemma}

\begin{lemma}[Lemma H.9 in~\cite{wen2023how}]\label{lemma: wen: lower bound, std SAM}
    For any absolute constant $C>0$, there exist absolute constant $C_1,C_2>0$ such that: if $\btheta_t\in \bbB(K;R)$ and $C_1\eta\rho\leq\norm{\btheta_t-\Phi(\btheta_{t})}\leq C\rho$, then it holds that:
    \begin{align*}
        \bbE_{i_t}\left[\norm{\btheta_{t+1/2}-\Phi(\btheta_{t+1/2})}\right]\leq\norm{\btheta_t-\Phi(\btheta_{t})}-C_2\eta\rho.
    \end{align*}
\end{lemma}

\begin{lemma}[Lemma H.10 in~\cite{wen2023how}]\label{lemma: wen: distance bound, std SAM}
For any absolute constant $C>0$, there
exists absoulte constant $C_3$ such that: if $\btheta_t\in \bbB(K;R)$ and $\norm{\btheta_t-\Phi(\btheta_{t})}\leq C\rho$, then we have
that
\begin{align*}
        \left|\norm{\btheta_{t+1/2}-\Phi(\btheta_{t+1/2})}-\norm{\btheta_t-\Phi(\btheta_{t})}\right|\leq C_3\eta\rho.
    \end{align*}
    
\end{lemma}

Now we fix the positive number $C=1$ and use the absolute constants $C_2,C_3$, defined in Lemma~\ref{lemma: wen: lower bound, std SAM} and Lemma~\ref{lemma: wen: distance bound, std SAM}.

\subsection{Proof of Theorem~\ref{thm: Phase II, std SAM-IRE}}

\subsubsection{Proof of Moving Near Minimizers for SAM-IRE}

This proof is similar to the proof for Section~\ref{subsec: theory: average SAM}.

Under the conditions in Theorem~\ref{thm: Phase II, std SAM-IRE}, the update rule of IRE on standard SAM is
\begin{align*}
    \btheta_{t+1}=&\btheta_t-\eta\nabla\cL_{i_t}\bracket{\btheta_t+\rho\frac{\nabla\cL_{i_t}(\btheta_t)}{\norm{\nabla\cL_{i_t}(\btheta_t)}}}
    \\&-\eta\kappa P_{n+1:p}\bracket{\nabla^2\cL(\btheta_t)}\nabla\cL_{i_t}\bracket{\btheta_t+\rho\frac{\nabla\cL_{i_t}(\btheta_t)}{\norm{\nabla\cL_{i_t}(\btheta_t)}}},\text{ where }i_t\sim\bbU([n]).
\end{align*}

Let the $\kappa$ in IRE satisfy
$$\kappa\leq 1/\rho.$$
    
Additionally, we fix a constant $\alpha\in(0,1)$ in the proof.

For simplicity, we denote 
\begin{gather*}
    \bv_t:=\frac{\nabla\cL_{i_t}(\btheta_t)}{\norm{\nabla\cL_{i_t}(\btheta_t)}},\quad  P(\btheta_t):= P_{n+1:p}\bracket{\nabla^2\cL(\btheta_t)};
\end{gather*}

and
\begin{align*}
    \btheta_{t+1/2}=&\btheta_t-\nabla\cL_{i_t}\left(\btheta_t+\rho\bv_t\right);
    \\
    \btheta_{t+1}=&\btheta_{t+1/2}-\kappa\eta P(\btheta_t)\nabla\cL_{i_t}\left(\btheta_t+\rho\bv_t\right).
\end{align*}

Additionally, we denote
\begin{gather*}
    \bbP_{t+1,t}:=\bbP\left(\norm{\btheta_{t+1}-\Phi(\btheta_{t+1})}\geq2C\eta^{1-\alpha}\rho\Big|\norm{\btheta_{t}-\Phi(\btheta_{t})}<C\eta^{1-\alpha}\rho\right),
\end{gather*}
\begin{align*}
    \bbP_{t+1,s}:=&\bbP\Big(\norm{\btheta_{t+1}-\Phi(\btheta_{t+1})}\geq2C\eta^{1-\alpha}\rho;
    \\&\quad\quad \forall\tau\in[s+1,{t}],C\eta^{1-\alpha}\rho\leq\norm{\btheta_{\tau}-\Phi(\btheta_{\tau})}<2C\eta^{1-\alpha}\rho\Big|\norm{\btheta_{s}-\Phi(\btheta_{s})}<C\sqrt{\eta}\rho\Big),\ s\in[T_{\rm I},t-1].
\end{align*}

Then the following bound holds naturally:
\begin{align*}
    \bbP\left(\exists\ t\in[T_{\rm I},T_{\rm I}+T_{\rm II}],\norm{\btheta_{t}-\Phi(\btheta_{t})}\geq2C\eta^{1-\alpha}\rho\right)\leq\sum_{t=T_{\rm I}}^{T_{\rm I}+T_{\rm II}-1}\left(\bbP_{t+1,t}+\sum_{s=T_{\rm I}}^{t-1}\bbP_{t+1,s}\right).
\end{align*}

\begin{itemize}[leftmargin=2em]
    \item \underline{Step I. Bounding $\bbP_{t+1,t}$.}

    From $\norm{\btheta_t-\Phi(\btheta_t)}=\cO(\eta^{1-\alpha}\rho)$, thus 
\begin{align*}
    \norm{\btheta_t+\rho\bv_t-\Phi(\btheta_t)}\leq\norm{\btheta_t-\Phi(\btheta_t)}+\rho=\cO(\eta^{1-\alpha}\rho)+\cO(\rho)<R,
\end{align*}
which means $\btheta_t+\rho\bv_t\in\bbB(K;R)$. Using Lemma~\ref{lemma: near minima, projection estimate, std SAM} and $\kappa\leq1/\rho$, we can estimate:
\begin{align*}
    &\norm{\btheta_{t+1}-\btheta_{t+1/2}}=\eta\kappa\norm{\nabla P(\btheta_t)\nabla\cL_{i_t}\left(\btheta_t+\rho\bv_t\right)}
    \\\leq&
    \eta\kappa\left(\norm{\nabla P(\btheta_t)\nabla\cL_{i_t}(\btheta_t)}+\rho\norm{\nabla P(\btheta_t)\nabla^2\cL_{i_t}(\btheta_t)}+\frac{\beta_3}{2}\rho^2\right)
    \\\leq&
    \eta\kappa\left(\eta\rho^2+\eta^{1-\alpha}\rho^2+\frac{\beta_3}{2}\rho^2\right)
    \leq\beta_3\eta\kappa\rho^2
    \leq\frac{C_2}{2(1+\zeta_1^\Phi)}\eta\rho,
\end{align*}
\begin{align*}
    &\norm{\btheta_{t+1}-\Phi(\btheta_{t+1})-(\btheta_{t+1/2}-\Phi(\btheta_{t+1/2}))}
    \\\leq&(1+\zeta_{1}^\Phi)\norm{\btheta_{t+1}-\btheta_{t+1/2}}\leq \frac{C_2}{2}\eta\rho.
\end{align*}

Then we have the following bound:
\begin{align*}
    &\norm{\btheta_{t+1}-\Phi(\btheta_{t+1})}
    \\\leq&\norm{\btheta_{t}-\Phi(\btheta_{t})}+\norm{\btheta_{t+1/2}-\Phi(\btheta_{t+1/2})-(\btheta_{t}-\Phi(\btheta_{t}))}
    \\&\quad+\norm{\btheta_{t+1}-\Phi(\btheta_{t+1})-(\btheta_{t+1/2}-\Phi(\btheta_{t+1/2}))}
    \\\overset{\text{Lemma~\ref{lemma: wen: distance bound, std SAM}}}{\leq}& C_1\eta^{1-\alpha}\rho+\cO(\eta\rho)+\norm{\btheta_{t+1}-\Phi(\btheta_{t+1})-(\btheta_{t+1/2}-\Phi(\btheta_{t+1/2}))}
    \\\leq& C_1\eta^{1-\alpha}\rho+\cO(\eta\rho)+\cO(\eta\rho)<\frac{3C_1}{2}\eta^{1-\alpha}\rho.
\end{align*}

Thus, we obtain
\begin{align*}
    \bbP_{t+1,t}=\bbP\left(\norm{\btheta_{\tau+1}-\Phi(\btheta_{\tau+1})}\geq2C\eta^{1-\alpha}\rho\Big|\norm{\btheta_{t}-\Phi(\btheta_{t})}<C\eta^{1-\alpha}\rho\right)=0.
\end{align*}

\item \underline{Step II. Bounding $\bbP_{t+1,s}$ for $s\in[T_{\rm I},{t-1}]$.}

    We prove this step under the condition $\norm{\btheta_{s}-\Phi(\btheta_s)}<C\eta^{1-\alpha}\rho$. Define a process $\{X_{\tau}\}_{\tau=s}^{t+1}$: $X_{s+1}=\norm{\btheta_{s+1}-\Phi(\btheta_{s+1})}$,
    \begin{align*}
        X_{\tau+1}=\begin{cases}
            \norm{\btheta_{\tau+1}-\Phi(\btheta_{\tau+1})},\quad\text{ if } C\eta^{1-\alpha}\rho\leq X_\tau=\norm{\btheta_{\tau}-\Phi(\btheta_{\tau})}\leq 2C\eta^{1-\alpha}\rho
            \\
            X_{\tau}-C_2\eta\rho/2,\quad\text{ else}
        \end{cases}.
    \end{align*}

    It is clear that 
    \begin{align*}
        \bbP_{t+1,s}\leq\bbP\left(X_{t+1}\geq2C\eta^{1-\alpha}\rho\right).
    \end{align*}

    Then our key step is to prove the following two claims about the process $\{X_{\tau}\}$.
    
    \begin{itemize}[leftmargin=2em]
        \item \underline{Claim I. $X_{\tau}-C_2\tau\eta\rho/2$ is a super-martingale.}
        From the definition of $X_{\tau}$, we only need to prove that if $C\eta^{1-\alpha}\rho\leq X_{\tau}=\norm{\btheta_{\tau}-\Phi(\btheta_{\tau})}\leq 2C\eta^{1-\alpha}\rho$, then $\bbE \norm{\btheta_{\tau+1}-\Phi(\btheta_{\tau+1})}\leq\norm{\btheta_{\tau}-\Phi(\btheta_{\tau})}-C_2\eta\rho/2$. 

        If $C\eta^{1-\alpha}\rho\leq X_{\tau}=\norm{\btheta_{\tau}-\Phi(\btheta_{\tau})}\leq 2C\eta^{1-\alpha}\rho$, similar to Step I, it holds that $\btheta_{\tau+1}\in\bbB(K;R)$ and $\norm{\btheta_{t+1}-\Phi(\btheta_{\tau+1})-(\btheta_{\tau+1/2}-\Phi(\btheta_{\tau+1/2}))}\leq \frac{C_2}{2}\eta\rho$. Moreover,
        \begin{align*}
            &\norm{\btheta_{\tau+1}-\Phi(\btheta_{\tau+1})}
            \\\leq&\norm{\btheta_{\tau+1/2}-\Phi(\btheta_{\tau+1/2})}+\norm{\btheta_{\tau+1}-\Phi(\btheta_{\tau+1})-(\btheta_{\tau+1/2}-\Phi(\btheta_{\tau+1/2}))}
            \\\leq&\norm{\btheta_{\tau+1/2}-\Phi(\btheta_{\tau+1/2})}+\frac{C_2}{2}\eta\rho.
        \end{align*}
    Taking the expectation and using Lemma~\ref{lemma: wen: lower bound, std SAM}, we have
        \begin{align*}
        &\bbE\norm{\btheta_{\tau+1}-\Phi(\btheta_{\tau+1})}
        \leq\bbE\norm{\btheta_{\tau+1/2}-\Phi(\btheta_{\tau+1/2})}+\frac{C_2}{2}\eta\rho
        \\\leq&
        \norm{\btheta_{\tau}-\Phi(\btheta_{\tau})}-C_2\eta\rho+\frac{C_2}{2}\eta\rho
        =\norm{\btheta_{\tau}-\Phi(\btheta_{\tau})}-\frac{C_2}{2}\eta\rho.
        \end{align*}

    \item \underline{Claim II. $X_{\tau+1}-X_\tau+C_2\eta\rho/2$ is $\cO(\eta\rho)$-bounded.} 
    From the definition of $X_{\tau}$, we only need to prove for the case $C\eta^{1-\alpha}\rho\leq X_{\tau}=\norm{\btheta_{\tau}-\Phi(\btheta_{\tau})}\leq 2C\eta^{1-\alpha}\rho$.

    If $C\eta^{1-\alpha}\rho\leq X_{\tau}=\norm{\btheta_{\tau}-\Phi(\btheta_{\tau})}\leq 2C\eta^{1-\alpha}\rho$, we have $\btheta_{\tau+1}\in\bbB(K;R)$ and $\norm{\btheta_{t+1}-\Phi(\btheta_{t+1})-(\btheta_{t+1/2}-\Phi(\btheta_{t+1/2}))}
    \leq\frac{C_2}{2}\eta\rho$. Combining this result and Lemma~\ref{lemma: wen: distance bound, std SAM}, we have
    \begin{align*}
        &\left|\norm{\btheta_{\tau+1}-\Phi(\btheta_{\tau+1})}-\norm{\btheta_{\tau}-\Phi(\btheta_{\tau})}\right|
        \\\leq&
        \left|\norm{\btheta_{\tau+1}-\Phi(\btheta_{\tau+1})}-\norm{\btheta_{\tau+1/2}-\Phi(\btheta_{\tau+1/2})}\right|+\left|\norm{\btheta_{\tau+1/2}-\Phi(\btheta_{\tau+1/2})}-\norm{\btheta_{\tau}-\Phi(\btheta_{\tau})}\right|
        \\\leq&
        \norm{(\btheta_{\tau+1}-\Phi(\btheta_{\tau+1}))-(\btheta_{\tau+1/2}-\Phi(\btheta_{\tau+1/2}))}+\left|\norm{\btheta_{\tau+1/2}-\Phi(\btheta_{\tau+1/2})}-\norm{\btheta_{\tau}-\Phi(\btheta_{\tau})}\right|
        \\\leq&\frac{C_2}{2}\eta\rho+C_3\eta\rho=\cO(\eta\rho).
    \end{align*}

    \end{itemize}

    With the preparation of Claim I and Claim II, we can use the Azuma-Hoeffeding inequality: for any $Q>0$, it holds that
    \begin{align*}
        \bbP\left(X_{t+1}-X_{s+1}+(t-s)C_2\eta\rho/2>Q\right)\leq2\exp\bracket{-\frac{Q^2}{2(t-s)\cO(\eta^2\rho^2)}}.
    \end{align*}
    As proved in Claim I, $X_{s+1}=\norm{\btheta_{s+1}-\Phi(\btheta_{s+1})}\leq\frac{3}{2}C\eta^{1-\alpha}\rho$ due to $\norm{\btheta_{s}-\Phi(\btheta_{s})}\leq C\eta^{1-\alpha}\rho$. Therefore, by choosing $Q=(t-s)C_2\eta\rho/2-\frac{3}{2}C\eta^{1-\alpha}\rho+2C\eta^{1-\alpha}\rho=(t-s)C_2\eta\rho/2+C\eta^{1-\alpha}\rho/2$, we have
    \begin{align*}
    &\bbP_{t+1,s}\leq\bbP\left(X_{t+1}\geq2C\eta^{1-\alpha}\rho\right)
    \\\leq&
    \bbP\left(X_{t+1}-X_{s+1}+(t-s)C_2\eta\rho/2>(t-s)\frac{C_2}{2}\eta\rho+\frac{C}{2}\eta^{1-\alpha}\rho\right)
    \\\leq&
    2\exp\bracket{-\frac{\bracket{(t-s)C_2\eta\rho/2+C\eta^{1-\alpha}\rho/2}^2}{2(t-s)\cO(\eta^2\rho^2)}}
    \\\leq&
    2\exp\bracket{-\frac{(t-s)C_2\eta\rho\cdot C\eta^{1-\alpha}\rho}{4(t-s)\cO(\eta^2\rho^2)}}\leq2\exp\left(-\Omega\left(\frac{1}{\eta^\alpha}\right)\right).
    \end{align*}

\end{itemize}

Therefore, we obtain the union bound:
\begin{align*}
    &\bbP\left(\exists\ t\in[T_{\rm I},T_{\rm I}+T_{\rm II}],\norm{\btheta_t-\Phi(\btheta_t)}\geq2C\eta^{1-\alpha}\rho\right)
    \leq\sum_{t=T_{I}}^{T_{\rm I}+T_{\rm II}-1}\left(\bbP_{t+1,t}+\sum_{s=T_{\rm I}}^{t-1}\bbP_{t+1,s}\right)
    \\\leq&
    \sum_{t=T_{I}}^{T_{\rm I}+T_{\rm II}-1}\sum_{s=T_{\rm I}}^{t-1}\bbP_{t+1,s}\leq T_{\rm II}^2\exp\left(-\Omega\left({1}/\eta^\alpha\right)\right).
\end{align*}

\subsubsection{Proof of the Effective Dynamics}

By our proof above, with probability at least $1-T_{\rm II}^2\exp\left(-\Omega\left(1/\eta^{\alpha}\right)\right)$, for any $t\in[T_{\rm I},T_{\rm I}+T_{\rm II}]$, $\norm{\btheta_t-\Phi(\btheta_t)}=\cO(\eta^{1-\alpha}\rho)$.
Then we prove this theorem when the above event occurs.

Due to $\norm{\btheta_{t}-\Phi(\btheta_t)}=\cO(\eta^{1-\alpha}\rho)$, we have:
\begin{align*}
    &\norm{\btheta_{t+1}-\btheta_t}
    \\\leq&\eta\norm{\nabla\cL_{i_t}(\btheta_t+\rho\bv_t)}+\eta\kappa\norm{ P(\btheta_t)\nabla\cL_{i_t}(\btheta_t+\rho\bv_t)}
    \\\overset{\text{Lemma~\ref{lemma: near minima, projection estimate, std SAM}}}{\leq}&
    \eta\norm{\nabla\cL_{i_t}(\btheta_t)}+\cO(\eta\rho)+\cO(\eta\kappa\rho^2)
    \leq\cO(\eta\rho).
\end{align*}

Then by Taylor's expansion, 
\begin{align*}
    &\Phi(\btheta_{t+1})-\Phi(\btheta_t)=\partial\Phi(\btheta_t)(\btheta_{t+1}-\btheta_{t})+\cO\bracket{\norm{\btheta_{t+1}-\btheta_{t}}^2}
    \\=&
    -\eta\partial\Phi(\btheta_t)\nabla\cL_{i_t}(\btheta_{t}+\rho\bv_t)-\eta\kappa\partial\Phi(\btheta_t) P(\btheta_t)\nabla\cL_{i_t}(\btheta_{t}+\rho\bv_t)+\cO(\eta^2\rho^2).
\end{align*}

For the term $\partial\Phi(\btheta_t)\nabla\cL_{i_t}(\btheta_{t}+\rho\bv_t)$ and $\partial\Phi(\btheta_t) P(\btheta_t)\nabla\cL_{i_t}(\btheta_{t}+\rho\bv_t)$, using Taylor's expansion and Lemma~\ref{lemma: near minima, projection estimate, std SAM}, we have
\begin{align*}
    &\partial\Phi(\btheta_t)\nabla\cL_{i_t}(\btheta_{t}+\rho\bv_t)
    \\=&\partial\Phi(\btheta_t)\nabla\cL_{i_t}(\btheta_{t})+\rho\partial\Phi(\btheta_t)\nabla^2\cL_{i_t}(\btheta_t)\bv_t+\frac{\rho^2}{2}\partial\Phi(\btheta_t)\nabla\Tr\left(\bv_t\nabla^2\cL_{i_t}(\btheta_t)\bv_t^\top\right)+\cO(\rho^3)
    \\=&\partial\Phi(\btheta_t)\nabla\cL_{i_t}(\btheta_{t})+\rho\partial\Phi(\btheta_t)\nabla^2\cL_{i_t}(\btheta_t)\bv_t+\frac{\rho^2}{2}\partial\Phi(\btheta_t)\nabla\left(\bv_t^\top\nabla^2\cL_{i_t}(\btheta_t)\bv_t\right)+\cO(\rho^3)
    \\=&\cO(\norm{\btheta_t-\Phi(\btheta_t)}^2)+\rho\partial\Phi(\Phi(\btheta_t))\nabla^2\cL_{i_t}(\Phi(\btheta_t))\frac{\nabla\cL_{i_t}(\Phi(\btheta_t))}{\norm{\nabla\cL_{i_t}(\Phi(\btheta_t))}}+\cO(\rho\norm{\btheta_t-\Phi(\btheta_t)})
    \\&+\frac{\rho^2}{2}\partial\Phi(\Phi(\btheta_t))\nabla\left(\frac{\nabla\cL_{i_t}(\Phi(\btheta_t))}{\norm{\nabla\cL_{i_t}(\Phi(\btheta_t))}}^\top\nabla^2\cL_{i_t}(\Phi(\btheta_t))\frac{\nabla\cL_{i_t}(\Phi(\btheta_t))}{\norm{\nabla\cL_{i_t}(\Phi(\btheta_t))}}\right)
    \\&+\cO(\rho^2\norm{\btheta_t-\Phi(\btheta_t)})+\cO(\rho^3)
    \\=&\frac{\rho^2}{2}\partial\Phi(\Phi(\btheta_t))\nabla\lambda_1\bracket{\nabla^2\cL_{i_t}(\Phi(\btheta_t))}+\cO(\eta^{1-\alpha}\rho^2+\rho^3),
\end{align*}
and
\begin{align*}
    &\partial\Phi(\btheta_t) P(\btheta_t)\nabla\cL_{i_t}(\btheta_{t}+\rho\bv_t)
    \\=&\partial\Phi(\btheta_t) P(\btheta_t)\nabla\cL_{i_t}(\btheta_{t})+\rho\partial\Phi(\btheta_t) P(\btheta_t)\nabla^2\cL_{i_t}(\btheta_t)\bv_t
    \\&+\frac{\rho^2}{2}\partial\Phi(\btheta_t) P(\btheta_t)\nabla\Tr\left(\bv_t\nabla^2\cL_{i_t}(\btheta_t)\bv_t^\top\right)+\cO(\rho^3)
    \\=&\cO(\norm{\btheta_t-\Phi(\btheta_t)}^2)+\cO(\rho\norm{\btheta_t-\Phi(\btheta_t)})+\frac{\rho^2}{2}\partial\Phi(\btheta_t) P(\btheta_t)\nabla\left(\bv_t^\top\nabla^2\cL_{i_t}(\btheta_t)\bv_t\right)+\cO(\rho^3)
    \\=&\cO(\eta^{1-\alpha}\rho^2)+\frac{\rho^2}{2}\partial\Phi(\Phi(\btheta_t))\nabla\left(\frac{\nabla\cL_{i_t}(\Phi(\btheta_t))}{\norm{\nabla\cL_{i_t}(\Phi(\btheta_t))}}^\top\nabla^2\cL_{i_t}(\Phi(\btheta_t))\frac{\nabla\cL_{i_t}(\Phi(\btheta_t))}{\norm{\nabla\cL_{i_t}(\Phi(\btheta_t))}}\right)
    \\&+\cO(\rho^2\norm{\btheta_t-\Phi(\btheta_t)})+\cO(\rho^3)
    \\=&\frac{\rho^2}{2}\partial\Phi(\Phi(\btheta_t))\nabla\lambda_1\bracket{\nabla^2\cL_{i_t}(\Phi(\btheta_t))}+\cO(\eta^{1-\alpha}\rho^2+\rho^3),.
\end{align*}


Combining the results above, we obtain:
\begin{align*}
    \Phi(\btheta_{t+1})=&\Phi(\btheta_{t})-(1+\kappa)\frac{\eta\rho^2}{2}\partial\Phi(\Phi(\btheta_t))\nabla\lambda_1\bracket{\nabla^2\cL_{i_t}(\Phi(\btheta_t))}+\cO(\kappa\eta^{2-\alpha}\rho^2+\kappa\eta\rho^3).
\end{align*}

Additionally, taking the expectation, we have:
\begin{align*}
    \bbE_{i_t}\left[\partial\Phi(\Phi(\btheta_t))\nabla\lambda_1\bracket{\nabla^2\cL_{i_t}(\Phi(\btheta_t))}\right]=\partial\Phi(\Phi(\btheta_t))\nabla\Tr\bracket{\nabla^2\cL(\Phi(\btheta_t))}.
\end{align*}

Therefore, we obtain
\begin{align*}
    \bbE_{i_t}\left[{\Phi(\btheta_{t+1})}\right]=&\Phi(\btheta_{t})-(1+\kappa)\frac{\eta\rho^2}{2}\partial\Phi(\Phi(\btheta_t))\nabla\Tr\bracket{\nabla^2\cL(\Phi(\btheta_t))}+\cO(\kappa\eta^{2-\alpha}\rho^2+\kappa\eta\rho^3)
    \\=&\Phi(\btheta_{t})-(1+\kappa)\frac{\eta\rho^2}{2}\partial\Phi(\Phi(\btheta_t))\nabla\Tr\bracket{\nabla^2\cL(\Phi(\btheta_t))}+h.o.t..
\end{align*}


%% file: appendix/proof_basic.tex
\vspace{1.cm}
\section{Useful Inequalities}
\label{appendix: lemmas}

\begin{definition}[$\mu$-PL]\label{def: PL}
Let $\mu>0$ be a constant.
A function $\cL$ is $\mu$-PL in a set $U$ iff 
$\norm{\nabla\cL(\btheta)}^2\geq2\mu(\cL(\btheta)-\inf\limits_{\btheta\in U}\cL(\btheta)),\forall\btheta\in U$.    
\end{definition}

\begin{lemma}[Weyl Theorem]\label{lemma: Weyl thm}
Let $\bA,\bB\in\bbR^{p\times p}$ be symmetric with eigenvalues $\lambda_1\geq\cdots\geq\lambda_p$ and $\mu_1\geq\cdots\geq\mu_p$ respectively, then for any $k\in[p]$, it holds that
\begin{align*}
    |\lambda_k-\mu_k|\leq\norm{\bA-\bB}.
\end{align*}
\end{lemma}

\begin{lemma}[Davis-Kahan $\sin(\Theta)$ theorem]\label{lemma: Davis-Kahan thm}
Let $\bA,\bB\in\bbR^{p\times p}$ be symmetric matrices.
Denote their orthogonal decomposition as $\bA=\bE_1\bLambda_1\bE_1^\top+\bE_2\bLambda_2\bE_2^\top$ and $\bB=\bF_1\bGamma_1\bF_1^\top+\bF_2\bGamma_2\bF_2^\top$ with $(\bE_1, \bE_2)$ and $(\bD_1, \bD_2)$ orthogonal. If the eigenvalues in $\bLambda_1$ are contained in an interval $(a,b)$, and the eigenvalues of $\bGamma_2$ are excluded from the interval $(a-\delta,b+\delta)$ for some $\delta>0$, then for any unitarily invariant norm $\norm{\cdot}_\star$, 
\begin{align*}
    \norm{\bF_2^\top\bE_1}_\star\leq\frac{\norm{\bF_2^\top(\bA-\bB)\bE_1}_\star}{\delta}.
\end{align*}
    
\end{lemma}

\begin{lemma}[Azuma-Hoeffding Inequality]
\label{lemma: Azuma}
Suppose $\{X_n\}_{n\in\bbN}$ is a super-martingale.
\begin{itemize}[leftmargin=2em]
    \item (i) (Bounded martingale difference). If $-\alpha\leq X_{i+1}-X_i\leq\beta$, then for any $n,t>0$, we have:
    \begin{align*}
        \bbP\left(X_n-X_0\geq t\right)\leq 2\exp\left(-\frac{t^2}{2n(\alpha+\beta)^2}
        \right).
    \end{align*}
    \item (ii) (Sub-Gaussian martingale difference). If $X_{i+1}-X_i$ is $\sigma_i^2$-sub-Gaussian, then for any $n,t>0$, we have:
    \begin{align*}
        \bbP\left(X_n-X_0\geq t\right)\leq 2\exp\left(-\frac{t^2}{2\sum_{i=1}^n\sigma_i^2}
        \right).
    \end{align*}
\end{itemize}

\end{lemma}

\begin{lemma}\label{lemma: gaussian projection prob bound}
    Let $\bv\in\bbR^{p}$. Let $\bg\sim\cN(\bzero,\bI)$.
    Then there exists an absolute constant $c>0$ such that for any $t>0$, 
    \begin{align*}
        \bbP\left(\left|\<\frac{\bv}{\norm{\bv}},\frac{\bg}{\norm{\bg}}\>\right|\geq t\right)\leq 4e^{-c p t^2}.
    \end{align*}
\end{lemma}

\begin{proof}[Proof of Lemma~\ref{lemma: gaussian projection prob bound}]
From P54 in~\cite{vershynin2018high}, there exists an absolute constant $c>0$ such that for any $t>0$, $\bbP\left(\frac{|\<\be_1,\bg\>|}{\norm{\bg}}\geq t\right)\leq 4 e^{-c p t^2}$. 
Without loss of generality, we can assume $\bv\ne\bzero$. Then we have:
\begin{align*}
    \bbP\left(\left|\<\frac{\bv}{\norm{\bv}},\frac{\bg}{\norm{\bg}}\>\right|\geq\frac{t}{\sqrt{p}}\right)
    =\bbP\left(\left|\frac{\<\be_1,\bg\>}{\norm{\bg}}\right|\geq t\right)
    \leq 4 e^{-c p t^2}.
\end{align*}
\end{proof}


\begin{lemma}\label{lemma: F norm bounded by 2, F norm}
    $\norm{\bA\bB}_{\rm F}\leq \norm{\bA}\norm{\bB}_{\rm F}$.
\end{lemma}